\documentclass[pmlr]{jmlr}% new name PMLR (Proceedings of Machine Learning Research)
\usepackage{bm}
\usepackage{soul}
\usepackage{url}

\usepackage{thm-restate}
\usepackage{thmtools} 
\usepackage{float}
\DeclareMathOperator*{\argmax}{argmax}
\DeclareMathOperator*{\argmin}{argmin}

\newcommand{\Pa}{\textbf{\textit{Pa}}}
\newcommand{\pa}{\textbf{\textit{pa}}}

\newcommand{\Anc}{\bm{\mathit{Anc}}}
\newcommand{\Des}{\bm{\mathit{Des}}}
\DeclareMathOperator\supp{supp}

\newcommand{\doi}{\textit{do}}

\newcommand{\Sphere}{\textit{Sphere}}

\newcommand{\E}{\mathbb{E}}
\newcommand{\I}{\mathbb{I}}

\newcommand{\bA}{\boldsymbol{A}}

\newcommand{\bb}{\boldsymbol{b}}

\newcommand{\bS}{\boldsymbol{S}}
\newcommand{\bs}{\boldsymbol{s}}

\newcommand{\bv}{\boldsymbol{v}}
\newcommand{\bV}{\boldsymbol{V}}

\newcommand{\bW}{\boldsymbol{W}}
\newcommand{\bX}{\bm{X}}

\newcommand{\bx}{\bm{x}}

\newcommand{\bz}{\boldsymbol{z}}

\newcommand{\bP}{\boldsymbol{P}}

\newcommand{\btheta}{\boldsymbol{\theta}}
\newcommand{\bbvarepsilon}{\boldsymbol{\varepsilon}}
\newcommand{\bgamma}{\boldsymbol{\gamma}}

\newcommand{\cA}{\mathcal{A}}

\newcommand{\cE}{\mathcal{E}}

\newcommand{\EE}{\mathbb{E}}
\newcommand{\II}{\mathbb{I}}

\usepackage{algorithmic}
\newtheorem{assumption}{Assumption}
\newtheorem{claim}{Claim}

\usepackage{longtable}% for long tables

\usepackage{booktabs}
\usepackage{booktabs}
\usepackage{lineno}

\renewcommand{\thepage}{}

\title[CCB without Graph Skeleton]{Combinatorial Causal Bandits without Graph Skeleton}

 \author{
  \Name{Shi Feng}\thanks{Equal Contribution.} \Email{shifeng-thu@outlook.com}\\
  \addr Harvard University, MA, USA
  \AND
  \Name{Nuoya Xiong}\footnotemark[1] \Email{nuoyax@andrew.cmu.edu}\\
  \addr Carnegie Mellon University, PA, USA
  \AND
  \Name{Wei Chen} \Email{weic@microsoft.com}\\
  \addr Microsoft Research, Beijing, China
}

\begin{document}

\maketitle

\begin{abstract}
In combinatorial causal bandits (CCB), the learning agent chooses a subset of variables in each round to intervene and collects feedback from the observed variables to minimize expected regret or sample complexity. Previous works study this problem in both general causal models and binary generalized linear models (BGLMs). However, all of them require prior knowledge of causal graph structure or unrealistic assumptions. This paper studies the CCB problem without the graph structure on binary general causal models and BGLMs. We first provide an exponential lower bound of cumulative regrets for the CCB problem on general causal models. To overcome the exponentially large space of parameters, we then consider the CCB problem on BGLMs. We design a regret minimization algorithm for BGLMs even without the graph skeleton and show that it still achieves $O(\sqrt{T}\ln T)$ expected regret, as long as the causal graph satisfies a weight gap assumption. This asymptotic regret is the same as the state-of-art algorithms relying on the graph structure. Moreover, we propose another algorithm with $O(T^{\frac{2}{3}}\ln T)$ regret to remove the weight gap assumption. 
\end{abstract}


\section{Introduction}

The multi-armed bandits (MAB) problem is a classical model in sequential decision-making \citep{robbins1952some, auer2002finite, bubeck2012regret}. 
In each round, the learning agent chooses an arm and observes the reward feedback corresponding to that arm, with the goal of either 
	maximizing the cumulative reward over $T$ rounds (regret minimization), or minimizing the sample complexity to find the arm closest to the optimal one (pure exploration).
MAB can be extended to have more structures among arms and reward functions, which leads to more advanced learning techniques. 
Such structured bandit problems include combinatorial bandits \citep{chen2013combinatorial, chen2016combinatorial}, linear bandits \citep{abbasi2011improved, agrawal2013thompson, li2017provably}, and sparse linear bandits \citep{abbasi2012online}. 

In this paper, we study another structured bandit problem called causal bandits, which is first proposed by \cite{lattimore2016causal}. It consists of a causal graph $G=({\bX}\cup\{Y\},E)$ indicating the causal relationship among the observed variables. In each round, the learning agent selects one or a few variables in $\bX$ to intervene, gains the reward as the output of $Y$, and observes the values of all variables in $\bX\cup \{Y\}$. 
The use of causal bandits is possible in a variety of contexts that involve causal relationships, including medical drug testing, performance tuning, policy making, scientific experimental process, etc.

In all previous literature except \cite{lu2021causal, konobeev2023causal}, the structure of the causal graph is known, but the underlying probability distributions governing the causal model are unknown. \cite{lu2021causal} further assume that the graph structure is unknown and the learning agent can only see the graph skeleton. Here, graph skeleton is also called essential graph \citep{gamez2013advances} and represents all the edges in $G$ without the directional information. 
In our paper, we further consider that the graph skeleton is unknown and remove the unrealistic assumption that $Y$ only has a single parent in \cite{lu2021causal, konobeev2023causal}. In many scenarios, the learning agent needs to learn the causal relationships between variables and thus needs to learn the graph without any prior information.
For example, in policymaking for combating COVID-19, many possible factors like food supply, medical resources, vaccine research, public security, and 
	public opinion may consequently impact the mortality rate. 
However, the causal relationships among these factors are not readily known and need to be clarified during the sequential decision-making process.
Learning the causal graph from scratch while identifying the optimal intervention raises a new challenge to the learning problem. 

For regret minimization, we study CCB under the BGLMs as in \cite{feng2022combinatorial,xiong2022pure}. 
Using a novel initialization phase, we could determine the ancestor structure of the causal graph for the BGLM when the minimum edge weight in the model
	satisfies a weight gap assumption.
This is enough to perform a CCB algorithm based on maximum likelihood estimation on it~\citep{feng2022combinatorial}. 
The resulting algorithm BGLM-OFU-Unknown achieves $O(\sqrt{T}\log T)$ regret, where $T$ is the time horizon. The big $O$ notation only holds for $T$ larger than a threshold so the weight gap assumption is hidden by the asymptotic notation. 
For binary linear models (BLMs), we can remove the weight gap assumption with the $O(T^{2/3})$ regret. The key idea is to measure the  difference in the reward between the estimated graph (may be inaccurate) and the true graph.
The algorithms we design for BLMs allow hidden variables and use linear regression instead of MLE to remove an assumption on parameters.

For pure exploration, we give some discussions on general causal models in Appendix~\ref{sec:pureexploration}. If we allow the weight gap, a trivial solution exists. Without the weight gap, we give an adaptive algorithm for general causal model in the atomic setting.

In summary, our contribution includes: 
	(a) providing an exponential lower bound of cumulative regret for CCB on general causal model, 
    (b) proposing an $O(\sqrt{T}\ln T)$ cumulative regret CCB algorithm BGLM-OFU-Unknown for BGLMs without graph skeleton (with the weight gap assumption), 
    (c) proposing an $O(T^{\frac{2}{3}}\ln T)$ cumulative regret CCB algorithm BLM-LR-Unknown for BLMs without graph skeleton and the weight gap assumption, (d) conducting a numerical experiment in Appendix~\ref{sec.experiments} for BGLM-OFU-Unknown and BLM-LR-Unknown and giving intuitions on how to choose between them,
    (e) giving the first discussion in Appendix~\ref{sec:pureexploration} including algorithms and lower bounds on the pure exploration of causal bandits on general causal models and atomic intervention 
    without knowing the graph structure.

\section{Related Works}

In this section, we introduce two related lines of research.

\subsection{Causal Bandits} 

The causal bandits problem is first proposed by \cite{lattimore2016causal}. 
They discuss the simple regret for parallel graphs and general graphs with known probability distributions $P(\Pa(Y)|a)$ for any action $a$. In this context, $\Pa(Y)$ represents the parent nodes of $Y$. \cite{sen2017identifying, nair2021budgeted, maiti2021causal} generalize the simple regret study for causal bandits to more general causal graphs and soft interventions. \cite{lu2020regret, nair2021budgeted, maiti2021causal} consider cumulative regret for causal bandits problem. 
However, all of these studies are not designed for combinatorial action set and has exponentially large regret or sample complexity with respect to 
	the graph size if the actions are combinatorial. \cite{yabe2018causal, feng2022combinatorial, xiong2022pure, varici2022causal} consider combinatorial action set for causal bandits problem. 
Among them, \cite{feng2022combinatorial} are the first to remove the requirement of $T>\sum_{X\in\bX}2^{|\Pa(X)|}$ and proposes practical CCB algorithms on BGLMs with $O(\sqrt{T}\ln T)$ regret. 
\cite{xiong2022pure} simultaneously propose CCB algorithms on BGLMs as well as general causal models with polynomial sample complexity with respect to the graph size. \cite{varici2022causal} further include soft interventions in the CCB problem, but their work is on linear structural equation models (SEM). \cite{lee2018structural, lee2019structural, lee2020characterizing} propose several CCB algorithms on general causal bandits problem, but they focus on empirical studies while we provide theoretical regret analysis. All of the above works require the learning agent to know the graph structure in advance. 
\cite{lu2021causal} are the first to work on causal bandits without graph structure. 
However, their algorithm is limited to the case of $|\Pa(Y)|=1$ for the atomic setting, and thus the main technical issue degenerates to finding the particular parent of $Y$ so that one could intervene on this node for the optimal reward. Recently, \cite{konobeev2023causal} has eliminated the need for prior knowledge of the graph skeleton as required in \cite{lu2021causal}. However, their approach is still limited to designing bandit algorithms for the atomic setting with $|\Pa(Y)|=1$. Furthermore, their algorithm may experience exponentially large regret when $1/\min_{X\in\Anc(Y), x\in\supp(X),y\in\supp(Y)}|P(Y=y|X=x)-P(Y=y)|$ is exponentially large in relation to the graph size. In this context, $\Anc(Y)$ denotes the ancestors of $Y$ and $\supp$ represents the support of a random variable. Recently, \cite{malek2023additive} also studied the causal bandits problem without a graph; however, their objective is different from ours. Instead of minimizing regret, they aim to find a near-optimal intervention in the fewest number of exploration rounds.

\subsection{Social Network and Causality} 
Causal models have intrinsic connections with influence propagation in social networks.
\cite{feng2021causal} study the identifiability in the Independent Cascade (IC) propagation model as a causal model. 
The BGLM studied in this paper contains the IC model and linear threshold (LT) model in a DAG as special cases, and is also related to 
	the general threshold model proposed by \cite{kempe2003maximizing}. 
Moreover, \cite{feng2022combinatorial, xiong2022pure} also study causal bandits on BGLMs to avoid the exponentially large parameter space of general causal models. 
These papers borrow some techniques and ideas from influence maximization literature, including \cite{li2020onlineinference} and \cite{zhang2022online}. 
However, in our BGLM CCB problem, the graph skeleton is unknown, and we need adaptation and integration of previous techniques together with some new ingredients.

\section{Model}

We utilize capital letters ($U, X, Y\ldots$) to represent variables and their corresponding lower-case letters to indicate their values, as was frequently done in earlier causal inference literature (see, for example, \citep{pearl2009causality, pearl2009causal, pearl2018book}).
To express a group or a vector of variables or values, we use boldface characters like $\bX$ and $\bx$. For a vector $x\in\mathbb{R}^d$, the weighted $\ell_2$-norm associated with a positive-definite matrix $A$ is defined by $\left\|x\right\|_A=\sqrt{x^\intercal Ax}$.
% \wei{The notation $\mathcal{l}_2$ is wrong. perhaps just use $L_2$ norm or $\ell_2$ norm.}

{\textbf Causal Models.} A {\em causal graph} $G=({\bX}\cup \{Y\},E)$ is a directed acyclic graph consisting of intervenable variables $\bX$, a special target node $Y$
	without outgoing edges, and the set of directed edges $E$ connecting nodes in ${\bX}\cup \{Y\}$.
 Denote $n=|\bX|$ as the number of nodes in $\bX$.
For simplicity, in this paper we consider all variables in ${\bX}\cup \{Y\}$ are $(0,1)$-binary random variables. In our main text, all the variables in ${\bX}\cup \{Y\}$ are known and their values can be observed but the edges in $E$ are unknown and cannot be directly observed.
We refer to the in-neighbor nodes of a node $X$ in $G$ as the {\em parents} of $X$, denoted by $\Pa(X)$, and the values of these parent random variables
as $\pa(X)$. According to the definition of causal Bayesian model \citep{pearl1995testing,pearl2009causality}, the probability distribution $P(X | \Pa(X))$ is used to represent the causal relationship between $X$ and its parents for every conceivable value combination of $\Pa(X)$. Moreover, we define the ancestors of a node $X\in\bX\cup\{Y\}$ by $\Anc(X)$.

We mainly study the {\em Markovian} causal graph $G$ in this paper, which means that there are no hidden variables in $G$ and every observed variable $X$ has some randomness that is not brought on by any other variables.\footnote{In Section~\ref{sec.withoutgap} we mention that our algorithm for BLM can also work for models with hidden variables.}
In this study, we dedicate random variable $X_1$ to be a special variable that always takes the value $1$ and is a parent of all other observed random variables in order to model the self-activation effect of the Markovian model. In essence, this represents the initial probability for each node, ensuring that even when all parent nodes of a node except $X_1$ are set to $0$, the given node still possesses a probability of being $1$.

In this paper, we study a special causal model called binary generalized linear model (BGLM). Specifically, in BGLM, we have $P(X=1|\Pa(X)=\pa(X))=f_X(\btheta^*_X\cdot\pa(X))+\varepsilon_X$, where $f_X$ is a monotone increasing function, $\btheta^*_X$ is an unknown weight vector in $[0,1]^{|\Pa(X)|}$, and $\varepsilon_X$ is a zero-mean sub-Gaussian noise that ensures the probability does not exceed $1$ or equivalently, $\varepsilon_X\leq 1-\max_{\pa(X)\in\{0,1\}^{|\Pa(X)|}}f_X(\pa(X)\cdot\btheta_X^*)$. The bounded epsilon follows the convention established by GLM \citep{li2020onlineinference} and provides randomness for linear models in our paper. We use the notation $\theta^*_{X',X}$ to denote the entry in the vector $\btheta^*_X$ that corresponds to node $X'\in \Pa(X)$, $\btheta^*$ to denote the vector of all the weights, and $\Theta$ to denote the feasible domain for the weights. We also use the notation $\bbvarepsilon$ to represent all noise random variables $(\varepsilon_X)_{X\in \bX\cup Y}$.

We also study binary linear model (BLM) and linear model in this paper. In BLMs, all $f_X$'s are identity functions, so $P(X=1|\Pa(X)=\pa(X)) = \btheta^*_{X}\cdot \pa(X)+\varepsilon_{X}$. When we remove the noise variable $\varepsilon_{X}$, BLM coincides with the {\em linear threshold (LT)} model for influence cascades~\citep{kempe2003maximizing} in a DAG. In linear models, we remove the randomness of conditional probabilities, so $X = \btheta^*_{X}\cdot \pa(X)+\varepsilon_{X}$.

For the unknown causal graph, there is an important parameter $\theta^*_{\min} = \min_{(X',X)\in E}\theta^*_{X',X}$, which represents the minimum weight gap 
	for all edges. 
Intuitively, this minimum gap measures the difficulty for the algorithm to discover the edge and its correct direction. When the gap is relatively large, 
	we can expect to discover the whole graph accurately during the learning process; When the gap is very small, we cannot guarantee to discover 
	the graph directly and we must come up with another way to solve the causal bandit problem on an inaccurate model. 

{\textbf Combinatorial Causal Bandits.} The problem of combinatorial causal bandits (CCB) was first introduced by \cite{feng2022combinatorial} and describes the following online learning task. The intervention can be performed on all variables except $X_1$ and $Y$ and is denoted by the do-operator $\doi$ following earlier causal inference literature \citep{pearl2009causality, pearl2009causal, pearl2018book}. The action set is defined by $\mathcal{A}\subseteq \{\doi(\bS=\bs)\}_{\bS\subseteq\bX\backslash\{X_1\},\bs\in\{0,1\}^{|\bS|}}$.
The expected reward $Y$ under an intervention on $\bS\subseteq\bX\backslash\{X_1\}$ is denoted as $\E[Y | \doi({\bS}={\bs})]$. A learning agent runs an algorithm $\pi$ for $T$ rounds, taking parameter initializations and feedback from causal propagation as inputs, and outputting the selected interventions in all rounds. In particular, an {\em atomic intervention} intervenes on only one node, i.e.  $|\bS|=1$. In this paper, we assume the null intervention $do()$ and atomic interventions $do(X=x)$ are always included in our action set $\mathcal{A}$, because they are needed to discover the graph structure.  

The performance of the agent could be measured by the {\em regret} of the algorithm $\pi$. The regret $R^\pi(T)$ in our context is the difference between the cumulative reward using algorithm $\pi$ and the expected cumulative reward of choosing 
	best action $\doi({\bS}^*=\bs^*)$. Here, $\doi({\bS}^*=\bs^*) \in \argmax_{\doi(\bS=\bs)\in\mathcal{A}} \E[Y | \doi({\bS})]$. Formally, we have
\begin{equation}
% \resizebox{.91\linewidth}{!}{$
%             \displaystyle
            \label{eq:regret}
R^\pi(T)= \mathbb{E}\left[\sum_{t=1}^T(\E[Y | \doi({\bS}^*={\bs}^*)]- \E[Y | \doi({\bS}_t^\pi={\bs}_t^\pi)])\right],
% $}
\end{equation}
where ${\bS}_t^\pi$ and ${\bs}_t^\pi$ are the intervention set and intervention values selected by algorithm $\pi$ in round $t$ respectively. 
The expectation is from the randomness of the causal model and the algorithm $\pi$.

 In this paper, we mainly focus on the regret minimization problem, and we will discuss the pure exploration problem and its sample complexity
in the Section~\ref{sec:pureexploration}.
We defer the definition of sample complexity to that section.

\section{Lower Bound on General Binary Causal Model}
In this section, we explain why we only consider BGLM and BLM instead of the general binary causal model in the combinatorial causal bandit setting. In this context, a general binary causal model refers to a causal Bayesian model, in which all variables are restricted to either $0$ or $1$ values. Both BGLM and BLM are special cases of this model.
Note that in the general case both the number of actions and the number of parameters of the causal model are exponentially large to the size of the graph. 
The following theorem shows that in the general binary causal model, the regret bound must be exponential to the size of the graph when $T$ is sufficiently large,
	or simply linear to $T$ when $T$ is not large enough.
This means that we cannot avoid the exponential factor for the general case, and thus justify our consideration of the BGLM and BLM settings with only a polynomial number of parameters relative to $n$.

\begin{theorem}[Binary Model Lower Bound]\label{thm:binarylowerbound}
    Recall that $n = |\bX|$. For any algorithm, when $T\ge \frac{16(2^n-1)}{3}$, there exists a precise bandit instance of general binary causal model $\mathcal{T}$ such that 
    \begin{align*}
        \mathbb{E}_{\mathcal{T}}[R(T)]\ge \frac{\sqrt{2^nT}}{8e}.
    \end{align*}
    Moreover, when $T\le \frac{16(2^n-1)}{3}$, there exists a precise bandit instance of general binary causal model $\mathcal{T}$ that 
    \begin{align*}
        \mathbb{E}_{\mathcal{T}}[R(T)] \ge \frac{T}{16e}.
    \end{align*}
\end{theorem}

The lower bound contains two parts. The first part shows that the asymptotic regret cannot avoid an exponential term $2^n$ when $T$ is large. The second part states that if $T$ is not exponentially large, the regret will be linear at the worst case.
The proof technique of this lower bound is similar to but not the same as previous classical bandit, because the existence of observation $do()$ and atomic intervention $do(X_i=1)$ may provide more information. To our best knowledge, this result is the first regret lower bound on the general causal model considering the potential role of observation and atomic intervention.
The result shows that in the general binary causal model setting, it is impossible to avoid the exponential term in the cumulative regret even with the observations
	on null and atomic interventions. 
The proof of lower bound is provided in Appendix~\ref{appendix:binarylowerbound}. 

The main idea is to consider the action set $\bA = \{do(), do(X=x), do(\bX = \bx)\}$ for all node $X$, $x \in \{0,1\}$, $\bx \in \{0,1\}^n$ be the null intervention, atomic interventions and actions that intervene all nodes. 
The causal graph we use is a parallel graph where all nodes in $\bX$ directly points to $Y$ with no other edges in the graph, and
	each node $X_i \in \bX$ has probability $P(X_i=1)=P(X_i=0)=0.5$. 
Intuitively, under this condition the null
	intervention and atomic interventions can provide limited information to the agent.
This fact shows that observations and atomic interventions may not be  conducive to our learning process in the worst case on the general binary causal model.

\section{BGLM CCB without Graph Skeleton but with Minimum Weight Gap}
\label{sec.bglm}

In this section, we propose an algorithm for causal bandits on Markovian BGLMs based on maximum likelihood estimation (MLE) without any prior knowledge of the graph skeleton. 

Our idea is to try to discover the causal graph structure and then apply the recent CCB algorithm with known graph structure~\citep{feng2022combinatorial}.
We discover the graph structure by using atomic interventions in individual variables.
However, there are a few challenges we need to face on graph discovery. 
First, it could be very difficult to exactly identify all parent-child relationships, since some grand-parent nodes may also have strong causal influence to
	its grand-child nodes.
Fortunately, we find that it is enough to identify ancestor-descendant relationships instead of parent-child relationships,  
	since we can artificially add an edge with $0$ weight between each pair of ancestor and descendant without impacting the 
	causal propagation results. 
Another challenge is the minimum weight gap.
When the weight of an edge is very small, we need to perform more atomic interventions to identify its existence and its direction. 
Hence, we design an initialization phase with the number of rounds proportional to the total round number $T$ and promise that the ancestor-descendant relationship can always be identified correctly with a large probability when $T$ is sufficiently large.  

Following \cite{li2017provably, feng2022combinatorial,xiong2022pure}, we have three assumptions:
\begin{assumption}
	\label{asm:glm_3}
	For every $X \in \bX\cup\{Y\}$,
	$f_X$ is twice differentiable. Its first and second order derivatives are upper-bounded by $L_{f_X}^{(1)} >0$ and $L_{f_X}^{(2)} >0$. 
\end{assumption}
Let $\kappa=\inf_{X \in \bX\cup\{Y\}, \bv\in[0,1]^{|\Pa(X)|},||\btheta-\btheta^*_X||\leq 1} f'_X(\bv\cdot \btheta)$.
\begin{assumption}
\label{asm:glm_2}
We have $\kappa > 0$.
\end{assumption}
\begin{assumption}
\label{asm:glm_4}
There exists a constant $\zeta>0$ such that for any $X\in {\bX}\cup\{Y\}$ and $X'\in \Anc(X)$, 
for any value vector $\bv \in \{0,1\}^{|\Anc(X)\setminus \{X',X_1\} |}$,
the following inequalities hold:
\begin{align}
    &\Pr_{\bbvarepsilon,\bX,Y}\left(X'=1|\Anc(X)\setminus \{X',X_1\}=\bv\right) \geq \zeta , \label{eq:parentsZero}\\
    &\Pr_{\bbvarepsilon,\bX,Y}\left(X'=0|\Anc(X)\setminus \{X',X_1\}=\bv\right)\geq \zeta.  \label{eq:parentsOne}
\end{align} 
\end{assumption}
Assumptions \ref{asm:glm_3} and \ref{asm:glm_2} are the classical assumptions in generalized linear model \citep{li2017provably}. Assumption \ref{asm:glm_4} makes sure that each ancestor node of $X$ has some freedom to become $0$ and $1$ with a non-zero probability, even when the values of all other ancestors of $X$ are fixed, and it is originally given in \cite{feng2022combinatorial} with additional justifications. For BLMs and continuous linear models, we propose an algorithm based on linear regression without the need of this assumption in Appendix~\ref{app.blmlr}. Furthermore, we suppose that $\text{Range}(f_X) = \mathbb{R}$. As in \cite{feng2024correction}, in Appendix~\ref{app.conversion} we demonstrate that any function $f_X$ can be transformed into a function with range $\mathbb{R}$ without affecting the propagation of BGLM.

To discover the ancestors of all variables, we need to perform an extra initialization phase (see Algorithm~\ref{alg:glm-ofu}). 
We denote the total number of rounds by $T$ and arbitrary constants $c_0,c_1$ to make sure that $c_0T^{1/2}\in \mathbb{N}^+$ for simplicity of our writing. In the initialization phase, from $X_1$ to $X_n$, we intervene each of them to $1$ and $0$ for $c_0 T^{1/2}$ times respectively. We denote the value of $X$ in the $t^{th}$ round by $X^{(t)}$. For every two variables $X_i,X_j\in{\mathbf X}\backslash\{X_1\}$, if \begin{equation}
% \resizebox{.91\linewidth}{!}{$
%             \displaystyle
    \frac{1}{c_0\sqrt{T}}\sum_{k=1}^{c_0\sqrt{T}}\left(X_j^{\left(2ic_0\sqrt{T}+k\right)}-X_j^{\left((2i+1)c_0\sqrt{T}+k\right)}\right)>c_1T^{-\frac{1}{5}},\label{eq.order}
    % $}
\end{equation}
we set $X_i$ as an ancestor of $X_j$.
%\wei{Explain why intuitively $X_j$ is a descendant of $X_i$?}
Here, $X_j^{\left(2ic_0\sqrt{T}+k\right)}$'s with $k\in[c_0\sqrt{T}]$ are the values of $X_j$ in the rounds that $\doi(X_i=1)$ is chosen; $X_j^{\left((2i+1)c_0\sqrt{T}+k\right)},k\in[c_0\sqrt{T}]$ are the values of $X_j$ in the rounds that $\doi(X_i=0)$ is chosen. Specifically, if $X_i$ is not an ancestor of $X_j$, the value of $X_j$ is not impacted by intervention on $X_i$. 
Simultaneously, if $X_i\in\Pa(X_j)$, the value of $X_j$ is notably impacted by $\doi(X_i)$ so the difference of $X_j$ under $\doi(X_i=1),\doi(X_i=0)$ can be used as a discriminator for 
	the ancestor-descendant relationship between $X_i$ and $X_j$. 
This is formally shown by Lemma~\ref{lemma.dodifference}. 

\begin{algorithm}[t]
\caption{BGLM-OFU-Unknown for BGLM CCB Problem}
\label{alg:glm-ofu}
\begin{algorithmic}[1]
\STATE {\bfseries Input:}
Graph $G=({\bX}\cup\{Y\},E)$, action set $\mathcal{A}$, parameters $L_{f_X}^{(1)},L_{f_X}^{(2)},\kappa,\zeta$ in Assumption~\ref{asm:glm_3}, \ref{asm:glm_2} and \ref{asm:glm_4}, $c$ in Lecu\'{e} and Mendelson's inequality \citep{nie2022matrix}, positive constants $c_0$ and $c_1$ for initialization phase such that $c_0\sqrt{T}\in\mathbb{N}^+$.
\STATE /* Initialization Phase: */
\STATE Initialize $T_0\leftarrow 2(n-1)c_0T^{1/2}$.
\STATE Do each intervention among $\doi(X_2=1), \doi(X_2=0), \ldots, \doi(X_n=1), \doi(X_n=0)$ for $c_0T^{1/2}$ times in order and observe the feedback $({\mathbf X}_t,Y_t)$, $1\leq t\leq T_0$.
\STATE Compute the ancestors $\widehat{\Anc}(X)$, $X \in \bX\cup \{Y\}$ by $\text{BGLM-Ancestors}((\bX_1,Y_1),\ldots,(\bX_{T_0},Y_{T_0}),c_0, c_1)$ (see Algorithm~\ref{alg:bglm-order}).
\STATE /* Parameters Initialization: */
\STATE Initialize $\delta\leftarrow\frac{1}{3n\sqrt{T}}$, $R\leftarrow\lceil\frac{512n(L_{f_X}^{(2)})^2}{\kappa^4}(n^2+\ln\frac{1}{\delta})\rceil$, $T_1\leftarrow T_0+\max\left\{\frac{c}{\zeta^2}\ln\frac{1}{\delta},\frac{(8n^2-6)R}{\zeta}\right\}$ and $\rho\leftarrow\frac{3}{\kappa}\sqrt{\log(1/\delta)}$.

\STATE \label{alg:secondinit}Do no intervention on BGLM $G$ for $T_1-T_0$ rounds and observe feedback $(\bX_t,Y_t),T_0+1\leq t\leq T_1$.
\STATE /* Iterative Phase: */
\FOR{$t=T_1+1,T_1+2,\ldots,T$}
\STATE $\{\hat{\btheta}_{t-1,X},M_{t-1,X}\}_{X\in {\bX}\cup\{Y\}}=
	\text{BGLM-Estimate}((\bX_1,Y_1),\ldots,(\bX_{t-1},Y_{t-1}))$ (see Algorithm \ref{alg:glm-est} in Appendix~\ref{sec.bglmestimate}).
\STATE Compute the confidence ellipsoid $\mathcal{C}_{t,X}=\{\btheta_X'\in[0,1]^{|\widehat{\Anc}(X)|}:\left\|\btheta_X'-\hat{\btheta}_{t-1,X}\right\|_{M_{t-1,X}}\leq\rho\}$ for any node $X\in{\bX}\cup\{Y\}$.
\STATE \label{alg:OFUargmax} Adopt $\argmax_{\doi(\bS=\bs)\in\mathcal{A}, \btheta'_{t,X} \in \mathcal{C}_{t,X} } \E[Y| \doi({\bS}=\bs)]$ as $({\bS_t}, \bs_t, \tilde{\btheta}_t)$.
\STATE Intervene all the nodes in ${\bS}_t$ to $\bs_t$ and observe the feedback $(\bX_t,Y_t)$.
\ENDFOR
\end{algorithmic}
\end{algorithm}

\begin{algorithm}[t]
\caption{BGLM-Ancestors}
\label{alg:bglm-order}
\begin{algorithmic}[1]
\STATE {\bfseries Input:}
Observations $(\bX_1,Y_1),\ldots,(\bX_{T_0},Y_{T_0})$, positive constants
$c_0$ and $c_1$.
\STATE {\bfseries Output:} $\widehat{\Anc}(X)$, ancestors of $X$, $X \in \bX\cup \{Y\}$. 
\STATE For all $X \in \bX, \widehat{\Anc}(X) = \emptyset$, $\widehat{\Anc}(Y) = \bX$.
\FOR{$i\in\{2,3,\ldots,n\}$}
\FOR{$j\in\{2,3,\ldots,n\}\backslash\{i\}$}
\IF{$\sum_{k=1}^{c_0\sqrt{T}}\left(X_j^{\left(2ic_0\sqrt{T}+k\right)}-X_j^{\left((2i+1)c_0\sqrt{T}+k\right)}\right)>c_0c_1T^{3/10}$}
\STATE{Add $X_i$ into $\widehat{\Anc}(X_j)$.} 
\ENDIF
\ENDFOR
\ENDFOR
\STATE{Recompute the transitive closure of $\widehat{\Anc}(\cdot)$, i.e., if $X_i \in \widehat{\Anc}(X_j)$ and $X_j \in \widehat{\Anc}(X_\ell)$, then add $X_i$ to $\widehat{\Anc}(X_\ell)$.}
\end{algorithmic}
\end{algorithm}

\begin{restatable}{lemma}{lemmadodiff}
	\label{lemma.dodifference}
	Let $G$ be a BGLM with parameter $\btheta^*$ that satisfies Assumption \ref{asm:glm_2}. 
	Recall that $\theta^*_{\min} = \min_{(X',X)\in E}\theta^*_{X',X}$.
	%    We denote the smallest nonzero term in $\btheta$ by $\theta_{\min}$. 
	If $X_i\in\Pa(X_j)$, we have $\mathbb{E}[X_j|\doi(X_i=1)]-\mathbb{E}[X_j|\doi(X_i=0)]\geq \kappa\theta^*_{X_i,X_j}\geq\kappa\theta^*_{\min}$; if $X_i$ is not an ancestor of $X_j$, we have $\mathbb{E}[X_j|\doi(X_i=1)]=\mathbb{E}[X_j|\doi(X_i=0)]$.
\end{restatable}

We use the above idea to implement the procedure in Algorithm~\ref{alg:bglm-order}, and then
put this procedure in the initial phase and integrate this step into BGLM-OFU proposed by \cite{feng2022combinatorial}, 
to obtain our main algorithm, BGLM-OFU-Unknown (Algorithm~\ref{alg:glm-ofu}).

Notice that each term in Eq.~\eqref{eq.order} is a random sample of $\mathbb{E}[X_j|\doi(X_i=1)]-\mathbb{E}[X_j|\doi(X_i=0)]$, 
	which means that the left-hand side of Eq.~\eqref{eq.order} is just an estimation of $\mathbb{E}[X_j|\doi(X_i=1)]-\mathbb{E}[X_j|\doi(X_i=0)]$. 
Such expression can be bounded with high probability by concentration inequalities. Hence we can prove that Algorithm~\ref{alg:bglm-order} identifies $X_i\in\Anc(X_j)$ 
%\wei{should this be $X_i\in\Anc(X_j)$?}\nuoya{I think it is $\Anc$.}
with false positive rate and false negative rate both no more than $\exp\left(-\frac{c_0c_1^2T^{1/10}}{2}\right)$ when $\theta^*_{\min}\ge 2c_1 \kappa^{-1}T^{-1/5}$. Formally, we have the following lemma that shows the probability of correctness for Algorithm~\ref{alg:bglm-order}. For completeness, the proof of Lemma~\ref{lemma.correctorder} is put in appendix.

\begin{restatable}[Positive Rate of BGLM-Order]{lemma}{lemmaorder}
\label{lemma.correctorder}
    Suppose Assumption \ref{asm:glm_2} holds for BGLM $G$. In the initialization phase of Algorithm~\ref{alg:glm-ofu}, Algorithm~\ref{alg:bglm-order} finds a consistent ancestor-descendant relationship for $G$ with probability no less than $1-2\binom{n-1}{2}\exp\left(-\frac{c_0c_1^2T^{1/10}}{2}\right)$ when $\theta^*_{\min}\ge 2c_1 \kappa^{-1}T^{-1/5}$.
\end{restatable}

We refer to the condition $\theta^*_{\min}\ge 2c_1 \kappa^{-1}T^{-1/5}$ in this lemma as {\em weight gap assumption}. The number of initialization rounds in Algorithm~\ref{alg:glm-ofu} is $O(\sqrt{T})$. According to Lemma~\ref{lemma.correctorder}, the expected regret contributed by incorrectness of the ancestor-descendant relationship does not exceed $O\left(T\exp\left(-\frac{c_0c_1^2T^{1/10}}{2}\right)\right)=o(\sqrt{T})$. 
Therefore, after adding the initialization, the expected regret of BGLM-OFU-Unknown increases by no more than $o(\sqrt{T})$ over BGLM-OFU (Algorithm~1 in \cite{feng2022combinatorial}). Similar to BGLM-OFU, during the iterative phase, MLE is employed to estimate all the parameters. Simultaneously, a pair oracle is utilized to identify the optimal parameter configuration and intervention set within the confidence ellipsoid. Thus we have the following theorem to show the regret of Algorithm~\ref{alg:bglm-order}, which is formally proved in appendix.
\begin{restatable}[Regret Bound of BGLM-OFU-Unknown]{theorem}{bglmregret}
\label{thm.regretbglm}
Under Assumptions \ref{asm:glm_3}, \ref{asm:glm_2} and \ref{asm:glm_4}, the regret of BGLM-OFU-Unknown (Algorithms~\ref{alg:glm-ofu}, \ref{alg:bglm-order} and~\ref{alg:glm-est}) is bounded as
\begin{equation}\begin{aligned}
    \label{equation.regret_glm}
    R(T)&=O\left(\frac{1}{\kappa} n^{\frac{3}{2}} L^{(1)}_{\max} \sqrt{T}\log T\right),
\end{aligned}
\end{equation}
where $L^{(1)}_{\max} = \max_{X\in {\bX}\cup\{Y\}}L_{f_X}^{(1)}$ and the terms of $o(\sqrt{T}\ln T)$ are omitted, and the big $O$ notation holds for $T\geq 32\left(\frac{c_1}{\kappa\theta^*_{\min}}\right)^5$.
\end{restatable}

Compared to \cite{feng2022combinatorial}, Theorem~\ref{thm.regretbglm} has the same asymptotic regret, The only additional assumption is $T\geq 32\left(c_1/(\kappa\theta^*_{\min})\right)^5$. Intuitively, this extra assumption guarantees that we can discover the ancestor-descendant relationship 
	consistent with the true graph.
Our result indicates that not knowing the causal graph does not provide substantial difficulty with the weight gap assumption.

\begin{remark}
Our BGLM-OFU-Unknown regret bound aligns with the regret bound of BGLM-OFU as presented in \cite{feng2022combinatorial}. Furthermore, the leading term $O(\sqrt{T}\ln T)$ is consistent with the regret bounds previously established for combinatorial bandit algorithms \citep{li2020onlineinference, zhang2022online} and causal bandit algorithms \citep{lu2020regret}.

Because Lemma~\ref{lemma.correctorder} requires weight gap assumption, in the proof of this regret bound, we only consider the case of $T\geq 32\left(c_1/(\kappa\theta^*_{\min})\right)^5$. This limitation does not impact the asymptotic big $O$ notation in our regret bound.
 However, when the round number $T$ is not that large, the regret can be linear with respect to $T$. 
 We remove this weight gap assumption in Section~\ref{sec.withoutgap} for the linear model setting. 
 The $c_0$ and $c_1$ are two adjustable constants in practice. 
When $T$ is small, one could try a small $c_0$ to shorten the initialization phase, i.e., to make sure that $T_0\ll T$, and a small $c_1$ to satisfy the weight gap assumption. When $T$ is large, one could consider larger $c_0$ and $c_1$ for a more accurate ancestor-descendant relationship. 
However, because $\theta^*_{\min}$ is unknown, one cannot promise that the weight gap assumption holds by manipulating $c_1$, i.e., $\theta^*_{\min}$ may be too small for any practical $T$ given $c_1$.
\end{remark}

\section{BLM CCB without Graph Skeleton and Weight Gap Assumption}
\label{sec.withoutgap}

In the previous section, we find that if $T>O((\theta^*_{\min})^{-5})$, we can get a valid upper bound. However, in reality, we have two challenges: 
1) We do not know the real value of $\theta^*_{\min}$, and this makes it hard to know when an edge's direction is identified.
2) When $\theta^*_{\min}\to 0$, it makes it very difficult to estimate the graph accurately. 
To solve these challenges, we must both eliminate the dependence of $\theta^*_{\min}$ in our analysis, and think about how the result will be influenced by an inaccurate model. 
In this section, we give a causal bandit algorithm and show that the algorithm can always give $\tilde{O}(T^{2/3})$ regret. This sub-linear regret result shows that the challenge can be solved by some additional techniques.

\begin{algorithm}[t]
\caption{BLM-LR-Unknown for BLM CCB Problem without Weight Gap}
\label{alg:blmlr-nogap}
\begin{algorithmic}[1]
\STATE {\bfseries Input:}
Graph $G=({\bX}\cup\{Y\},E)$, action set $\mathcal{A}$, positive constants $c_0$ and $c_1$ for initialization phase.

\STATE /* Initialization Phase: */
\STATE Initialize $T_0\leftarrow 2(n-1)c_0T^{2/3}\log(T)$.
\STATE Do each intervention among $\doi(X_2=1), \doi(X_2=0), \ldots, \doi(X_n=1), \doi(X_n=0)$ for $c_0T^{2/3}$ times in order and observe the feedback $({\mathbf X}_t,Y_t)$ for $1\leq t\leq T_0$.
\STATE Compute the ancestors $\widehat{\Anc}(X)$, $X \in \bX\cup \{Y\}$ by $\text{Nogap-BLM-Ancestors}((\bX_1,Y_1),\ldots,(\bX_{T_0},Y_{T_0}),c_0$ $, c_1)$ (see Algorithm~\ref{alg:nogap-bglm-order}).

\STATE /* Parameters Initialization: */
\STATE Initialize $\delta\leftarrow\frac{1}{n\sqrt{T}}$, $\rho_t\leftarrow\sqrt{n\log(1+tn)+2\log\frac{1}{\delta}}$ $+\sqrt{n}$ for $t=0,1,2,\ldots,T$, $M_{T_0,X}\leftarrow \mathbf{I}\in\mathbb{R}^{|\widehat{\Anc}(X)|\times |\widehat{\Anc}(X)|}$, $\bb_{T_0,X}\leftarrow {\mathbf 0}^{|\widehat{\Anc}(X)|}$ for all $X\in {\bX}\cup\{Y\}$ and $\hat{\btheta}_{T_0,X}\leftarrow 0\in\mathbb{R}^{|\widehat{\Anc}(X)|}$ for all $X\in {\bX}\cup\{Y\}$.

\STATE /* Iterative Phase: */
\FOR{$t=T_0+1,T_0+2,\ldots,T$}
\STATE Compute the confidence ellipsoid $\mathcal{C}_{t,X}=\{\btheta_X'\in[0,1]^{|\widehat{\Anc}(X)|}:\left\|\btheta_X'-\hat{\btheta}_{t-1,X}\right\|_{M_{t-1,X}}\leq\rho_{t-1}\}$ for any node $X\in{\bX}\cup\{Y\}$.
\STATE Adopt $\argmax_{\doi(\bS=\bs)\subseteq \mathcal{A}, \btheta'_{t,X} \in \mathcal{C}_{t,X} } \E[Y| \doi({\bS}=\bs)]$ as $({\bS_t}, \bs_t, \tilde{\btheta}_t)$.
\STATE Intervene all the nodes in ${\bS}_t$ to $\bs_t$ and observe the feedback $(\bX_t,Y_t)$.
\FOR{$X\in\bX\cup\{Y\}$ }\label{line:LRbegin} 
\STATE Construct data pair $(\bV_{t,X},X^{(t)})$ with $\bV_{t,X}$ the vector of ancestors of $X$ in round $t$, and $X^{(t)}$ the value of $X$ in round $t$ if $X\not\in S_t$.
\STATE $M_{t,X}=M_{t-1,X}+\bV_{t,X}\bV_{t,X}^\intercal$, $\bb_{t,X}=\bb_{t-1,X}+X^{(t)}\bV_{t,X}$, $\hat{\btheta}_{t,X}=M^{-1}_{t,X}\bb_{t,X}$.
\ENDFOR \label{line:LRend}
\ENDFOR
\end{algorithmic}
\end{algorithm}

\begin{algorithm}[t]
\caption{Nogap-BLM-Ancestors}
\label{alg:nogap-bglm-order}
\begin{algorithmic}[1]
\STATE {\bfseries Input:}
Observations $(\bX_1,Y_1),\ldots,(\bX_{T_0},Y_{T_0})$, positive constants
$c_0$ and $c_1$.
\STATE {\bfseries Output:} For all $X \in \bX\cup \{Y\}, \widehat{\Anc}(X)$.
\STATE For all $X \in \bX, \widehat{\Anc}(X) = \emptyset, \widehat{\Anc}(Y) = \bX$.
\FOR{$i\in\{2,3,\ldots,n\}$}
\FOR{$j\in\{2,3,\ldots,n\}\backslash\{i\}$}
\IF{$\sum_{k=1}^{c_0T^{2/3}}\left(X_j^{\left(c_0(2i)T^{2/3}+k\right)}-X_j^{\left(c_0(2i+1)T^{2/3}+k\right)}\right) 
	$ 
 $>c_0c_1T^{1/3}\log(T^2)$}
\STATE{Add $X_i$ into $\widehat{\Anc}(X_j)$.}
\ENDIF
\ENDFOR
\ENDFOR
\STATE{Recompute the transitive closure of $\widehat{\Anc}(\cdot)$.}
%\STATE{Deduce an inaccurate graph for $\bX\cup\{Y\}$.}
\end{algorithmic}
\end{algorithm}

In this section, we consider a special case of BGLM called Binary Linear Model (BLM), where $f_X$ becomes identity function. 
The linear structure allows us to release the Assumption~\ref{asm:glm_3}-\ref{asm:glm_4} \citep{feng2022combinatorial} and analyze the influence of an inaccurate model.

The main algorithm follows the BLM-LR algorithm in \cite{feng2022combinatorial}, which uses linear regression to estimate the weight $\btheta^*$, and the pseudocode is provided in Algorithm~\ref{alg:blmlr-nogap}.
We add a graph discovery process (Algorithm~\ref{alg:nogap-bglm-order})
	in the initialization phase using $O(nT^{2/3}\log T)$ times rather than $O(nT^{1/2})$ in the previous section. 
For any edge $X'\to X$ with weight $\theta^*_{X',X} \ge T^{-1/3}$, with probability at least $1-1/T^2$, 
	we expect to identify the edge's direction within $O(nT^{2/3}\log(T))$ samples for $do(X'=1)$ and $do(X'=0)$ by checking whether the difference $P(X\mid do(X'=1))-P(X\mid do(X'=0))$ is 
	large than $T^{-1/3}$. 
Since the above difference is always larger than $\theta^*_{X',X}$, after the initialization phase, the edge $X'\to X$ will be added to the graph if $\theta^*_{X',X}\ge T^{-1/3}$.

Moreover, if $X'$ is not an ancestor of $X$, we claim that it cannot be a estimated as an ancestor after the initialization phase.
This is because in this case $P(X\mid do(X'=1)) = P(X) = P(X\mid do(X'=0))$. 
Denote the estimated graph $G'$ as the graph with edge $X'\to X$ for all $X' \in \widehat{\Anc}(X)$.  
We then have the following lemma.
\begin{restatable}{lemma}{gaplemma}\label{lemma:gap lemma}
    In Algorithm~\ref{alg:blmlr-nogap}, if the constants $c_0$ and $c_1$ satisfy that $c_0\ge \max\{\frac{1}{c_1^2}, \frac{1}{(1-c_1)^2}\}$, with probability at least $1-(n-1)(n-2)\frac{1}{T^{1/3}}$, after the initialization phase we have
    
    1). If $X'$ is a true parent of $X$ in $G$ with weight $\theta^*_{X',X}\ge T^{-1/3}$, the edge $X'\to X$ will be identified and added to the estimated graph $G'$. 
    
    2). If $X'$ is not an ancestor of $X$ in $G$, $X'\to X$ will not be added into $G'$.  

\end{restatable}
The properties above together provide the analytic basis for the following observation, which plays a key role in our further analysis. 
Denote the estimated accuracy $r = T^{-1/3}$. We know the linear regression for $X$ will be performed on $X$ and all its possible ancestors $\widehat{\Anc}(X)$ we estimated. 
For the true parent node $X'$ in $G$ that is not contained in $\widehat{\Anc}(X)$,  we have $\theta^*_{X',X}\le r$. 
Suppose $\widehat{\Anc}(X) = \{X_1,X_2,\ldots,X_m\}$, and true parents which is not contained in $\widehat{\Anc}(X)$ are $X_{m+1},\ldots,X_{m+k}$. Thus we have $\theta^*_{X_{m+i}, X} \le r$ 
for all $1\le i\le k$.

Also, assume $X_1,\ldots,X_t(t<m)$ are true parents of $X$ in $G$. 
For $X_{m+i}$, by law of total expectation, the expectation of $X$ can be rewritten as
\begin{align*}&\ \ \ \ \ \mathbb{E}[X\mid X_1,\ldots,X_t]\\&=\mathbb{E}_{X_{m+1},\ldots,X_{m+k}}[\mathbb{E}[X\mid X_1,\ldots,X_t, X_{m+1},\ldots,X_{m+k}]]\\&=\mathbb{E}_{X_{m+1},\ldots,X_{m+k}}\left[\sum_{i=1}^t \theta^*_{X_i,X}X_i + \sum_{i=m+1}^{m+k}\theta^*_{X_i,X}X_i\right]\\&=\sum_{i=1}^t \theta^*_{X_i,X}X_i+\sum_{i=m+1}^{m+k}\theta^*_{X_i,X}\mathbb{E}[X_i]=\sum_{i=1}^t \theta^{*'}_{X_i,X}X_i,
\end{align*}
where
\begin{align}
\theta^{*'}_{X_i,X} &= \theta^*_{X_i,X}, \ \ \ i\ge 2,\label{eq:transform1}\\
\theta^{*'}_{X_1,X}&=
\theta^*_{X_1,X}+\sum_{i=m+1}^{m+k}\theta^*_{X_i,X}\mathbb{E}[X_i].\label{eq:transform2}
\end{align}
Eq.~\eqref{eq:transform2} is because $X_1=1$ always holds.
Then we have  $|\theta^{*'}_{X_i,X}-\theta^*_{X_i,X}|\le \sum_{i=m+1}^{m+k}\theta^*_{X_{i}, X}\le kr\le nr$, which shows that the difference between $\btheta'$ and $\btheta$ is small if accuracy $r$ is small. In Eqs.~\eqref{eq:transform1} and~\eqref{eq:transform2}, we employ the linear property of BLMs, which is the reason that we are only able to perform transformations from $\btheta^*$ to ${\btheta^*}'$ for BLMs.
Let model $M'$ represent the model with graph $G'$ with weights $\btheta^{*'}$ defined above. 
The following lemma shows the key observation:
\begin{lemma}\label{lemma: linearregressionequivalent}
The linear regression performed on graph $G'$ in Algorithm~\ref{alg:blmlr-nogap} (lines~\ref{line:LRbegin}--\ref{line:LRend}) gives the estimation $\hat{\btheta'}$ such that 
\begin{equation*}
% \resizebox{.99\linewidth}{!}{$
%             \displaystyle
   \Vert(\hat{\btheta'}_{t,X}-\btheta^{*'}_X)\Vert_{M_{t,X}}\leq\sqrt{n\log (1+tn)+2\log(1/\delta)}+\sqrt{n},
   % $}
\end{equation*}
where $M_{t,X}$ is defined in Algorithm~\ref{alg:blmlr-nogap}. 
\end{lemma}

This lemma shows that, the linear regression performed on the inaccurate estimated linear model $M'$
	is equivalent to the regression for $\btheta^{*'}$. 
Note that this regression only gives us the approximation in some direction with respect to elliptical norm, allowing the variables to be dependent.

Based on claim above, we only need to measure the difference for $\mathbb{E}[Y\mid do(\bS=1)]$ on model $M$ and $M'$. The following lemma shows that the difference between two models can be bounded by our estimated accuracy $r$:

\begin{lemma}\label{lemma:diff}
    $|\mathbb{E}_{M}[Y\mid do(S=\textbf{1})]-\mathbb{E}_{M'}[Y\mid do(S=\textbf{1})]|\le n^2(n+1)r$, where $r$ is the estimated accuracy defined in the start of this section.
\end{lemma}

The Lemma~\ref{lemma:diff} gives us a way to bound our linear regression performance on the estimated model $M'$. 
Suppose our linear regression achieves $O(\sqrt{T})$ regret comparing to $\max_{\bS} \mathbb{E}_{M'}[Y\mid do(\bS=\textbf{1})]$, 
based on our estimated accuracy $r = O(T^{-1/3})$, the regret for optimization error is $O(T^{2/3})$, which is the same order as the initialization phase.
Moreover, it implies that we cannot set $r$ to a larger gap, such as $r=O(T^{-1/2})$, as doing so would result in the initialization phase regret becoming linearly proportional to $T$.

From these two lemmas, we can measure the error for initialization phase. Motivated by Explore-then-Commit framework, 
	we can achieve sublinear regret without the weight gap assumption. 
The detailed proof is provided in Appendix~\ref{appendix:linearregressionequivalent} and Appendix~\ref{appendix:diff}.
\begin{theorem}\label{thm:nogap}
	If $c_0\ge \max\{\frac{1}{c_1^2}, \frac{1}{(1-c_1)^2}\}$, the regret of Algorithm~\ref{alg:blmlr-nogap} running on BLM is upper bounded as  
		$$R(T) = O((n^3T^{2/3})\log T).$$
\end{theorem}
Theorem~\ref{thm:nogap} states the regret of our algorithm without weight gap. 
The leading term of the result is $O(T^{2/3}\log T)$, which has higher order than $O(\sqrt{T}\log T)$, the regret of
	the previous Algorithm~\ref{alg:bglm-order} and the BLM-LR algorithm in \cite{feng2022combinatorial}. 
This degradation in regret bound can be viewed as the cost of removing the weight
gap assumption, which makes the accurate discovery of the causal graph extremely difficult.  For a detailed discussion of the weight gap assumption, interested readers can refer to Appendix~\ref{app.explanationweightgap}.
How to devise a $O(\sqrt{T}\log T)$ algorithm without weight gap assumption is still an open problem.

Using the transformation in Section 5.1 in \cite{feng2022combinatorial}, this algorithm can also work with hidden variables. The model our algorithm can work on allows hidden variables but disallows the graph structure where a hidden node has two paths to $X_i$ and $X_i$'s descendant $X_j$ and the paths contain only hidden nodes except the end points $X_i$ and $X_j$.

Moreover, observe that Algorithms~\ref{alg:glm-ofu} and \ref{alg:blmlr-nogap} necessitate prior knowledge of the horizon $T$. To circumvent this constraint, the "Doubling Trick" can be employed, converting our algorithms into anytime algorithms without compromising the regret bounds.

\section{Future Work}

This paper is the first theoretical study on causal bandits without the graph skeleton. 
There are many future directions to extend this work. 
We believe that similar initialization methods and proof techniques can be used to design causal bandits algorithms for other parametric models without the skeleton, like linear structural equation models (SEM). 
Moreover, how to provide an algorithm with $\widetilde{O}(\sqrt{T})$ regret without weight gap assumption is interesting and still open. One possibility is to investigate the utilization of feedback during the iterative phase of the BGLM-OFU-Unknown algorithm in order to enhance graph structure identification, which could potentially lead to an improved regret bound.

\section*{Acknowledgements}
The authors would like to thank the anonymous reviewers for their valuable comments and constructive feedback.

\bibliography{acml24}

\begin{thebibliography}{40}
\providecommand{\natexlab}[1]{#1}
\providecommand{\url}[1]{\texttt{#1}}
\expandafter\ifx\csname urlstyle\endcsname\relax
  \providecommand{\doi}[1]{doi: #1}\else
  \providecommand{\doi}{doi: \begingroup \urlstyle{rm}\Url}\fi

\bibitem[Abbasi-Yadkori et~al.(2011)Abbasi-Yadkori, P{\'a}l, and Szepesv{\'a}ri]{abbasi2011improved}
Yasin Abbasi-Yadkori, D{\'a}vid P{\'a}l, and Csaba Szepesv{\'a}ri.
\newblock Improved algorithms for linear stochastic bandits.
\newblock \emph{Advances in neural information processing systems}, 24, 2011.

\bibitem[Abbasi-Yadkori et~al.(2012)Abbasi-Yadkori, Pal, and Szepesvari]{abbasi2012online}
Yasin Abbasi-Yadkori, David Pal, and Csaba Szepesvari.
\newblock Online-to-confidence-set conversions and application to sparse stochastic bandits.
\newblock In \emph{Artificial Intelligence and Statistics}, pages 1--9. PMLR, 2012.

\bibitem[Agrawal and Goyal(2013)]{agrawal2013thompson}
Shipra Agrawal and Navin Goyal.
\newblock Thompson sampling for contextual bandits with linear payoffs.
\newblock In \emph{International conference on machine learning}, pages 127--135. PMLR, 2013.

\bibitem[Auer et~al.(2002)Auer, Cesa-Bianchi, and Fischer]{auer2002finite}
Peter Auer, Nicolo Cesa-Bianchi, and Paul Fischer.
\newblock Finite-time analysis of the multiarmed bandit problem.
\newblock \emph{Machine learning}, 47\penalty0 (2):\penalty0 235--256, 2002.

\bibitem[Bonferroni(1936)]{bonferroni1936teoria}
Carlo Bonferroni.
\newblock Teoria statistica delle classi e calcolo delle probabilita.
\newblock \emph{Pubblicazioni del R Istituto Superiore di Scienze Economiche e Commericiali di Firenze}, 8:\penalty0 3--62, 1936.

\bibitem[Bubeck et~al.(2012)Bubeck, Cesa-Bianchi, et~al.]{bubeck2012regret}
S{\'e}bastien Bubeck, Nicolo Cesa-Bianchi, et~al.
\newblock Regret analysis of stochastic and nonstochastic multi-armed bandit problems.
\newblock \emph{Foundations and Trends{\textregistered} in Machine Learning}, 5\penalty0 (1):\penalty0 1--122, 2012.

\bibitem[Chen et~al.(2013)Chen, Wang, and Yuan]{chen2013combinatorial}
Wei Chen, Yajun Wang, and Yang Yuan.
\newblock Combinatorial multi-armed bandit: General framework and applications.
\newblock In \emph{International conference on machine learning}, pages 151--159. PMLR, 2013.

\bibitem[Chen et~al.(2016)Chen, Wang, Yuan, and Wang]{chen2016combinatorial}
Wei Chen, Yajun Wang, Yang Yuan, and Qinshi Wang.
\newblock Combinatorial multi-armed bandit and its extension to probabilistically triggered arms.
\newblock \emph{The Journal of Machine Learning Research}, 17\penalty0 (1):\penalty0 1746--1778, 2016.

\bibitem[Feng and Chen(2021)]{feng2021causal}
Shi Feng and Wei Chen.
\newblock Causal inference for influence propagation—identifiability of the independent cascade model.
\newblock In \emph{International Conference on Computational Data and Social Networks}, pages 15--26. Springer, 2021.

\bibitem[Feng and Chen(2022)]{feng2022combinatorialarxiv}
Shi Feng and Wei Chen.
\newblock Combinatorial causal bandits.
\newblock \emph{arXiv preprint arXiv:2206.01995}, 2022.

\bibitem[Feng and Chen(2023)]{feng2022combinatorial}
Shi Feng and Wei Chen.
\newblock Combinatorial causal bandits.
\newblock In \emph{Proceedings of the 37th AAAI Conference on Artificial Intelligence (AAAI)}, February 2023.
\newblock URL \url{https://www.microsoft.com/en-us/research/publication/combinatorial-causal-bandits/}.

\bibitem[Feng et~al.(2024)Feng, Xiong, Zhang, and Chen]{feng2024correction}
Shi Feng, Nuoya Xiong, Zhijie Zhang, and Wei Chen.
\newblock A correction of pseudo log-likelihood method.
\newblock \emph{arXiv preprint arXiv:2403.18127}, 2024.

\bibitem[Galles and Pearl(1995)]{pearl1995testing}
David Galles and Judea Pearl.
\newblock Testing identifiability of causal effects.
\newblock In \emph{Proceedings of the Eleventh Conference on Uncertainty in Artificial Intelligence}, UAI'95, page 185–195, San Francisco, CA, USA, 1995. Morgan Kaufmann Publishers Inc.
\newblock ISBN 1558603859.

\bibitem[G{\'a}mez et~al.(2013)G{\'a}mez, Moral, and Cerdan]{gamez2013advances}
Jos{\'e}~A G{\'a}mez, Seraf{\'\i}n Moral, and Antonio~Salmer{\'o}n Cerdan.
\newblock \emph{Advances in Bayesian networks}, volume 146.
\newblock Springer, 2013.

\bibitem[Hoeffding(1994)]{hoeffding1994probability}
Wassily Hoeffding.
\newblock Probability inequalities for sums of bounded random variables.
\newblock In \emph{The collected works of Wassily Hoeffding}, pages 409--426. Springer, 1994.

\bibitem[Kempe et~al.(2003)Kempe, Kleinberg, and Tardos]{kempe2003maximizing}
David Kempe, Jon Kleinberg, and {\'E}va Tardos.
\newblock Maximizing the spread of influence through a social network.
\newblock In \emph{Proceedings of the ninth ACM SIGKDD international conference on Knowledge discovery and data mining}, pages 137--146, 2003.

\bibitem[Konobeev et~al.(2023)Konobeev, Etesami, and Kiyavash]{konobeev2023causal}
Mikhail Konobeev, Jalal Etesami, and Negar Kiyavash.
\newblock Causal bandits without graph learning.
\newblock \emph{arXiv preprint arXiv:2301.11401}, 2023.

\bibitem[Lattimore et~al.(2016)Lattimore, Lattimore, and Reid]{lattimore2016causal}
Finnian Lattimore, Tor Lattimore, and Mark~D Reid.
\newblock Causal bandits: Learning good interventions via causal inference.
\newblock \emph{Advances in Neural Information Processing Systems}, 29, 2016.

\bibitem[Lee and Bareinboim(2018)]{lee2018structural}
Sanghack Lee and Elias Bareinboim.
\newblock Structural causal bandits: where to intervene?
\newblock \emph{Advances in Neural Information Processing Systems}, 31, 2018.

\bibitem[Lee and Bareinboim(2019)]{lee2019structural}
Sanghack Lee and Elias Bareinboim.
\newblock Structural causal bandits with non-manipulable variables.
\newblock In \emph{Proceedings of the AAAI Conference on Artificial Intelligence}, volume~33, pages 4164--4172, 2019.

\bibitem[Lee and Bareinboim(2020)]{lee2020characterizing}
Sanghack Lee and Elias Bareinboim.
\newblock Characterizing optimal mixed policies: Where to intervene and what to observe.
\newblock \emph{Advances in neural information processing systems}, 33:\penalty0 8565--8576, 2020.

\bibitem[Li et~al.(2017)Li, Lu, and Zhou]{li2017provably}
Lihong Li, Yu~Lu, and Dengyong Zhou.
\newblock Provably optimal algorithms for generalized linear contextual bandits.
\newblock In \emph{International Conference on Machine Learning}, pages 2071--2080. PMLR, 2017.

\bibitem[Li et~al.(2020)Li, Kong, Tang, Li, and Chen]{li2020onlineinference}
Shuai Li, Fang Kong, Kejie Tang, Qizhi Li, and Wei Chen.
\newblock Online influence maximization under linear threshold model.
\newblock \emph{Advances in Neural Information Processing Systems}, 33:\penalty0 1192--1204, 2020.

\bibitem[Lu et~al.(2020)Lu, Meisami, Tewari, and Yan]{lu2020regret}
Yangyi Lu, Amirhossein Meisami, Ambuj Tewari, and William Yan.
\newblock Regret analysis of bandit problems with causal background knowledge.
\newblock In \emph{Conference on Uncertainty in Artificial Intelligence}, pages 141--150. PMLR, 2020.

\bibitem[Lu et~al.(2021)Lu, Meisami, and Tewari]{lu2021causal}
Yangyi Lu, Amirhossein Meisami, and Ambuj Tewari.
\newblock Causal bandits with unknown graph structure.
\newblock \emph{Advances in Neural Information Processing Systems}, 34:\penalty0 24817--24828, 2021.

\bibitem[Maiti et~al.(2021)Maiti, Nair, and Sinha]{maiti2021causal}
Aurghya Maiti, Vineet Nair, and Gaurav Sinha.
\newblock Causal bandits on general graphs.
\newblock \emph{arXiv preprint arXiv:2107.02772}, 2021.

\bibitem[Malek et~al.(2023)Malek, Aglietti, and Chiappa]{malek2023additive}
Alan Malek, Virginia Aglietti, and Silvia Chiappa.
\newblock Additive causal bandits with unknown graph.
\newblock In \emph{International Conference on Machine Learning}, pages 23574--23589. PMLR, 2023.

\bibitem[Nair et~al.(2021)Nair, Patil, and Sinha]{nair2021budgeted}
Vineet Nair, Vishakha Patil, and Gaurav Sinha.
\newblock Budgeted and non-budgeted causal bandits.
\newblock In \emph{International Conference on Artificial Intelligence and Statistics}, pages 2017--2025. PMLR, 2021.

\bibitem[Nie(2022)]{nie2022matrix}
Zipei Nie.
\newblock Matrix anti-concentration inequalities with applications.
\newblock In \emph{Proceedings of the 54th Annual ACM SIGACT Symposium on Theory of Computing}, pages 568--581, 2022.

\bibitem[Pearl(2009{\natexlab{a}})]{pearl2009causal}
Judea Pearl.
\newblock {Causal inference in statistics: An overview}.
\newblock \emph{Statistics Surveys}, 3\penalty0 (none):\penalty0 96 -- 146, 2009{\natexlab{a}}.
\newblock \doi{10.1214/09-SS057}.
\newblock URL \url{https://doi.org/10.1214/09-SS057}.

\bibitem[Pearl(2009{\natexlab{b}})]{pearl2009causality}
Judea Pearl.
\newblock \emph{Causality}.
\newblock Cambridge university press, 2009{\natexlab{b}}.

\bibitem[Pearl(2012)]{pearl2012do}
Judea Pearl.
\newblock The do-calculus revisited.
\newblock In \emph{Proceedings of the Twenty-Eighth Conference on Uncertainty in Artificial Intelligence}, UAI'12, page 3–11, Arlington, Virginia, USA, 2012. AUAI Press.
\newblock ISBN 9780974903989.

\bibitem[Pearl and Mackenzie(2018)]{pearl2018book}
Judea Pearl and Dana Mackenzie.
\newblock \emph{The book of why: the new science of cause and effect}.
\newblock Basic books, 2018.

\bibitem[Robbins(1952)]{robbins1952some}
Herbert Robbins.
\newblock Some aspects of the sequential design of experiments.
\newblock \emph{Bulletin of the American Mathematical Society}, 58\penalty0 (5):\penalty0 527--535, 1952.

\bibitem[Sen et~al.(2017)Sen, Shanmugam, Dimakis, and Shakkottai]{sen2017identifying}
Rajat Sen, Karthikeyan Shanmugam, Alexandros~G Dimakis, and Sanjay Shakkottai.
\newblock Identifying best interventions through online importance sampling.
\newblock In \emph{International Conference on Machine Learning}, pages 3057--3066. PMLR, 2017.

\bibitem[Slivkins et~al.(2019)]{slivkins2019introduction}
Aleksandrs Slivkins et~al.
\newblock Introduction to multi-armed bandits.
\newblock \emph{Foundations and Trends{\textregistered} in Machine Learning}, 12\penalty0 (1-2):\penalty0 1--286, 2019.

\bibitem[Varici et~al.(2022)Varici, Shanmugam, Sattigeri, and Tajer]{varici2022causal}
Burak Varici, Karthikeyan Shanmugam, Prasanna Sattigeri, and Ali Tajer.
\newblock Causal bandits for linear structural equation models.
\newblock \emph{arXiv preprint arXiv:2208.12764}, 2022.

\bibitem[Xiong and Chen(2023)]{xiong2022pure}
Nuoya Xiong and Wei Chen.
\newblock Combinatorial pure exploration of causal bandits.
\newblock In \emph{Proceedings of the 11th International Conference on Learning Representations (ICLR)}, May 2023.

\bibitem[Yabe et~al.(2018)Yabe, Hatano, Sumita, Ito, Kakimura, Fukunaga, and Kawarabayashi]{yabe2018causal}
Akihiro Yabe, Daisuke Hatano, Hanna Sumita, Shinji Ito, Naonori Kakimura, Takuro Fukunaga, and Ken-ichi Kawarabayashi.
\newblock Causal bandits with propagating inference.
\newblock In \emph{International Conference on Machine Learning}, pages 5512--5520. PMLR, 2018.

\bibitem[Zhang et~al.(2022)Zhang, Chen, Sun, and Zhang]{zhang2022online}
Zhijie Zhang, Wei Chen, Xiaoming Sun, and Jialin Zhang.
\newblock Online influence maximization with node-level feedback using standard offline oracles.
\newblock In \emph{Proceedings of the AAAI Conference on Artificial Intelligence}, volume~36, pages 9153--9161, 2022.

\end{thebibliography}

\appendix
%%%%%%%%%%%%%%%%%%%%%%%%%%%%%%%%%%%%%%%%%%%%%%%%%%%%%%%%%%%%%%%%%%%%%%%%%%%%%%%
%%%%%%%%%%%%%%%%%%%%%%%%%%%%%%%%%%%%%%%%%%%%%%%%%%%%%%%%%%%%%%%%%%%%%%%%%%%%%%%
% APPENDIX
%%%%%%%%%%%%%%%%%%%%%%%%%%%%%%%%%%%%%%%%%%%%%%%%%%%%%%%%%%%%%%%%%%%%%%%%%%%%%%%
%%%%%%%%%%%%%%%%%%%%%%%%%%%%%%%%%%%%%%%%%%%%%%%%%%%%%%%%%%%%%%%%%%%%%%%%%%%%%%%
\newpage
\appendix
\onecolumn
\nolinenumbers

\section*{Appendix}
\section{Conversion of Function \texorpdfstring{$f_X$}{Lg}}
\label{app.conversion}

\subsection{Conversion to \texorpdfstring{$g_X$}{Lg} such that \texorpdfstring{$\lim_{x\rightarrow+\infty}g_X(x)=+\infty$}{Lg}}
In this section, we firstly prove that any monotone increasing function $f_X$ that satisfies Assumptions~\ref{asm:glm_3} and \ref{asm:glm_2} can be converted to a function $g_X$ such that the conversion does not impact the propagation of BGLM, i.e., $f_X(x)=g_X(x)$ for $x\in [0, |\Pa(X)|]$, $\lim_{x\rightarrow+\infty}g_X(x)=+\infty$, $g_X$ is twice differentiable and Assumptions~\ref{asm:glm_3} and \ref{asm:glm_2} still hold.

On one hand, if for all $x\geq 2|\Pa(X)|$, $f''_X(x)\geq0$, then $f_X(x)\geq f_X(2|\Pa(X)|)+f'_X(2|\Pa(X)|)(x-2|\Pa(X)|)$, which already satisfies $\lim_{x\rightarrow+\infty}f_X(x)=+\infty$. In this case, no conversion is needed (let $g_X\equiv f_X$). On another hand, we can find a $x^*\geq 2|\Pa(X)|$ such that $f''_X(x^*)<0$.

We define the conversion as\begin{align*}
    g_X(x)=\begin{cases}
        f_X(x)&x\leq x^*\\
        f_X(x^*)+\frac{f'_X(x^*)^2}{f''_X(x^*)}\ln\left(-\frac{f'_X(x^*)}{f''_X(x^*)}\right)-\frac{f'_X(x^*)^2}{f''_X(x^*)}\ln\left(x-x^*-\frac{f'_X(x^*)}{f''_X(x^*)}\right)&x>x^*
    \end{cases}.
\end{align*}

During the propagation of the BGLM, the input of $f_X$ is $\Pa(X)\cdot\btheta_X^*$, which is in the range $[0,|\Pa(X)|]\subseteq [0,x^*]$. Hence, when we replace $f_X$ by $g_X$ in the BGLM, the propagation is not impacted.

Moreover, we can compute that \begin{align*}
    g'_X(x)=\begin{cases}
        f'_X(x)&x\leq x^*\\
        -\frac{f'_X(x^*)^2}{f''_X(x^*)\left(x-x^*-\frac{f'_X(x^*)}{f''_X(x^*)}\right)}&x>x^*
    \end{cases},
\end{align*}
and\begin{align*}
    g''_X(x)=\begin{cases}
        f''_X(x)&x\leq x^*\\
        \frac{f'_X(x^*)^2}{f''_X(x^*)\left(x-x^*-\frac{f'_X(x^*)}{f''_X(x^*)}\right)^2}&x>x^*
        \end{cases}.
\end{align*}
Therefore, we have $\lim_{x\rightarrow{x^*}^+}g_X(x)=f_X(x^*)$ and $\lim_{x\rightarrow{x^*}^-}g_X(x)=f_X(x^*)$. Hence, $g_X$ is continuous. Moreover, $\lim_{x\rightarrow{x^*}^+}g'_X(x)=f'_X(x^*)=\lim_{x\rightarrow{x^*}^-}g'_X(x)$ and $\lim_{x\rightarrow{x^*}^+}g''_X(x)=f''_X(x^*)=\lim_{x\rightarrow{x^*}^-}g''_X(x)$, so $g_X(x)$ is twice differentiable and $g''_X$ is continuous.

Now we only need to verify Assumptions~\ref{asm:glm_3} and \ref{asm:glm_2}. Firstly, when $x>x^*$, we have $g'_X(x)<g'_X(x^*)=f'_X(x^*)\leq L_{f_X}^{(1)}$ and $g''_X(x)<g''_X(x^*)=f''_X(x^*)\leq L_{f_X}^{(2)}$, so Assumption~\ref{asm:glm_3} holds. Secondly, $\max_{\bv\in[0,1]^{|\Pa(X)|},\left\|\btheta-\btheta_X^*\right\|\leq 1}\bv\cdot \btheta\leq 2|\Pa(X)|\leq x^*$, so the conversion does not impact the value of $\kappa$. Until now, we complete the conversion.

\subsection{Conversion to \texorpdfstring{$h_X$}{Lg} such that \texorpdfstring{$\lim_{x\rightarrow-\infty}h_X(x)=-\infty$}{Lg} and \texorpdfstring{$\lim_{x\rightarrow+\infty}h_X(x)=+\infty$}{Lg}}

Then we prove that the monotone increasing function $g_X$ that satisfies Assumptions~\ref{asm:glm_3} and \ref{asm:glm_2} can be converted to a function $h_X$ such that the conversion does not impact the propagation of BGLM, i.e., $g_X(x)=h_X(x)$ for $x\in [0, |\Pa(X)|]$, $\lim_{x\rightarrow-\infty}h_X(x)=-\infty$, $\lim_{x\rightarrow+\infty}h_X(x)=+\infty$, $h_X$ is twice differentiable and Assumptions~\ref{asm:glm_3} and \ref{asm:glm_2} still hold.

On one hand, if for all $x\leq -|\Pa(X)|$, $f''_X(x)\leq0$, then $f_X(x)\leq f_X(-|\Pa(X)|)-f'_X(-|\Pa(X)|)(-x-|\Pa(X)|)$, which already satisfies $\lim_{x\rightarrow-\infty}f_X(x)=-\infty$. In this case, no conversion is needed (let $h_X\equiv g_X$). On another hand, we can find a $x^*\leq -|\Pa(X)|$ such that $f''_X(x^*)>0$.

We define the conversion as\begin{align*}
    h_X(x)=\begin{cases}
        g_X(x)&x\geq x^*\\
        g_X(x^*)-\frac{g'_X(x^*)^2}{g''_X(x^*)}\ln\left(-\frac{g'_X(x^*)}{g''_X(x^*)}\right)+\frac{g'_X(x^*)^2}{g''_X(x^*)}\ln\left(-x+x^*+\frac{g'_X(x^*)}{g''_X(x^*)}\right)&x<x^*
    \end{cases}.
\end{align*}

During the propagation of the BGLM, the input of $g_X$ is $\Pa(X)\cdot\btheta_X^*$, which is in the range $[0,|\Pa(X)|]\subseteq [0,x^*]$. Hence, when we replace $g_X$ by $h_X$ in the BGLM, the propagation is not impacted.

Moreover, we can compute that \begin{align*}
    h'_X(x)=\begin{cases}
        g'_X(x)&x\geq x^*\\
        \frac{g'_X(x^*)^2}{g''_X(x^*)\left(-x+x^*+\frac{g'_X(x^*)}{g''_X(x^*)}\right)}&x<x^*
    \end{cases},
\end{align*}
and\begin{align*}
    h''_X(x)=\begin{cases}
        g''_X(x)&x\geq x^*\\
        \frac{g'_X(x^*)^2}{g''_X(x^*)\left(x-x^*-\frac{g'_X(x^*)}{g''_X(x^*)}\right)^2}&x<x^*
        \end{cases}.
\end{align*}
Therefore, we have $\lim_{x\rightarrow{x^*}^+}h_X(x)=g_X(x^*)$ and $\lim_{x\rightarrow{x^*}^-}h_X(x)=g_X(x^*)$. Hence, $h_X$ is continuous. Moreover, $\lim_{x\rightarrow{x^*}^+}h'_X(x)=g'_X(x^*)=\lim_{x\rightarrow{x^*}^-}h'_X(x)$ and $\lim_{x\rightarrow{x^*}^+}h''_X(x)=g''_X(x^*)=\lim_{x\rightarrow{x^*}^-}h''_X(x)$, so $h_X(x)$ is twice differentiable and $h''_X$ is continuous.

Now we only need to verify Assumptions~\ref{asm:glm_3} and \ref{asm:glm_2}. Firstly, when $x<x^*$, we have $h'_X(x)<h'_X(x^*)=g'_X(x^*)\leq L_{f_X}^{(1)}$ and $h''_X(x)<h''_X(x^*)=g''_X(x^*)\leq L_{f_X}^{(2)}$, so Assumption~\ref{asm:glm_3} holds. Secondly, $\min_{\bv\in[0,1]^{|\Pa(X)|},\left\|\btheta-\btheta_X^*\right\|\leq 1}\bv\cdot \btheta\geq -|\Pa(X)|\geq x^*$, so the conversion does not impact the value of $\kappa$. Until now, we complete the conversion.

In conclusion we have found a conversion from $f_X$ to $h_X$ such that the conversion does not impact the propagation of BGLM, i.e., $h_X(x)=f_X(x)$ for $x\in [0, |\Pa(X)|]$, $\text{Range}(h_X)=\mathbb{R}$, $h_X$ is twice differentiable and Assumptions~\ref{asm:glm_3} and \ref{asm:glm_2} still hold.

\section{Pseudocode of Algorithm \ref{alg:glm-est}}
\label{sec.bglmestimate}

\begin{algorithm}[htbp]
\caption{BGLM-Estimate}
\label{alg:glm-est}
\begin{algorithmic}[1]
\STATE {\bfseries Input:}
All observations $((\bX_1,Y_1),\ldots,(\bX_{t},Y_{t}))$ until round $t$.
\STATE {\bfseries Output:}
$\{\hat{\btheta}_{t,X},M_{t,X}\}_{X\in {\bX}\cup\{Y\}}$
\STATE For each $X\in {\bX}\cup\{Y\}$, $i \in [t]$, construct data pair $(\bV_{i,X},X^{(i)})$
	with $\bV_{i,X}$ the vector of ancestors of $X$ in round $i$, and $X^{(i)}$ the value of $X$ in round $i$ if $X\not\in S_i$.
\FOR{$X\in{\bX}\cup\{Y\}$}
\STATE Calculate the maximum-likelihood estimator $\hat{\btheta}_{t,X}$ by solving the equation $ \sum_{i=1}^{t}(X^{(i)}-f_X(\bV_{i,X}^\intercal \btheta_X))\bV_{i,X}=0$.
	\label{alg:pseudolikelihood}
\STATE $M_{t,X}=\sum_{i=1}^t \bV_{i,X} \bV_{i,X}^\intercal$.
\ENDFOR
\end{algorithmic}
\end{algorithm}

Here, we want to give a lemma to clarify why we can always find a solution for equation $ \sum_{i=1}^{t}(X^{(i)}-f_X(\bV_{i,X}^\intercal \btheta_X))\bV_{i,X}=0$ in Line~\ref{alg:pseudolikelihood} of Algorithm~\ref{alg:glm-est}.

\begin{lemma}
    When $\lim_{x\rightarrow+\infty}f_X(x)=+\infty$, $\lim_{x\rightarrow-\infty}f_X(x)=-\infty$, and $f_X$ is monotone increasing, equation $ \sum_{i=1}^{t}(X^{(i)}-f_X(\bV_{i,X}^\intercal \btheta_X))\bV_{i,X}=0$ has a solution. 
\end{lemma}

\begin{proof}
    We define $m_X(x)$ as\begin{align*}
        m_X(x)=\begin{cases}
            \int_0^xf_X(c)\text{d}c&x\geq 0\\
            -\int_x^0f_X(c)\text{d}c&x<0
        \end{cases}.
    \end{align*} Then we can compute $$ \sum_{i=1}^{t}(X^{(i)}-f_X(\bV_{i,X}^\intercal \btheta_X))\bV_{i,X}$$ as $$\nabla_{\btheta}\sum_{i=1}^{t}\left(X^{(i)}\bV_{i,X}^\intercal\btheta_X-m_X(\bV_{i,X}^\intercal\btheta_X)\right).$$ Hence, we only need to prove that $$H_X(\btheta_X)\triangleq\sum_{i=1}^{t}\left(X^{(i)}\bV_{i,X}^\intercal\btheta_X-m_X(\bV_{i,X}^\intercal\btheta_X)\right)$$ is a concave function with respect to $\btheta_X$ and $\lim_{(\btheta_X)_j\rightarrow\infty}H_X(\btheta_X)=-\infty$ or $\frac{\partial H_X(\btheta_X)}{\partial (\btheta_X)_j}\equiv0$ for all $j\in[|\Pa(X)|]$, which implies that $H_X$ has a maximal point. Firstly, we know that \begin{align*}
        \frac{\partial^2 m_X(x)}{\partial x^2}=f'_X(x)>0,
    \end{align*}
    so $m_X$ is a convex function. Therefore, for any vectors $\btheta_1, \btheta_2\in \mathbb{R}^{|\Pa(X)|}$ and $\lambda\in[0,1]$, we have \begin{align*}
        m_X\left( \bV_{i,X}^\intercal(\lambda\btheta_1+(1-\lambda)\btheta_2)\right)&=m_X\left( \lambda\bV_{i,X}^\intercal\btheta_1+(1-\lambda)\bV_{i,X}^\intercal\btheta_2)\right)\\
        &\leq \lambda m_X(\bV_{i,X}^\intercal\btheta_1)+(1-\lambda)m_X(\bV_{i,X}^\intercal\btheta_2),
    \end{align*}
    so $m_X(\bV_{i,X}^\intercal\btheta_X)$ is also a convex function with respect to $\btheta_X$ and the Hessian matrix $\mathbf{H}[m_X(\bV_{i,X}^\intercal\btheta_X)]$ of $m_X(\bV_{i,X}^\intercal\btheta_X)$ with respect to $\btheta_X$ should be positive semidefinite. Now we can compute the Hessian matrix $\mathbf{H}[H_X(\btheta_X)]$ as\begin{align*}
        \mathbf{H}[H_X(\btheta_X)]=\sum_{i=1}^t\left(-\bV_{i,X}^\intercal\bV_{i,X}\cdot\mathbf{H}[m_X(\bV_{i,X}^\intercal\btheta_X)]\right).
    \end{align*}
    Hence, $\mathbf{H}[H_X(\btheta_X)]$ is negative semidefinite because multiplying a positive semidefinite matrix by a negative scalar preserves the semidefiniteness. Thus $H_X$ is a concave function with respect to $\btheta_X$.

    Now for any $j\in[\Pa(X)]$, we prove that $\lim_{(\btheta_{X})_j\rightarrow+\infty}H_X(\btheta_ X)=-\infty$ and $\lim_{(\btheta_{X})_j\rightarrow-\infty}H_X(\btheta_ X)=-\infty$ or $\frac{\partial H_X(\btheta_X)}{\partial (\btheta_X)_j}\equiv 0$. Firstly, we have\begin{align*}
        \frac{\partial H_X(\btheta_X)}{\partial (\btheta_X)_j}&=\sum_{i=1}^t\left(X^{(i)}(\bV_{i,X})_j-(\bV_{i,X})_jm_X'(\bV_{i,X}^\intercal \btheta_X)\right)\\
        &=\sum_{i=1}^t\left(X^{(i)}(\bV_{i,X})_j-(\bV_{i,X})_jf_X(\bV_{i,X}^\intercal \btheta_X)\right).
    \end{align*}
    If $(\bV_{i,X})_j=0$ for all $i\in[t]$, we have $\frac{\partial H_X(\btheta_X)}{\partial (\btheta_X)_j}\equiv 0$. Otherwise, we have\begin{align*}
        \lim_{(\btheta_{X})_j\rightarrow+\infty}\frac{\partial H_X(\btheta_X)}{\partial (\btheta_X)_j}&=\lim_{(\btheta_{X})_j\rightarrow+\infty}\sum_{i=1}^t\left(X^{(i)}(\bV_{i,X})_j-(\bV_{i,X})_jf_X(\bV_{i,X}^\intercal \btheta_X)\right)\\
        &=\lim_{(\btheta_{X})_j\rightarrow+\infty}\sum_{i=1}^t(\bV_{i,X})_j\left(X^{(i)}-f_X(\bV_{i,X}^\intercal \btheta_X)\right)\\
        &=-\infty,\tag{$\lim_{(\btheta_{X})_j\rightarrow+\infty}f_X(\bV_{i,X}^\intercal \btheta_X)=+\infty$}
    \end{align*}
    which indicates that $\lim_{(\btheta_{X})_j\rightarrow+\infty}H_X(\btheta_ X)=-\infty$. Also, we have\begin{align*}
        \lim_{(\btheta_{X})_j\rightarrow-\infty}\frac{\partial H_X(\btheta_X)}{\partial (\btheta_X)_j}&=\lim_{(\btheta_{X})_j\rightarrow-\infty}\sum_{i=1}^t\left(X^{(i)}(\bV_{i,X})_j-(\bV_{i,X})_jf_X(\bV_{i,X}^\intercal \btheta_X)\right)\\
        &=\lim_{(\btheta_{X})_j\rightarrow-\infty}\sum_{i=1}^t(\bV_{i,X})_j\left(X^{(i)}-f_X(\bV_{i,X}^\intercal \btheta_X)\right)\\
        &=+\infty,\tag{$\lim_{(\btheta_{X})_j\rightarrow-\infty}f_X(\bV_{i,X}^\intercal \btheta_X)=-\infty$}
    \end{align*}
    which indicates that $\lim_{(\btheta_{X})_j\rightarrow-\infty}H_X(\btheta_ X)=-\infty$. 

    Until now, we have proved that $H_X(\btheta_X)$ has at least one global maximum, which indicates that the equation has at least one solution.
\end{proof}

\section{Proofs for Propositions in Section~\ref{sec.bglm}}
In this section, we give proofs that are omitted in Section~\ref{sec.bglm} of our main text.
\subsection{Proof of Lemma~\ref{lemma.dodifference}}
\lemmadodiff*

\begin{proof}
At first, we define an equivalent threshold model form of the BGLM as follows.
For each node $X$, we randomly sample a threshold $\gamma_X$ uniformly from $[0,1]$, i.e., $\gamma_X\sim\mathcal{U}[0,1]$. Then if $f_X(\Pa(X)\cdot \btheta^*_X)+\varepsilon_X\geq \gamma_X$, $X$ is activated, i.e., $X$ is set to $1$; otherwise, $X$ is not activated, i.e., $X$ is set to $0$. Therefore, if we ignore $\bbvarepsilon$, the BGLM model belongs to the family of general threshold models \citep{kempe2003maximizing}. For convenience, we denote the vector of all $\gamma_X, X\in\bX\cup\{Y\}\backslash\{X_1\}$ by $\bgamma$. The vector of fixing all entries in $\bgamma$ except $\gamma_X$ is denoted by $\bgamma_{-X}$.

Now we prove the first part of this lemma: $\mathbb{E}[X_j|\doi(X_i=1)]-\mathbb{E}[X_j|\doi(X_i=0)]\geq\kappa\theta^*_{X_i,X_j}\geq \kappa\theta^*_{\min}$ if $X_i\in\Pa(X_j)$. By the definition of our equivalent threshold model, we know that after fixing all the thresholds $\gamma_X$'s and noises $\varepsilon_X$'s, the propagation result is completely determined merely by the intervention. Therefore, we have\begin{align*}
    \mathbb{E}[X_j|\doi(X_i=1)]&=\mathbb{E}_{\bgamma\in(\mathcal{U}[0,1])^n, \bbvarepsilon}[X_j|\doi(X_i=1)]\\
    &=\mathbb{E}_{\bgamma_{-X_j}\in(\mathcal{U}[0,1])^{n-1}, \bbvarepsilon}\left[\Pr_{\gamma_{X_j}\sim \mathcal{U}[0,1]}\left\{X_j=1|\doi(X_i=1),\bgamma_{-X_j},\varepsilon\right\}\right],
\end{align*} 
and\begin{align*}
    \mathbb{E}[X_j|\doi(X_i=0)]=\mathbb{E}_{\bgamma_{-X_j}\in(\mathcal{U}[0,1])^{n-1}, \bbvarepsilon}\left[\Pr_{\gamma_{X_j}\sim \mathcal{U}[0,1]}\left\{X_j=1|\doi(X_i=0),\bgamma_{-X_j},\varepsilon\right\}\right].
\end{align*}
Hence, in order to prove $\mathbb{E}[X_j|\doi(X_i=1)]-\mathbb{E}[X_j|\doi(X_i=0)]\geq\kappa\theta^*_{X_i,X_j}\geq \kappa\theta^*_{\min}$, we only need to prove \begin{align*}
    \Pr_{\gamma_{X_j}\sim \mathcal{U}[0,1]}\left\{X_j=1|\doi(X_i=1),\bgamma_{-X_j},\varepsilon\right\}-\Pr_{\gamma_{X_j}\sim \mathcal{U}[0,1]}\left\{X_j=0|\doi(X_i=0),\bgamma_{-X_j},\varepsilon\right\}\geq \kappa\theta^*_{\min}.
\end{align*}
When $\bgamma_{-X_j}$ and $\bbvarepsilon$ are fixed, all the nodes in $\bX\cup\{Y\}\backslash\left(\{X_j\}\cup\{\Des(X_j)\}\right)$ are already fixed given an arbitrarily fixed intervention. Here, $\Des(X_j)$ is used to represent the descendants of $X_j$. Suppose under $\doi(X_i=1), \bgamma_{-X_j}$ and $\bbvarepsilon$, the value vector of parents of $X_j$ is $\pa_1(X_j)$; under $\doi(X_i=0), \bgamma_{-X_j}$ and $\bbvarepsilon$, the value vector of parents of $X_j$ is $\pa_0(X_j)$. By induction along the topological order, nodes in $\bX\cup\{Y\}\backslash\left(\{X_j\}\cup\{\Des(X_j)\}\right)$ that is activated under $\doi(X_i=0), \bgamma_{-X_j}$ and $\bbvarepsilon$ must be also activated under $\doi(X_i=1), \bgamma_{-X_j}$ and $\bbvarepsilon$. Therefore, entries in $\pa_1(X_j)-\pa_0(X_j)$ are all non-negative and the entry in $\pa_1(X_j)-\pa_0(X_j)$ for the value of $X_j$ is $1$. From this observation, we can deduce that\begin{align*}
    f_{X_j}(\pa_1(X_j)\cdot\btheta^*_{X_j})-f_{X_j}(\pa_0(X_j)\cdot\btheta^*_{X_j})&\geq \kappa\left(\pa_1(X_j)\cdot\btheta^*_{X_j}-\pa_0(X_j)\cdot\btheta^*_{X_j}\right)\\
    &\geq \kappa\theta^*_{X_i,X_j}.
\end{align*} 
Hence, we have\begin{align*}
    &\Pr_{\gamma_{X_j}\sim \mathcal{U}[0,1]}\left\{X_j=1|\doi(X_i=1),\bgamma_{-X_j},\varepsilon\right\}-\Pr_{\gamma_{X_j}\sim \mathcal{U}[0,1]}\left\{X_j=0|\doi(X_i=0),\bgamma_{-X_j},\varepsilon\right\}\\
    &=\Pr_{\gamma_{X_j}\sim \mathcal{U}[0,1]}\left\{f_{X_j}(\pa_1(X_j)\cdot\btheta^*_{X_j})\geq \gamma_{X_j}+\varepsilon_{X_j}|\varepsilon_{X_j}\right\}
    \\&\quad\quad-\Pr_{\gamma_{X_j}\sim \mathcal{U}[0,1]}\left\{f_{X_j}(\pa_0(X_j)\cdot\btheta^*_{X_j})\geq \gamma_{X_j}+\varepsilon_{X_j}|\varepsilon_{X_j}\right\}\\
    &=\left(f_{X_j}(\pa_1(X_j)\cdot\btheta^*_{X_j})-\varepsilon_{X_j}\right)-\left(f_{X_j}(\pa_0(X_j)\cdot\btheta^*_{X_j})-\varepsilon_{X_j}\right)\\
    &\geq \kappa\theta^*_{X_i,X_j}\geq\kappa\theta^*_{\min},
\end{align*}
which is what we want. Until now, the first part of Lemma~\ref{lemma.dodifference} has been proved.

Then we prove the second part of this lemma: $\mathbb{E}[X_j|\doi(X_i=1)]=\mathbb{E}[X_j|\doi(X_i=0)]$ if $X_j$ is not a descendant of $X_i$. In this situation, we know from the graph structure that $(X_j\perp\!\!\!\!\perp X_i)_{G_{\overline{\{X_i\}}}}$, where $G_{\overline{\{X_i\}}}$ is the graph obtained by deleting from $G$ all arrows pointing to $X_i$. According to the third law of $\doi$-calculus \citep{pearl2012do}, we deduce that \begin{align*}
    \mathbb{E}[X_j|\doi(X_i=1)]&=\Pr\{X_j=1|\doi(X_i=1)\}=\Pr\{X_j=1|\}\\&=\Pr\{X_j=1|\doi(X_i=0)\}=\mathbb{E}[X_j|\doi(X_i=0)].
\end{align*}
Now Lemma~\ref{lemma.dodifference} is completely proved.
\end{proof}

\begin{corollary}[An Extension of Lemma~\ref{lemma.dodifference}]
Suppose $G$ is a BGLM with parameter $\btheta^*$ that satisfying Assumption \ref{asm:glm_2} and $\doi(\bS=\bs)$ is an intervention such that $X_i,X_j\notin\bS$. If $X_i\in\Pa(X_j)$, we have $\mathbb{E}[X_j|\doi(X_i=1),\doi(\bS=\bs)]-\mathbb{E}[X_j|\doi(X_i=0),\doi(\bS=\bs)]\geq \kappa\theta^*_{X_i,X_j}\geq\kappa\theta^*_{\min}$; if $X_i$ is not an ancestor of $X_j$, we have $\mathbb{E}[X_j|\doi(X_i=1),\doi(\bS=\bs)]=\mathbb{E}[X_j|\doi(X_i=0),\doi(\bS=\bs)]$.
\end{corollary}

\begin{proof}
    According to \cite{pearl2012do}, $\Pr\{X_j|\doi(X_i),\doi(\bS)\}$ is equivalent to $\Pr\{X_j|\doi(X_i)\}$ in a new model $G'$ such that all in-edges of $\bS$ are deleted and all nodes in $\bS$ are fixed by $\bs$. We know that Lemma~\ref{lemma.dodifference} holds in $G'$, so this corollary holds in $G$.
\end{proof}

\subsection{Proof of Lemma~\ref{lemma.correctorder}}
\label{app.correctorder}
\lemmaorder*

\begin{proof}
    We first assume that for every pair of nodes if $X_i\in\Pa(X_j)$, Algorithm~\ref{alg:bglm-order} puts $X_j$ as a descendant of $X_i$ in the ancestor-descendant relationship; if $X_j$ is not a descendant of $X_i$, Algorithm~\ref{alg:bglm-order} do not put $X_j$ as an descendant of $X_i$ in the ancestor-descendant relationship. This event is denoted by $\mathcal{E}$ for simplicity. We prove that when event $\mathcal{E}$ does occur, the ancestor-descendant relationship we find is absolutely consistent with the true graph structure of $G$. Otherwise, suppose there is a mistake in the ancestor-descendant relationship such that $X_i$ is an ancestor of $X_j$ but not put in $\widehat{\Anc}(X_j)$. We denote a directed path from $X_i$ to $X_j$ by $X_i\rightarrow X_{k_1}\rightarrow X_{k_2}\rightarrow\cdots \rightarrow X_{k_p}\rightarrow X_{j}$. Therefore, $X_{k_1}$ must be put in $\widehat{\Anc}(X_i)$, $X_{k_2}$ must be put in $\widehat{\Anc}(X_{k_1})$, \ldots, $X_{j}$ must be put in $\widehat{\Anc}(X_{k_p})$. In conclusion, $X_{j}$ should be put in $\widehat{\Anc}(X_i)$, which is a contradiction. Hence, there is no mistake in the ancestor-descendant relationship given event $\mathcal{E}$. 

    Now we only prove that using Algorithm~\ref{alg:bglm-order}, with probability no less than $$1-2\binom{n-1}{2}\exp\left(-\frac{c_0c_1^2T^{1/10}}{2}\right),$$ event $\mathcal{E}$ defined in the paragraph above occurs. For a pair of nodes $X_i,X_j\in\bX\backslash\{X_1\}$, if $X_i\in\Pa(X_j)$, we know from Lemma~\ref{lemma.dodifference} that $\mathbb{E}[X_j|\doi(X_i=1)]-\mathbb{E}[X_j|\doi(X_i=0)]\geq\kappa\theta^*_{\min}$. We denote the difference between random variable $X_j$ given $\doi(X_i=1)$ and random variable $X_j$ given $\doi(X_i=0)$ by $Z$. In $\sum_{k=1}^{c_0T^{1/2}}\left(X_j^{\left(2ic_0T^{1/2}+k\right)}-X_j^{\left((2i+1)c_0T^{1/2}+k\right)}\right)$, each term $X_j^{\left(2ic_0T^{1/2}+k\right)}-X_j^{\left((2i+1)c_0T^{1/2}+k\right)}$ is an i.i.d. sample of $Z$. We denote $X_j^{\left(2ic_0T^{1/2}+k\right)}-X_j^{\left((2i+1)c_0T^{1/2}+k\right)}$ by $Z_k$. We know that $Z_k\in[-1,1]$ and $\mathbb{E}[Z_k]\geq \kappa\theta^*_{\min}$, so according to Hoeffding's inequality \citep{hoeffding1994probability}, we have\begin{align*}
        &\Pr\left\{\sum_{k=1}^{c_0T^{1/2}}\left(X_j^{\left(2ic_0T^{1/2}+k\right)}-X_j^{\left((2i+1)c_0T^{1/2}+k\right)}\right)>c_0c_1T^{3/10}\right\}\\
        &=\Pr\left\{\sum_{k=1}^{c_0T^{1/2}}Z_k>c_0c_1T^{3/10}\right\}\\
        &=1-\Pr\left\{\sum_{k=1}^{c_0T^{1/2}}Z_k\leq c_0c_1T^{3/10}\right\}\\
        &\geq 1-\exp\left(-\frac{2\left(c_0T^{1/2}\kappa\theta^*_{\min}-c_0c_1T^{3/10}\right)^2}{4c_0T^{1/2}}\right)=1-\exp\left(-\frac{c_0\left(T^{1/4}\kappa\theta^*_{\min}-c_1T^{1/20}\right)^2}{2}\right)\\
        &\geq 1-\exp\left(-\frac{c_1^2c_0T^{1/10}}{2}\right).\tag{because $T\geq32\left(\frac{c_1}{\kappa\theta^*_{\min}}\right)^5$}
    \end{align*}

    Similarly, if $X_j$ is not a descendant of $X_i$, we do not put $X_i$ in $\widehat{\Anc}(X_j)$ in the ancestor-descendant relationship if and only if $\sum_{k=1}^{c_0T^{1/2}}\left(X_j^{\left(2ic_0T^{1/2}+k\right)}-X_j^{\left((2i+1)c_0T^{1/2}+k\right)}\right)\leq c_0c_1T^{3/10}$. Now we still have $Z_k\in[-1,1]$ but $\mathbb{E}[Z_k]=0$. Therefore, according to Hoeffding's inequality \citep{hoeffding1994probability}, we have\begin{align*}
        &\Pr\left\{\sum_{k=1}^{c_0T^{1/2}}\left(X_j^{\left(2ic_0T^{1/2}+k\right)}-X_j^{\left((2i+1)c_0T^{1/2}+k\right)}\right)\leq c_0c_1T^{3/10}\right\}\\
        &=1-\Pr\left\{\sum_{k=1}^{c_0T^{1/2}}Z_k> c_0c_1T^{3/10}\right\}\\
        &>1-\exp\left(-\frac{2\left(c_0c_1T^{3/10}\right)^2}{4c_0T^{1/2}}\right)=1-\exp\left(-\frac{c_1^2c_0T^{1/10}}{2}\right).
    \end{align*}

    Hence, by union bound (Boole's inequality \citep{bonferroni1936teoria}), the probability of $\mathcal{E}$ is no less than $1-2\binom{n-1}{2}\exp\left(-\frac{c_1^2c_0T^{1/10}}{2}\right)$. This is because when $X_i,X_j\in\bX\backslash\{X_1\}$, there are $2\binom{n-1}{2}$ possible choices of them that are tested by Algorithm~\ref{alg:bglm-order}. When $\mathcal{E}$ happens, Algorithm~\ref{alg:bglm-order} gets the ancestor-descendant relationship correct, so Lemma~\ref{lemma.correctorder} is proved.
\end{proof}

\subsection{Proof of Theorem~\ref{thm.regretbglm}}

In the following proofs on a BGLM $G$, when $X'\in\Anc(X)$ but $X'\notin\Pa(X)$, we add an edge $X'\rightarrow X$ with weight $\theta_{X',X}=0$ into $G$ and this does not impact the propagation results of $G$. Let $D = \max_{X \in \bX\cup{Y}}|\Pa(X)|$ represent the maximum in-degree. After applying this transformation, $D=n$ and $\Anc(X)=\Pa(X)$ for all $X\in\bX\cup{Y}$ in this subsection. This transformation effectively converts the ancestor-descendant relationship into an ancestor-descendant graph.

Before the proof of this theorem, we introduce several lemmas at first. The first component is based on the result of maximum-likelihood estimation (MLE). It gives a theoretical measurement for the accuracy of estimated $\hat{\btheta}$ computed by MLE. One who is interested could find the proof of this lemma in Appendix~C.2 of \cite{feng2022combinatorialarxiv}. 
\begin{restatable}[Lemma 1 in \cite{feng2022combinatorial}]{lemma}{thmlearningglm}
\label{thm.learning_glm}
Suppose that Assumptions~\ref{asm:glm_3} and~\ref{asm:glm_2} hold.
Moreover, given $\delta\in(0,1)$, assume that
\begin{small}
\begin{align}
    \lambda_{\min}(M_{t,X})\geq \frac{512|\Pa(X)|\left(L_{f_X}^{(2)}\right)^2}{\kappa^4}\left(|\Pa(X)|^2+\ln\frac{1}{\delta}\right). \label{eq:lambdamin}
\end{align}
\end{small}
Then with probability at least $1-3\delta$, the maximum-likelihood estimator satisfies , for any $\bv\in\mathbb{R}^{|\Pa(X)|}$, 
\begin{align*}
    \left|\bv^\intercal(\hat{\btheta}_{t,X}-\btheta^*_X)\right|\leq\frac{3}{\kappa}\sqrt{\log(1/\delta)}\left\|\bv\right\|_{M_{t,X}^{-1}},
\end{align*}
where the probability is taken from the randomness of all data collected from round $1$ to round $t$.
\end{restatable}

The second component is called the group observation modulated (GOM) bounded smoothness property \citep{li2020onlineinference}. It shows that a small change in parameters $\btheta$ leads to a small change in the reward. Under our BGLM setting, this lemma is proved in Appendix~C.3 of \cite{feng2022combinatorialarxiv}.
\begin{restatable}[Lemma~2 in \cite{feng2022combinatorial}]{lemma}{thmgomglm}
\label{thm.gom_glm}
For any two weight vectors $\btheta^1,\btheta^2\in \Theta$ for a BGLM $G$, the difference of their expected reward for any intervened set $\bS$ can be bounded as
\begin{small}\begin{align}
    &\left|\sigma({\bS},\btheta^1)-\sigma({\bS},\btheta^2)\right| \leq\mathbb{E}_{\bbvarepsilon, \bgamma}\left[ \sum_{X\in {\bX}_{{\bS},Y}}
    	\!\!\! \left|\bV_{X}^\intercal (\btheta^1_X-\btheta^2_X)\right|L_{f_X}^{(1)}\right], \label{eq:GOMcond}
\end{align}\end{small}
where ${\bX}_{{\bS},Y}$ is the set of nodes in paths from ${\bS}$ to $Y$ excluding $\bS$, and $\bV_X$ is the propagation result of the parents of $X$ under parameter $\btheta^2$. The expectation is taken over the randomness of the thresholds $\bgamma$ and the noises $\bbvarepsilon$.
\end{restatable}

Thirdly, we propose a lemma in order to bound the sum of $\left\|\bV_{t,X}\right\|_{M_{t-1,X}^{-1}}$ at first. This lemma is proved in Appendix~C.4 of \cite{feng2022combinatorialarxiv}.  
\begin{lemma}[Lemma 9 in \cite{feng2022combinatorialarxiv}]
\label{lemma.bound_V}
Let $\{\bW_t\}_{t=1}^\infty$ be a sequence in $\mathbb{R}^d$ satisfying $\left\|\bW_t\right\|\leq\sqrt{d}$. Define $\bW_0={\mathbf 0}$ and $M_t=\sum_{i=0}^{t}\bW_i\bW_i^\intercal$. Suppose there is an integer $t_1$ such that $\lambda_{\min}(M_{t_1+1})\geq 1$, then for all $t_2>0$,\begin{align*}
    \sum_{t=t_1}^{t_1+t_2}\left\|\bW_t\right\|_{M_{t-1}^{-1}}\leq\sqrt{2t_2d\log(t_2d+t_1)}.
\end{align*}
\end{lemma}

At last, in order to show that $\lambda_{\min}(M_{T_1,X})\geq R$ after the initialization phase of Algorithm~\ref{alg:glm-ofu} and thus satisfy the condition of Lemma~\ref{thm.learning_glm}, we introduce Lemma~\ref{lemma:init_condition}. This lemma is improved upon Lemma~7 in \cite{feng2022combinatorialarxiv} and enables us to use Lecu\'{e} and Mendelson's inequality \citep{nie2022matrix} in our later theoretical regret analysis.

Let $\Sphere(d)$ denote the sphere of the $d$-dimensional unit ball.
\begin{restatable}{lemma}{lemmainit}
	\label{lemma:init_condition}
	For any $\bv=(v_1,v_2,\ldots,v_{|\Pa(X)|})\in \Sphere(|\Pa(X)|)$ and any $X\in{\bX}\cup\{Y\}$ in a BGLM that 
	satisfies Assumption \ref{asm:glm_4}, we have
	\begin{align*}
	\Pr_{\bbvarepsilon, \bX, Y}\left\{\left|\Pa(X)\cdot \bv\right|\geq\frac{1}{\sqrt{4D^2-3}}\right\}&\geq\zeta,
	\end{align*}
	where $\Pa(X)$ is the random vector generated by the natural Bayesian propagation in BGLM $G$ with no interventions
	(except for setting $X_1$ to 1).
\end{restatable}

\begin{proof}
	The lemma is similarly proved as Lemma 7 in \cite{feng2022combinatorialarxiv} using the idea of Pigeonhole principle.
	Let $\Pa(X)=(X_{i_1}=X_1,X_{i_2},X_{i_3},\ldots, X_{i_{|\Pa(X)|}})$ as the random vector and  
	$\pa(X)=(x_1=1,x_{i_1},x_{i_2},x_{i_3},\ldots, x_{i_{|\Pa(X)|}})$ as a possible valuation of $\Pa(X)$. 
	Without loss of generality, we suppose that $\left|v_2\right|\geq \left|v_3\right|\geq\ldots\ge\left|v_{|\Pa(X)|}\right| $. 
	For simplicity, we denote $D_0=\sqrt{D-1}+\frac{1}{2\sqrt{D-1}}$. If $\left|v_1\right|\geq\frac{D_0}{\sqrt{D_0^2+1}}$, 
	we can deduce that 
	\begin{align}
	|\pa(X)\cdot \bv|&\geq\left|v_1\right|-\left|v_2\right|-\left|v_3\right|-\cdots - \left|v_{|\Pa(X)|}\right| \nonumber\\
	&\geq\frac{D_0}{\sqrt{D_0^2+1}}-\sqrt{(D-1)\left(\left|v_2\right|^2+\left|v_3\right|^2+\cdots +\left|v_{|\Pa(X)|}\right|^2\right)} \label{eq:cauchy}\\
	&\ge \frac{D_0}{\sqrt{D_0^2+1}}-\sqrt{(D-1)\left(1-\frac{D_0^2}{D_0^2+1}\right)} \label{eq:sphere}\\
	&=\frac{1}{2\sqrt{(D_0^2+1)(D-1)}} = \frac{1}{\sqrt{4D^2-3}}, \nonumber 
	\end{align}
	where Inequality~\eqref{eq:cauchy} is by the Cauchy-Schwarz inequality and the fact that $|\Pa(X)|\leq D$, and 
	Inequality~\eqref{eq:sphere} uses the fact that $\bv \in \Sphere(|\Pa(X)|)$.
	Thus, when $\left|v_1\right|\geq\frac{D_0}{\sqrt{D_0^2+1}}$, the event $\left|\Pa(X)\cdot \bv\right|\geq\frac{1}{\sqrt{4D^2-3}}$
	holds deterministically.
	Otherwise, when $\left|v_1\right| < \frac{D_0}{\sqrt{D_0^2+1}}$, 
	we use the fact that $|v_2|$ is the largest among $|v_2|, |v_3|, \ldots$ and deduce that
	\begin{align} \label{eq:inequalityv2}
	\left|v_2\right|\geq\frac{1}{\sqrt{n-1}}\sqrt{\left|v_2\right|^2+\left|v_3\right|^2+\cdots}
	\geq\frac{\sqrt{1-\left(\frac{D_0}{\sqrt{D_0^2+1}}\right)^2}}{\sqrt{n-1}}
	=\frac{2}{\sqrt{4D^2-3}}.
	\end{align}
	
	Therefore, using the fact that 
	\begin{align*}
	&\Pr_{\bbvarepsilon, \bX,Y}\left\{X_{i_1}=1,X_{i_2}=x_{i_2},X_{i_3}=x_{i_3},\ldots\right\}\\&=\Pr_{\bbvarepsilon, \bX,Y}\left\{X_{i_2}=x_{i_2}|X_{i_1}=1,X_{i_3}=x_{i_3},\ldots\right\}\cdot\Pr_{\bbvarepsilon, \bX,Y}\left\{(X_{i_1}=1,X_{i_3}=x_{i_3},\ldots\right\}\\&\geq \zeta \Pr_{\bbvarepsilon, \bX,Y}\left\{X_{i_1}=1,X_{i_3}=x_{i_3},\ldots\right\}
	\end{align*}
	
	and $\sum_{x_{i_3},x_{i_4},\ldots} \Pr_{\bbvarepsilon, \bX,Y}\left\{X_{i_1}=1,X_{i_3}=x_{i_3},\ldots\right\}=1$, we have
	\begin{small}
	\begin{align}
		&\Pr_{\bbvarepsilon, \bX,Y}\left\{\left|\Pa(X)\cdot \bv\right|\geq\frac{1}{\sqrt{4D^2-3}}\right\} \nonumber\\
		&=\sum_{x_{i_3},x_{i_4},\ldots}\Pr\{X_{i_1}=1,X_{i_2}=1,X_{i_3}=x_{i_3},\ldots\}\cdot  \I\left\{ |(1,1,x_{i_3},x_{i_4},\ldots)\cdot (v_1,v_2,v_3,\ldots)| \geq\frac{1}{\sqrt{4D^2-3}}\right\} \nonumber\\
		&+\sum_{x_{i_3},x_{i_4},\ldots}\Pr\{X_{i_1}=1,X_{i_2}=0,X_{i_3}=x_{i_3},\ldots\}\cdot \I\left\{|(1,0,x_{i_3},x_{i_4},\ldots)\cdot (v_1,v_2,v_3,\ldots)| \geq\frac{1}{\sqrt{4D^2-3}}\right\} \nonumber\\
		&\geq \sum_{x_{i_3},x_{i_4},\ldots}\zeta \Pr\{X_{i_1}=1,X_{i_3}=x_{i_3},X_{i_4}=x_{i_4}\ldots\}
		\cdot \I\left\{|(1,1,x_{i_3},x_{i_4},\ldots)\cdot (v_1,v_2,v_3,\ldots)| \geq\frac{1}{\sqrt{4D^2-3}}\right\} \nonumber\\
		&+\sum_{x_{i_3},x_{i_4},\ldots}\zeta \Pr\{X_{i_1}=1,X_{i_3}=x_{i_3},X_{i_4}=x_{i_4},\ldots\}
		\cdot \I\left\{|(1,0,x_{i_3},x_{i_4},\ldots)\cdot (v_1,v_2,v_3,\ldots)| \geq\frac{1}{\sqrt{4D^2-3}}\right\} \nonumber\\
		&=\zeta  \sum_{x_{i_3},x_{i_4},\ldots}\Pr\{X_{i_1}=1,X_{i_3}=x_{i_3},X_{i_4}=x_{i_4},\ldots\} \left(\I\left\{|(1,1,x_{i_3},x_{i_4},\ldots)\cdot (v_1,v_2,v_3,\ldots)| \geq\frac{1}{\sqrt{4D^2-3}}\right\} \right. \nonumber \\
		& \left. +\I\left\{|(1,0,x_{i_3},x_{i_4},\ldots)\cdot (v_1,v_2,v_3,\ldots)| \geq\frac{1}{\sqrt{4D^2-3}}\right\} \right) \nonumber\\
		&\geq \zeta\sum_{x_{i_3},x_{i_4},\ldots}\Pr\{X_{i_1}=1,X_{i_3}=x_{i_3},X_{i_4}=x_{i_4},\ldots\} \label{eq:identityrelax} \\
		&=\zeta, \nonumber
		\end{align}\end{small}
	which is exactly what we want to prove. 
	Inequality~\eqref{eq:identityrelax} holds because 
	otherwise, at least for some $x_{i_3}, x_{i_4}, \ldots$, both indicators on the left-hand side of the inequality have to be $0$, which implies that
    \begin{small}
	\begin{align}
	\left|(1,1,x_{i_3},x_{i_4},\ldots)\cdot (v_1,v_2,v_3,\ldots)-(1,0,x_{i_3},x_{i_4},\ldots)\cdot (v_1,v_2,v_3,\ldots)\right|=\left|v_2\right|< \frac{2}{\sqrt{4D^2-3}},
	\end{align}\end{small}
	but this contradicts to Inequality~\eqref{eq:inequalityv2}.
\end{proof}

Having these four lemmas above together with Lemma~\ref{lemma.correctorder} proved in Appendix~\ref{app.correctorder}, we are finally able to prove the regret bound of BGLM-OFU-Unknown algorithm (Theorem~\ref{thm.regretbglm}) as below.

\bglmregret*

\begin{proof}
We only consider the case of $T\geq32\left(\frac{c_1}{\kappa\theta^*_{\min}}\right)^5$ in this proof because the big O notation is asymptotic.

Let $H_t$ be the history of the first $t$ rounds and $R_t$ be the regret in the $t^{th}$ round. Because the reward node $Y$ is in interval $[0,1]$, we can deduce that for any $t\leq T_1$, $R_t\leq 1$. Now we consider the case of $t>T_1$. According to Lemma~\ref{lemma.correctorder}, with probability at least $1-2\binom{n-1}{2}\exp\left(-\frac{c_0c_1^2T^{1/10}}{2}\right)$, Algorithm~\ref{alg:bglm-order} returns a correct ancestor-descendant relationship, i.e., $\widehat{\Anc}(X)=\Anc(X)$ for $X\in\bX\cup\{Y\}$. Next we bound the regret conditioned on the correct ancestor-descendant relationship. When $t>T_1$, we have\begin{align}
    \mathbb{E}[R_t|H_{t-1}]=\mathbb{E}[\sigma({\mathbf S}^{\text{opt}},\btheta^*)-\sigma({\mathbf S}_t,\btheta^*)|H_{t-1}],
\end{align}
where the expectation is taken over the randomness of ${\mathbf S}_t$. Then for $T_1<t\leq T$, we define $\xi_{t-1,X}$ for $X\in {\mathbf X}\cup \{Y\}$ as $\xi_{t-1,X}=\left\{\left|\bv^T(\hat{\btheta}_{t-1,X}-\btheta_X^*)\right|\leq\rho\cdot\left\|\bv\right\|_{M_{t-1,X}^{-1}},\forall \bv\in \mathbb{R}^{|\Pa(X)|}\right\}$. According to the definition of Algorithm \ref{alg:glm-ofu}, we can deduce that $\lambda_{\min}(M_{t-1,X})\geq\lambda_{\min}(M_{T_1,X})$. By Lecu\'{e} and Mendelson's inequality \citep{nie2022matrix, feng2022combinatorialarxiv} (conditions of this inequality satisfied according to Lemma \ref{lemma:init_condition}), we have 
\begin{align*}
\Pr\left\{\lambda_{\min}(M_{T_1,X})<R\right\}\leq\Pr\left\{\lambda_{\min}(M_{T_1,X}-M_{T_0,X})<R\right\}\leq\exp\left(-\frac{(T_1-T_0)\zeta^2}{c}\right)
\end{align*}
where $c,\zeta$ are constants. Then we can define $\xi_{t-1}=\wedge_{X\in{\mathbf X}\cup\{Y\}}\xi_{t-1,X}$ and let $\overline{\xi_{t-1}}$ be its complement. By Lemma \ref{thm.learning_glm}, we have $$\Pr\left\{\overline{\xi_{t-1}}\right\}\leq \left(3\delta+\exp\left(-\frac{(T_1-T_0)\zeta^2}{c}\right)+3\delta \exp\left(-\frac{(T_1-T_0)\zeta^2}{c}\right)\right)n\triangleq p_{\text{error}}.$$

Because under $\xi_{t-1}$, for any $X\in{\mathbf X}\cup\{Y\}$ and $\bv\in \mathbb{R}^{|\Pa(X)|}$, we have $\left|\bv^T(\hat{\btheta}_{t-1,X}-\btheta_X^*)\right|\leq\rho\cdot\left\|\bv\right\|_{M_{t-1,X}^{-1}}$. Therefore, by the definition of $\tilde{\btheta}_t$, we have $\sigma({\mathbf S}_t,\tilde{\btheta}_t)\geq \sigma({\mathbf S}^{\text{opt}},\btheta^*)$ because $\btheta^*$ is in our confidence ellipsoid. Hence,\begin{align*}
    \mathbb{E}[R_t]&\leq \Pr\left\{\xi_{t-1}\right\}\cdot\mathbb{E}[\sigma({\mathbf S}^{\text{opt}},\btheta^*)-\sigma({\mathbf S}_t,\btheta^*)] + \Pr(\overline{\xi_{t-1}})\\
    &\leq \mathbb{E}[\sigma({\mathbf S}^{\text{opt}},\btheta^*)-\sigma({\mathbf S}_t,\btheta^*)]+p_{\text{error}}\\
    &\leq \mathbb{E}[\sigma({\mathbf S}_t,\tilde{\btheta}_t)-\sigma({\mathbf S}_t,\btheta^*)]+p_{\text{error}}.
\end{align*}
Then we need to bound $\sigma({\mathbf S}_t,\tilde{\btheta}_t)-\sigma({\mathbf S}_t,\btheta^*)$ carefully. 

Therefore, according to Lemma~\ref{thm.learning_glm} and Lemma~\ref{thm.gom_glm}, we can deduce that\begin{align*}
    \mathbb{E}[R_t]
    &\leq\mathbb{E}\left[\sum_{X\in {\mathbf X}_{{\mathbf S}_t,Y}}\left|\bV_{t,X}(\tilde{\btheta}_{t,X}-\btheta^*_X)\right|L_{f_X}^{(1)}\right]
    +p_{\text{error}}\\
    &\leq\mathbb{E}\left[\sum_{X\in {\mathbf X}_{{\mathbf S}_t,Y}}\left\|\bV_{t,X}\right\|_{M_{t-1,X}^{-1}}
    \left\|\tilde{\btheta}_{t,X}-\btheta^*_X\right\|_{M_{t-1,X}}L_{f_X}^{(1)}\right]
    +p_{\text{error}}\\
    &\leq 2\rho\cdot\mathbb{E}\left[\sum_{X\in {\mathbf X}_{{\mathbf S}_t,Y}}\left\|\bV_{t,X}\right\|_{M_{t-1,X}^{-1}}L_{f_X}^{(1)}\right]
    +p_{\text{error}}.
\end{align*}
The last inequality holds because \begin{align*}\left\|\tilde{\btheta}_{t,X}-\btheta^*_X\right\|_{M_{t-1,X}}\leq \left\|\tilde{\btheta}_{t,X}-\hat{\btheta}_{t-1,X}\right\|_{M_{t-1,X}}+\left\|\hat{\btheta}_{t-1,X}-\btheta^*_X\right\|_{M_{t-1,X}}\leq 2\rho.\end{align*}

Therefore, conditioned on the correct ancestor-descendant relationship, the total regret can be bounded as\begin{align*}
    R(T)&\leq 2\rho\cdot\mathbb{E}\left[\sum_{t=T_0+1}^T\sum_{X\in {\mathbf X}_{{\mathbf S}_t,Y}}\left\|\bV_{t,X}\right\|_{M_{t-1,X}^{-1}}L_{f_X}^{(1)}\right]
    +p_{\text{error}}(T-T_1)+T_1.
\end{align*}

For convenience, we define $\bW_{t,X}$ as a vector such that if $X\in S_t$, $\bW_{t,X}={\mathbf 0}^{|\Pa(X)|}$; if $X\not\in S_t$, $\bW_{t,X}=\bV_{t,X}$.
Using Lemma \ref{lemma.bound_V}, we can get the result:\begin{align*}
    R(T)&\leq \left(2\rho\mathbb{E}\left[\sum_{t=T_0+1}^T\sum_{X\in {\mathbf X}_{{\mathbf S}_t,Y}}\left\|\bV_{t,X}\right\|_{M_{t-1,X}^{-1}}L_{f_X}^{(1)}\right]
    +p_{\text{error}}(T-T_1)+T_1\right)\\
    &\quad\quad\cdot\left(1-2\binom{n-1}{2}\exp\left(-\frac{c_0c_1^2T^{1/10}}{2}\right)\right)+2\binom{n-1}{2}\exp\left(-\frac{c_0c_1^2T^{1/10}}{2}\right)T\\
    &\leq 2\rho\mathbb{E}\left[\sum_{t=T_0+1}^T\sum_{X\in\bX\cup\{Y\}}\left\|\bW_{t,X}\right\|_{M_{t-1,X}^{-1}}L_{f_X}^{(1)}\right]
    +p_{\text{error}}(T-T_1)+T_1\\&\quad\quad+2\binom{n-1}{2}\exp\left(-\frac{c_0c_1^2T^{1/10}}{2}\right)T\\
    &\leq 2\rho\cdot \max_{X\in{\mathbf X}\cup\{Y\}}\left(L_{f_X}^{(1)}\right)\mathbb{E}\left[\sum_{X\in \bX\cup\{Y\}}
    \sqrt{2(T-T_0)|\Pa(X)|\log\left((T-T_0)|\Pa(X)|+T_0\right)}\right]\\
    &\quad+p_{\text{error}}(T-T_1)+T_1+2\binom{n-1}{2}\exp\left(-\frac{c_0c_1^2T^{1/10}}{2}\right)T\\
    &=O\left(\frac{1}{\kappa}n^{\frac{3}{2}}\sqrt{T}L^{(1)}_{\max}\ln T\right)=\tilde{O}\left(\frac{1}{\kappa}n^{\frac{3}{2}}\sqrt{T}L^{(1)}_{\max}\right)
    \end{align*}
because $\rho=\frac{3}{\kappa}\sqrt{\log(1/\delta)}$, $\exp\left(-\frac{c_0c_1^2T^{1/10}}{2}\right)T=o(\sqrt{T})$ and $p_{\text{error}}T=o(\sqrt{T})$.
\end{proof}

\section{A BLM CCB Algorithm with Minimum Weight Gap Based on Linear Regression}
\label{app.blmlr} 

As BLM is a special case of BGLM, the initialization phase in BGLM-OFU-Unknown to determine the ancestor-descendant relationship can also be used on BLMs. \cite{feng2022combinatorial} propose a CCB algorithm for BLMs using linear regression instead of MLE to remove the requirement of Assumption~\ref{asm:glm_4}. Furthermore, BLM takes the identity function as $f_X$'s, so Assumptions~\ref{asm:glm_3} and~\ref{asm:glm_2} is neither required. The specific algorithm BLM-LR-Unknown-SG (BLM-LR-Unknown Algorithm with Safety Gap (Minimum Weight Gap)) is demonstrated in Algorithm~\ref{alg:blmlr}.

\begin{algorithm}[t]
\caption{BLM-LR-Unknown-SG for BLM and Linear Model CCB Problem}
\label{alg:blmlr}
\begin{algorithmic}[1]
\STATE {\bfseries Input:}
Graph $G=({\bX}\cup\{Y\},E)$, action set $\mathcal{A}$, positive constants $c_0$ and $c_1$ for initialization phase such that $c_0\sqrt{T}\in\mathbb{N}^+$.
\STATE /* Initialization Phase: */
\STATE Initialize $T_0\leftarrow 2(n-1)c_0T^{1/2}$.
\STATE Do each intervention among $\doi(X_2=1), \doi(X_2=0), \ldots, \doi(X_n=1), \doi(X_n=0)$ for $c_0T^{1/2}$ times in order and observe the feedback $({\mathbf X}_t,Y_t)$ for $1\leq t\leq T_0$.
\STATE Determine a feasible ancestor-descendant relationship $\widehat{\Anc}(X)$'s for $X\in\bX\cup\{Y\}$ by $\text{BGLM-Ancestors}((\bX_1,Y_1),\ldots,(\bX_{T_0},Y_{T_0}),c_1)$ (see Algorithm~\ref{alg:bglm-order}).

\STATE /* Parameters Initialization: */
\STATE Initialize $M_{T_0,X}\leftarrow \mathbf{I}\in\mathbb{R}^{|\widehat{\Anc}(X)|\times |\widehat{\Anc}(X)|}$, $\bb_{T_0,X}\leftarrow {\mathbf 0}^{|\widehat{\Anc}(X)|}$ for all $X\in {\bX}\cup\{Y\}$, $\hat{\btheta}_{T_0,X}\leftarrow 0\in\mathbb{R}^{|\widehat{\Anc}(X)|}$ for all $X\in {\bX}\cup\{Y\}$, $\delta\leftarrow\frac{1}{n\sqrt{T}}$ and $\rho_t\leftarrow\sqrt{n\log(1+tn)+2\log\frac{1}{\delta}}+\sqrt{n}$ for $t=0,1,2,\ldots,T$.

\STATE /* Iterative Phase: */
\FOR{$t=T_0+1,T_0+2,\ldots,T$}
\STATE Compute the confidence ellipsoid $\mathcal{C}_{t,X}=\{\btheta_X'\in[0,1]^{|\widehat{\Anc}(X)|}:\left\|\btheta_X'-\hat{\btheta}_{t-1,X}\right\|_{M_{t-1,X}}\leq\rho_{t-1}\}$ for any node $X\in{\bX}\cup\{Y\}$.
\STATE \label{alg:LRargmax} $({\bS_t}, \bs_t, \tilde{\btheta}_t) = \argmax_{\doi(\bS=\bs)\in\mathcal{A}, \btheta'_{t,X} \in \mathcal{C}_{t,X} } \E[Y| \doi({\bS}=\bs)]$.
\STATE Intervene all the nodes in ${\bS}_t$ to $\bs_t$ and observe the feedback $(\bX_t,Y_t)$.
\FOR{$X\in\bX\cup\{Y\}$}
\STATE Construct data pair $(\bV_{t,X},X^{(t)})$ with $\bV_{t,X}$ the vector of ancestors of $X$ in round $t$, and $X^{(t)}$ the value of $X$ in round $t$ if $X\not\in S_t$.
\STATE $M_{t,X}=M_{t-1,X}+\bV_{t,X}\bV_{t,X}^\intercal$, $\bb_{t,X}=\bb_{t-1,X}+X^{(t)}\bV_{t,X}$, $\hat{\btheta}_{t,X}=M^{-1}_{t,X}\bb_{t,X}$.
\ENDFOR
\ENDFOR
\end{algorithmic}
\end{algorithm}

The following theorem shows the regret bound of BLM-LR-Unknown-SG. It is not surprising that this algorithm could also work on linear models with continuous variables as Appendix F in \cite{feng2022combinatorialarxiv}.
The dominant term in the expected regret does not increase compared to BLM-LR in \cite{feng2022combinatorial}.
\begin{restatable}[Regret Bound of BLM-LR-Unknown-SG]{theorem}{thmlinearlr}
\label{thm.blmlr}
The regret of BLM-LR-Unknown-SG running on BLM or linear model is bounded as
\begin{align*}
    R(T)=O\left(n^{\frac{5}{2}}\sqrt{T}\log T\right),
\end{align*}
where the terms of $o(\sqrt{T}\ln T)$ are omitted, and the big $O$ notation holds for $T\geq 32\left(\frac{c_1}{\kappa\theta^*_{\min}}\right)^5$.
\end{restatable}

\begin{proof}
In the following proof on $G$, when $X'\in\Anc(X)$ but $X'\notin\Pa(X)$, we add an edge $X'\rightarrow X$ with weight $\theta_{X',X}=0$ into $G$ and this does not impact the propagation results of $G$. After doing this transformation, $D=n$ and $\Anc(X)=\Pa(X)$ for all $X\in\bX\cup\{Y\}$.

According to Lemma~\ref{lemma.correctorder}, with probability at least $1-2\binom{n-1}{2}\exp\left(-\frac{c_0c_1^2T^{1/10}}{2}\right)$, Algorithm~\ref{alg:bglm-order} returns a correct ancestor-descendant relationship, i.e., $\Anc(X)=\widehat{\Anc}(X)$ for $X\in\bX\cup\{Y\}$. Moreover, by Lemma 11 in \cite{feng2022combinatorialarxiv}, with probability at most $n\delta$, event $\left\{\exists T_0<t\leq T,x\in\bX\cup\{Y\}:\left\|\btheta^{*'}_X-\hat{\btheta}_{t,X}\right\|>\rho_t\right\}$ occurs. Now we bound the expected regret conditioned on the absence of this event and finding a correct ancestor-descendant relationship. For $T_0<t\leq T$, according to Theorem 1 in \cite{li2020onlineinference} and Theorem \ref{thm.gom_glm}, we can deduce that \begin{align*}
    \mathbb{E}\left[R_t\right]&=\mathbb{E}\left[\sigma'(\bS^{\text{opt}},\btheta^{*'})-\sigma'(\bS_t,\btheta^{*'})\right]\\
    &\leq \mathbb{E}\left[\sigma'(\bS_t,\tilde{\btheta_{t}})-\sigma'(\bS_t,\btheta^{*'})\right]\\
    &\leq \mathbb{E}\left[\sum_{X\in \bX_{\bS_t,Y}}\left|\bV_{t,X}^\intercal(\tilde{\btheta}_{t,X}-\btheta^{*'}_{X})\right|\right]\\
    &\leq \mathbb{E}\left[\sum_{X\in \bX_{\bS_t,Y}}\left\|\bV_{t,X}\right\|_{M_{t-1,X}^{-1}}\left\|\tilde{\btheta}_{t,X}-\btheta^{*'}_{X}\right\|_{M_{t-1,X}}\right]\\
    &\leq \mathbb{E}\left[\sum_{X\in \bX_{\bS_t,Y}}2\rho_{t-1}\left\|\bV_{t,X}\right\|_{M_{t-1,X}^{-1}}\right],
\end{align*}
since $\tilde{\btheta}_{t,X},\btheta^*_X$ are both in the confidence set. Thus, we have\begin{align*}
    R(T)&=\mathbb{E}\left[\sum_{t=1}^TR_t\right]\leq \mathbb{E}\left[\sum_{t=T_0+1}^TR_t\right]+T_0 \\
    &\leq 2\rho_{T} \cdot\mathbb{E}\left[\sum_{t=T_0+1}^T\sum_{X\in\bX_{\bS_t,Y}}\left\|\bV_{t,X}\right\|_{M_{t-1,X}^{-1}}\right]+T_0.
\end{align*}

For convenience, we define $\bW_{t,X}$ as a vector such that if $X\in S_t$, $\bW_{t,X}={\mathbf 0}^{|\Pa(X)|}$; if $X\not\in S_t$, $\bW_{t,X}=\bV_{t,X}$. According to Cauchy-Schwarz inequality, we have\begin{align*}
    R(T)&\leq 2\rho_T\cdot\mathbb{E}\left[\sum_{t=T_0+1}^T\sum_{X\in\bX\cup\{Y\}}\left\|\bW_{t,X}\right\|_{M_{t-1,X}^{-1}}\right]+T_0\\
    &\leq 2\rho_T\cdot\mathbb{E}\left[\sqrt{T}\cdot\sum_{X\in \bX\cup\{Y\}}\sqrt{\sum_{t=T_0+1}^T\left\|\bW_{t,X}\right\|^2_{M_{t-1,X}^{-1}}}\right]+T_0\\
    &\leq 2\rho_T\cdot\mathbb{E}\left[\sqrt{T}\cdot\sum_{X\in \bX\cup\{Y\}}\sqrt{\sum_{t=1}^T\left\|\bW_{t,X}\right\|^2_{M_{t-1,X}^{-1}}}\right]+2(n-1)c_0T^{1/2}.
\end{align*}

Note that $M_{t,X}=M_{t-1,X}+\bW_{t,X}\bW_{t,X}^\intercal$ and therefore, $$\det\left(M_{t,X}\right)=\det(M_{t-1,X})\left(1+\left\|\bW_{t,X}\right\|^2_{M_{t-1,X}^{-1}}\right),$$ we have\begin{align*}
    \sum_{t=1}^T \left\|\bW_{t,X}\right\|^2_{M_{t-1,X}^{-1}}&\leq\sum_{t=1}^T\frac{n}{\log(n+1)}\cdot\log\left(1+\left\|\bW_{t,X}\right\|^2_{M_{t-1,X}^{-1}}\right)\\
    &\leq \frac{n}{\log(n+1)}\cdot\log\frac{\det(M_{T,X})}{\det({\mathbf I})}\\
    &\leq \frac{n|\Pa(X)|}{\log(n+1)}\cdot\log\frac{\text{tr}(M_{T,X})}{|\Pa(X)|}\\
    &\leq \frac{n|\Pa(X)|}{\log(n+1)}\cdot\log\left(1+\sum_{t=1}^T\frac{\left\|\bW_{t,X}\right\|^2_2}{|\Pa(X)|}\right)\\
    &\leq \frac{nD}{\log(n+1)}\log(1+T).
\end{align*}
Therefore, the final conditional regret $R(T)$ is bounded by\begin{align*}
    R(T)&\leq 2\rho_Tn\sqrt{T\frac{nD}{\log(n+1)}\log(1+T)}+2(n-1)c_0T^{1/2},
\end{align*}
because $\rho_T=\sqrt{D\log(1+TD)+2\log\frac{1}{\delta}}+\sqrt{D}$. When $$\left\{\exists t\in(T_0, T],x\in\bX\cup\{Y\}:\left\|\btheta^{*'}_X-\hat{\btheta}_{t,X}\right\|>\rho_t\right\}$$ does occur or Algorithm~\ref{alg:bglm-order} finds an incorrect order, the regret is no more than $T$. Therefore, the total regret is no more than\begin{small}\begin{align*}
    &\left(2\rho_Tn\sqrt{T\frac{nD}{\log(n+1)}\log(1+T)}+2(n-1)c_0T^{1/2}\right)\left(1-n\delta-2\binom{n-1}{2}\exp\left(-\frac{c_0c_1^2T^{1/10}}{2}\right)\right)\\
    &\quad+T\left(n\delta+2\binom{n-1}{2}\exp\left(-\frac{c_0c_1^2T^{1/10}}{2}\right)\right)\\
    &\leq 2\rho_Tn\sqrt{T\frac{nD}{\log(n+1)}\log(1+T)}+o(\sqrt{T}\ln T)\\
    &=O\left(n^{\frac{5}{2}}\sqrt{T}\log T\right),
\end{align*}\end{small}
which is exactly what we want.

Replacing Lemma 11 in \cite{feng2022combinatorialarxiv} by Lemma 12 in \cite{feng2022combinatorialarxiv}, the above proof for BLMs is still feasible for the regret on linear models without any other modification.
\end{proof}

\begin{remark}\label{remark:linear regression} 
According to the transformation in Section 5.1 of \cite{feng2022combinatorial}, this algorithm also works for some BLMs with hidden variables. Using that transformation, running BLM-LR-Unknown-SG on $G$ is equivalent to running on a Markovian BLM or linear model $G'$, where parameter $\btheta^*$ is also transformed to a new set of parameters $\btheta^{*'}$. Here, we disallow the graph structure where a hidden node has two paths to $X_i$ and $X_i$'s descendant $X_j$ and the paths contain only hidden nodes except the end points $X_i$ and $X_j$.
\end{remark}

\section{Proofs for Propositions in Section~\ref{sec.withoutgap}}

\subsection{Proof of Lemma~\ref{lemma:gap lemma}}
\gaplemma*

\begin{proof}
    First, for each node $X_j$ and its parent $X_i$ with weight $\theta^*_{X_i,X_j}\ge T^{-1/3}$, by Lemma~\ref{lemma.dodifference}, we can have 
    \begin{align*}
        \EE[X_j\mid do(X_i=1)]-\EE[X_j\mid do(X_i=0)] \ge \theta^*_{X_i,X_j}
    \end{align*}
    
    Then each element $X_j^{\left(c_0(2i)T^{2/3}+k\right)}-X_j^{\left(c_0(2i+1)T^{2/3}+k\right)}$ is an i.i.d sample of $Z=X_j\mid_{do(X_i=1)}-X_j\mid_{do(X_i=0)}$ with $\EE[Z]\ge \theta^*_{X_i,X_j}\ge T^{-1/3}$. By the Hoeffding's inequality, if we choose $c_1<1$ and $c_0(1-c_1)^2>\frac{1}{3}$, we have 
    \begin{align*}
        &\Pr\left\{\sum_{k=1}^{c_0T^{2/3}}\left(X_j^{\left(c_0(2i)T^{2/3}+k\right)}-X_j^{\left(c_0(2i+1)T^{2/3}+k\right)}\right)>c_0c_1T^{1/3}\log(T^2)\right\}\\
        &\geq 1-\exp\left(-\frac{2\log(T^2)\left(c_0T^{2/3}\EE[Z]-c_0c_1T^{1/3}\right)^2}{4c_0T^{2/3}}\right)\\
        &\ge 1-\exp\left(-\frac{2\log(T^2)\left(c_0T^{1/3}-c_0c_1T^{1/3}\right)^2}{4c_0T^{2/3}}\right)\\
        &\ge 1-\exp\left(-\frac{c_0(1-c_1)^2\log(T^2)}{2}\right)\\
        &\geq 1-T^{-c_0(1-c_1)^2}\\
        &\ge 1-\frac{1}{T}.
    \end{align*}
    Taking the union bound for all $X$ and $X'$, with probability at least $1-\binom{n-1}{2}\frac{1}{T^2}$, the edge $X'\to X$ with $\theta^*_{X',X}$ will be identified and added to the estimated graph $G'$.
    Also, assume $X_i$ is not an ancestor of $X_j$, then
    \begin{align*}
        \EE[X_j\mid do(X_i=1)]-\EE[X_i\mid do(X_i=0)]=0.
    \end{align*}
    Thus the element $X_j^{\left(c_0(2i)T^{2/3}+k\right)}-X_j^{\left(c_0(2i+1)T^{2/3}+k\right)}$ is an i.i.d sample of $Z'=X_j\mid_{do(X_i=1)}-X_j\mid_{do(X_i=0)}$ with $\EE[Z']=0$. Thus by Hoeffding's inequality,
    \begin{align*}
        &\Pr\left\{\sum_{k=1}^{c_0T^{2/3}}\left(X_j^{\left(c_0(2i)T^{2/3}+k\right)}-X_j^{\left(c_0(2i+1)T^{2/3}+k\right)}\right)>c_0c_1T^{1/3}\log(T^2)\right\}\\
        &\leq \exp\left(-\frac{2\log(T^2)\left(c_0T^{2/3}\EE[Z]-c_0c_1T^{1/3}\right)^2}{4c_0T^{2/3}}\right)\\&\le \exp\left(-c_0c_1^2\log T\right)\\
        &\leq T^{-c_0c_1^2}\\
        &\leq \frac{1}{T}.
    \end{align*}
    and then with probability at least $1-\binom{n-1}{2}\frac{1}{T}$, we will not add the edge $X'\to X$ in the graph G. Combining these two facts, we complete the proof.
\end{proof}

\subsection{Proof of Lemma~\ref{lemma: linearregressionequivalent}}\label{appendix:linearregressionequivalent}
For each node $X$, consider the estimated possible parent $\Pa'(X)$, then our observation $\bV_{t,X} \in \{0,1\}^{\Pa'(X)}$ are the values of $\Pa'(X)$. Since we have $\theta'$ that 
\begin{align}
    \mathbb{E}[X_t\mid \bV_{t,X}] = \btheta_{t,X}^T\bV_{t,X}.
\end{align}
Thus applying Lemma 1 in \cite{li2020onlineinference}, we can have 
\begin{align}
    |\btheta_X'-\btheta_{t,X}'|_{M_{t,X}}\le \sqrt{n\log(1+tn)+2\log(1/\delta)}+\sqrt{n}. 
\end{align}

\subsection{Proof of Lemma~\ref{lemma:diff}}\label{appendix:diff}
Note that $M$ represents the model with true graph $G$ and true weights $\btheta$, and $M'$ represents the model with estimated graph $G'$ and estimated weights $M'$, then difference \begin{align}|\theta'_{X_i,X}-\theta_{X_i,X}|\le nr\label{ineq:thetadiff}\end{align}

Now we construct a auxillary model $M''$, which has graph $G'$ and weights $\theta$ on it. The parent of $X$ in model M $\Pa''(X)$ is equivalent to $\Pa'(X)$. Then we prove the following two claims:
\begin{claim}\label{claim:MM''}
    $|\mathbb{E}_{M}[Y\mid do(S=\textbf{1})]-\mathbb{E}_{M''}[Y\mid do(S=\textbf{1})]|\le n^2r$.
\end{claim}

\begin{proof}
    Let the topological order be $X_1,X_2,\ldots,X_n$.
    First, $\mathbb{E}_{M}[X_1\mid do(S)]-\mathbb{E}_{M''}[X_1\mid do(S)]=0 \le nr$ because $X_1$ is always 1.
       Assume $X_{q+1}\notin S$ $\mathbb{E}_{M}[X_i\mid do(S)]-\mathbb{E}_{M''}[X_i\mid do(S)]\le qnr$ for all $i\le q$, then if $X_{q+1}\in S$, $\mathbb{E}_{M}[X_{q+1}\mid do(S)]-\mathbb{E}_{M''}[X_{q+1}\mid do(S)]=0\le (q+1)nr$ holds trivially. Thus now we assume $X_{q+1}\notin S$.
\begin{align*}
    &\ \ \ \mathbb{E}_{M}[X_{q+1}\mid do(S)]-\mathbb{E}_{M''}[X_{q+1}\mid do(S)] \nonumber\\&= \mathbb{E}_{M}\left[\sum_{X_i \in \Pa(X_{q+1})} \theta_{X_i,X_{q+1}}X_i\Bigg| do(S)\right]-\mathbb{E}_{M''}\left[\sum_{X_i \in \Pa''(X_{q+1})} \theta_{X_i,X_{q+1}}X_i\Bigg| do(S)\right]\\
    &= \sum_{X_i \in \Pa''(X_{q+1})}\theta_{X_i,X_{q+1}}(\mathbb{E}_{M}[X_i\mid do(S)]-\mathbb{E}_{M''}[X_i\mid do(S)]) +\nonumber \\  &\ \ \ \ \ \ \ \ \sum_{X_i\in \Pa(X_{q+1})\setminus \Pa''(X_{q+1})}\theta_{X_i,X_{q+1}}\mathbb{E}_{M}[X_i\mid do(S)]\\
    &\le \sum_{X_i \in \Pa(X_{q+1})}\theta_{X_i,X_{q+1}} qnr+rn\label{ineq:induction}\\
    &\le (q+1)nr
\end{align*}
where the first equality follows the definition of linear model, the second equality is because $\theta'_{X',X}=0$ if $X'$ is not a true parent of $X$ in $G$. The third inequality is derived by induction, and the last inequality is because $\Vert\theta_{X', X_{q+1}}\Vert_1\le 1$.
\end{proof}

\begin{claim}\label{claim:M'M''}
     $|\mathbb{E}_{M'}[Y\mid do(S=\textbf{1})]-\mathbb{E}_{M''}[Y\mid do(S=\textbf{1})]|\le n^3r$.
\end{claim}

\begin{proof}
    First, $\mathbb{E}_{M}[X_1\mid do(S)]-\mathbb{E}_{M''}[X_1\mid do(S)]=0 \le n^2r$
    Then similarly, 
    assume $\mathbb{E}_{M}[X_i\mid do(S)]-\mathbb{E}_{M''}[X_i\mid do(S)]\le qn^2r$ for all $i\le q$ and $X_{q+1}\notin S$.   Then 
\begin{align*}
     &\ \ \ \ \ \mathbb{E}_{M'}[X_{q+1}\mid do(S)]-\mathbb{E}_{M''}[X_{q+1}\mid do(S)] \\&= \mathbb{E}_{M'}\left[\sum_{X_i \in \Pa'(X_{q+1})} \theta'_{X_i,X_{q+1}}X_i\Bigg| do(S)\right]-\mathbb{E}_{M''}\left[\sum_{X_i \in \Pa''(X_{q+1})} \theta_{X_i,X_{q+1}}X_i\Bigg| do(S)\right]\\
     & = \sum_{X_i \in \Pa''(X_{q+1})}\theta'_{X_i,X_{q+1}}\mathbb{E}_{M'}[X_i\mid do(S)]-\theta_{X_i, X_{q+1}}\mathbb{E}_{M''}[X_i\mid do(S)]\\
     &= \sum_{X_i \in \Pa''(X_{q+1})}(\theta'_{X_i,X_{q+1}}-\theta_{X_i, X_{q+1}})\mathbb{E}_{M'}[X_i\mid do(S)] + \\ & \ \ \ \ \ \ \ \ \ \ \ \ \  \sum_{X_i \in \Pa''(X_{q+1})}\theta_{X_i, X_{q+1}}(\mathbb{E}_{M'}[X_i\mid do(S)]-\mathbb{E}_{M''}[X_i\mid do(S)])\\
     &= n^2r+n^2qr\\&\le (q+1)n^2r.
\end{align*}
where the first equality follows the definition, the second equality is because $\Pa'(X) = \Pa''(X)$ for any node $X$. The fourth inequality derived from induction , inequality \eqref{ineq:thetadiff} and $X_i \in [0,1]$. By induction, we complete the proof.
\end{proof} 

Now we prove the Lemma~\ref{lemma:diff}:

\begin{proof}Combining Claim~\ref{claim:MM''} and Claim~\ref{claim:M'M''}, we have 
\begin{align}
    \mathbb{E}_{M}[Y\mid do(S)]-\mathbb{E}_{M'}[Y\mid do(S)]\le n^2(n+1)r.
\end{align}
\end{proof}

\subsection{Proof of Theorem~\ref{thm:nogap}}
\begin{proof}
    Denote the original model and estimated model as $M$ and $M'$
    The initialization phase will lead to regret at most $T_0=16(n-1)T^{2/3}$. At Iterative phase, denote the optimal action to be $do(\bS^* = \textbf{1})$, by Lemma~\ref{lemma:gap lemma} and the guarantee of BLM-LR, with probability at least $1-(n-1)(n-2)\frac{1}{T}$
    \begin{align*}
        &\ \ \ \ \ \ \sum_{t=1}^T\mathbb{E}_M[Y\mid do(\bS^*=\textbf{1})] - \mathbb{E}_M[Y\mid do(\bS_t = \textbf{1})]\\& =\sum_{t=1}^T( (\mathbb{E}_M[Y\mid do(\bS^*=\textbf{1})] - \mathbb{E}_{M'}[Y\mid do(\bS^*=\textbf{1})]) \\
        &\quad\quad+ (\mathbb{E}_{M'}[Y\mid do(\bS^*=\textbf{1})]-\mathbb{E}_{M'}[Y\mid do(\bS_t = \textbf{1})]))\\&\le T_0+\sum_{t=T_0+1}^{T} n^2(n+1)T^{-1/3}+\sum_{t=T_0+1}^T(\mathbb{E}_{M'}[Y\mid do(\bS^*=\textbf{1})]-\mathbb{E}_{M'}[Y\mid do(\bS_t = \textbf{1})])\\
        &\le T_0+n^2(n+1)T^{2/3}+cn^2\sqrt{nT\log T}\\
        & = O((n^3T^{2/3}+n^3\sqrt{T})\log T)
        \\&= O(n^3T^{2/3}\log T),
    \end{align*}
    where the first inequality is derived from Lemma~\ref{lemma:diff}, and the second inequality is the guarantee of BLM-LR in Theorem 3 of \cite{feng2022combinatorial}.

    Thus the total regret will be bounded by 
    \begin{align*}
        R(T) &\le \frac{(n-1)(n-2)}{T}\cdot T + O((n^3T^{2/3})\log T) \\&= O((n^3T^{2/3})\log T).
    \end{align*}
    The first inequality is because our regret have an upper bound $T$.
\end{proof}

\subsection{Proof of Theorem~\ref{thm:binarylowerbound}}\label{appendix:binarylowerbound}
\begin{proof}
Consider the causal bandit instances $\mathcal{T}_i$ with parallel graph $(E = \{X_i\to Y, 1\le i\le n\}.)$ and $\bA = \{do(), do(X=x), do(\bX = \bx)\}$ for all node $X$, $x \in \{0,1\}$, $\bx \in \{0,1\}^n$ be all observation, atomic intervention and actions that intervene all nodes.

For $\mathcal{T}_1$, we assume $X_i$ are independent with each other and $P(X_i=1)=P(X_i=0)=0.5$. Define 
\begin{align*}
    P(Y=1)=\left\{\begin{aligned}
        &0.5+\Delta\ \ \ \ \ \ \ \ \ \ \ \ &\mbox{if}\  X_1=X_2=\cdots=X_n=0\\
        &0.5 \ \ \ &\mbox{otherwise}
    \end{aligned}\right.
\end{align*}

Then for $\mathcal{T}_i, 2\le i\le 2^n$, consider the binary representation of $i-1$ as $\overline{b_1b_2\ldots b_n}$. Then assume $X_i$ are independent with each other and $P(X_i=1)=0.5$, and define 
\begin{align*}
    P(Y=1)=\left\{\begin{aligned}
        &0.5+\Delta\ \ \ \ \ \ \ \ \ \ \ \ &\mbox{if}\  X_1=X_2=\cdots=X_n=0\\
        &0.5+2\Delta \ \ \ &\mbox{if}\ X_j=b_j\ \mbox{for all}\ 1\le j\le n\\
        &0.5  & \mbox{otherwise}
    \end{aligned}\right.
\end{align*}

Now in $\mathcal{T}_i$, $do\left(\bX = \overline{b_1b_2\ldots b_n}\right)$ is the best action, and other actions will lead to at least $\Delta$ regret.

Denote $T_a(t)$ for action $a \in \bA$ as the number of times taking $a$ until time $t$. To simplify the notation, we denote $a_i$ as $do(\bX = \bx)$, where $\bx$ is the binary representation of $i-1$, $\{b_1,b_2,\ldots,b_n\}$. Then for instances $\mathcal{T}_1$ and $\mathcal{T}_i$, we have 
\begin{align*}
    \mathbb{E}_{\mathcal{T}_1}[R(t)]\ge \mathbb{P}_{\mathcal{T}_1}(T_{a_1}(t)\le t/2)\frac{t\Delta}{2}, \ \
     \ \mathbb{E}_{\mathcal{T}_i}[R(t)]
    \ge \mathbb{P}_{\mathcal{T}_i}(T_{a_1}(t)>t/2)\frac{t\Delta}{2}.
\end{align*}
Thus 
\begin{align*}
    \mathbb{E}_{\mathcal{T}_1}[R(t)]+\mathbb{E}_{\mathcal{T}_i}[R(t)]&>\frac{t\Delta}{2}(\mathbb{P}_{\mathcal{T}_1}(T_{a_1}(t)\le t/2)+\mathbb{P}_{\mathcal{T}_i}(T_{a_1}(t)>t/2))\\&\ge \frac{t\Delta}{4}\exp\left(-\mbox{KL}(\mathbb{P}_{\mathcal{T}_1}, \mathbb{P}_{\mathcal{T}_i})\right).
\end{align*}
Now we need to bound $\mbox{KL}(\mathbb{P}_{\mathcal{T}_i}, \mathbb{P}_{\mathcal{T}_1})$.

\begin{small}
\begin{align}
    \mbox{KL}(\mathbb{P}_{\mathcal{T}_1}, \mathbb{P}_{\mathcal{T}_i}) &\le \sum_{a \in \bA}\mathbb{E}_{\mathcal{T}_1}[T_{a}(t)]\mbox{KL}(\mathbb{P}_{\mathcal{T}_1}(\bX, Y\mid a) \Vert \mathbb{P}_{\mathcal{T}_i}(\bX, Y\mid a))\\
    &= \sum_{a \in \bA}\mathbb{E}_{\mathcal{T}_1}[T_{a}(t)]\mbox{KL}(\mathbb{P}_{\mathcal{T}_1}(Y\mid a) \Vert \mathbb{P}_{\mathcal{T}_i}(Y\mid a))\\
    &\le \mathbb{E}_{\mathcal{T}_1}[T_{a_i}(t)]\cdot \mbox{KL}(0.5\Vert 0.5+2\Delta) + \sum_{a = do(X_i=x), do()}\mathbb{E}_{\mathcal{T}_1}[T_{a}(t)]\cdot\mbox{KL}(0.5\Vert 0.5+\frac{\Delta}{2^{n-2}})\label{ineq:keyKL}\\
    &\le \mathbb{E}_{\mathcal{T}_1}[T_{a_i}(n)]\cdot 2\Delta^2 + t\cdot \frac{\Delta^2}{2^{2n-3}}\label{ineq:KLbound},
\end{align}
\end{small}
where \eqref{ineq:keyKL} is because for $a = do(X_i=x)$ or $a = do()$, $P(Y\mid do(a))\ge 0.5$ in $\mathcal{T}_1$, and $P(Y\mid do(a))\le 0.5+\frac{2\Delta}{2^{n-1}} = 0.5+\frac{\Delta}{2^{n-2}}$ in $\mathcal{T}_i$. Now we choose 
\begin{align}
    i = \argmin_{j>1}\mathbb{E}_{\mathcal{T}_1}[T_{a_j}(t)],
\end{align}
then we have 
\begin{align}
    \mathbb{E}_{\mathcal{T}_1}[T_{a_i}(t)]\le \frac{T}{2^n-1}. \label{ineq:minbound}
\end{align}
Then by \eqref{ineq:KLbound}, choosing $\Delta = \sqrt{\frac{2^n-1}{3t}}$, we have 
\begin{align}
    \mbox{KL}(\mathbb{P}_{\mathcal{T}_1}, \mathbb{P}_{\mathcal{T}_i})&\le \frac{2t\Delta^2}{2^n-1}+\frac{t\Delta^2}{2^{2n-3}}\le t\Delta^2\cdot \frac{3}{2^n-1} = 1
\end{align}
Thus 
\begin{align*}
    \mathbb{E}_{\mathcal{T}_1}[R(t)]+\mathbb{E}_{\mathcal{T}_i}[R(t)]&\ge \frac{t\Delta}{4}\exp\left(-\mbox{KL}(\mathbb{P}_{\mathcal{T}_1}, \mathbb{P}_{\mathcal{T}_i})\right)\\
    &\ge \frac{t\Delta}{4e}\\
    & \ge \frac{\sqrt{(2^n-1)t}}{4\sqrt{3}e}\\
    &\ge \frac{\sqrt{2^nt}}{8e}.
\end{align*}
Then $\max\{\mathbb{E}_{\mathcal{T}_1}[R(t)], \mathbb{E}_{\mathcal{T}_i}[R(t)]\}\ge \frac{\sqrt{2^nt}}{16e}$. We complete the proof when $t\ge \frac{16(2^n-1)}{3}$.

Now suppose $t\le \frac{16(2^n-1)}{3}$, choose $\Delta = \frac{1}{4}$, then based on \eqref{ineq:KLbound} and \eqref{ineq:minbound}, we have 
\begin{align*}
    \mbox{KL}(\mathbb{P}_{\mathcal{T}_1}, \mathbb{P}_{\mathcal{T}_i})&\le \frac{t}{8(2^n-1)}+\frac{t}{2^{2n+1}}\\
    &\le \frac{2}{3}+\frac{16}{3}\cdot \frac{2^n-1}{2^{2n+1}}\\
    &\le 1.
\end{align*}
Then we have 
\begin{align*}
    \mathbb{E}_{\mathcal{T}_1}[R(t)]+\mathbb{E}_{\mathcal{T}_i}[R(t)]&\ge \frac{t\Delta}{4}\exp\left(-\mbox{KL}(\mathbb{P}_{\mathcal{T}_1}, \mathbb{P}_{\mathcal{T}_i})\right)\\
    &\ge \frac{t\Delta}{4e}\\
    & \ge \frac{t}{16e},
\end{align*}
and $\max\{\mathbb{E}_{\mathcal{T}_1}[R(t)], \mathbb{E}_{\mathcal{T}_i}[R(t)]\}\ge \frac{t}{32e}$.
\end{proof}

\section{An Explanation of Weight Gap Assumption}
\label{app.explanationweightgap}

The weight gap assumption states that the parameter $\theta_{\min}$ is larger than a term relative to $T$. In Lemma 2, the parameter $\theta_{\min}$ represents the minimum difference between $\mathbb{E}[X_j\mid \doi(X_i=1)]$ and $\mathbb{E}[X_j\mid \doi(X_i=0)]$, where $X_i$ and $X_j$ form a causal edge. Intuitively, this assumption suggests that the causal relationship represented by each edge is sufficiently significant, making it a stronger version of the causal faithfulness assumption. If the causal relationship is too weak to be observed, it may indicate the presence of intermediate factors not accounted for in practice. In such cases, one could address the issue by collecting and observing additional intermediate factors.

Furthermore, it is important to note that the weight gap assumption on $\theta^*_{\min}$ depends on $T$. Therefore, if the weight gap assumption is not satisfied and the intermediate factors are unobservable, the user has two options. The first is to increase the number of rounds until $\theta^*_{\min}\geq 2c_1\kappa^{-1}T^{-1/5}$. Alternatively, BGLM-OFU-Unknown can guarantee an $O(\sqrt{T})$ regret bound for $T\geq 32\left(\frac{c_1}{\kappa\theta^*_{\min}}\right)^5$. The second option is to use BLM-LR-Unknown if $T$ cannot be increased. In this case, a theoretical regret bound of $O(T^{\frac{2}{3}})$ can be achieved.

Therefore, our results account for both scenarios, whether the weight gap assumption is satisfied or not.

\section{Experiments}
\label{sec.experiments}

\subsection{Experiment Results}

We conduct our experiments on a parallel BLM consisting of $7$ nodes, $X_1,\ldots,X_6$, and $Y$, with $X_1$ being the unique always-$1$ node. To simplify the analysis, we apply Algorithms~\ref{alg:glm-ofu} and~\ref{alg:blmlr-nogap} solely to identify the edges between $X_2,\ldots,X_6$ and $Y$. As per the definition of our algorithms, if a node $X_i, 2\leq i\leq 6$ is not a parent of $Y$, it will never be selected for interventions. We set $\mathcal{A}$ to be all interventions budgeted by $2$ nodes. The parameters are set as follows:\begin{align*}
&\theta^*_{X_1,X_2}=\theta^*_{X_1,X_3}=0.3, \theta^*_{X_1,X_4}=\theta^*_{X_1,X_5}=\theta^*_{X_1,X_6}=0.2,\\
&\theta^*_{X_2,Y}=\theta^*_{X_3,Y}=0.3, \theta^*_{X_4,Y}=\theta^*_{X_5,Y}=\theta^*_{X_6,Y}=0.13.
\end{align*}

We run BGLM-OFU-Unknown and BLM-LR-Unknown on this BLM and compare them to the standard Upper Confidence Bound (UCB) algorithm and the $\epsilon$-greedy algorithm ($\epsilon=0.02$) as baseline methods. Additional implementation details can be found in the Appendix~\ref{app.experiment}. Due to computational resource constraints, we run these $4$ algorithms on this BLM for $T=10000, 20000, 40000, 80000$, each executed $50$ times, and compute the average regrets as follows.

\begin{figure}[H]
\centering

\subfigure[$T=10000$]{
\begin{minipage}[t]{0.35\linewidth}
\centering
\includegraphics[width=\linewidth]{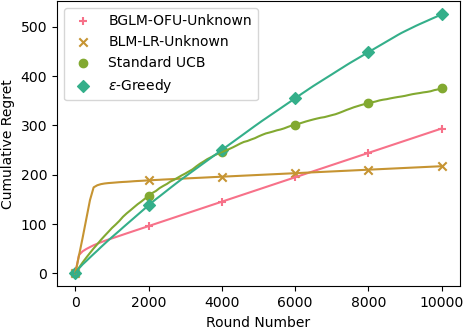}
\end{minipage}%
}%
\subfigure[$T=20000$]{
\begin{minipage}[t]{0.35\linewidth}
\centering
\includegraphics[width=\linewidth]{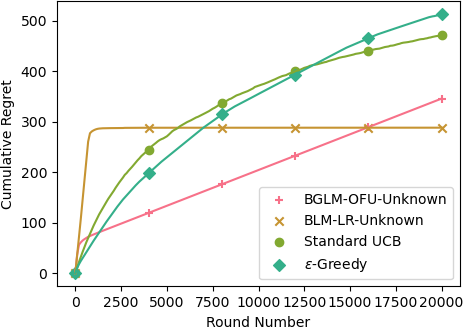}
\end{minipage}%
}%
                 
\subfigure[$T=40000$]{
\begin{minipage}[t]{0.35\linewidth}
\centering
\includegraphics[width=\linewidth]{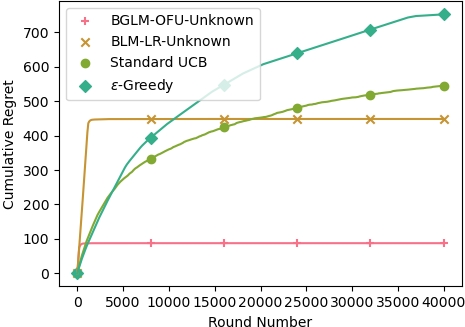}
\end{minipage}
}%
\subfigure[$T=80000$]{
\begin{minipage}[t]{0.35\linewidth}
\centering
\includegraphics[width=\linewidth]{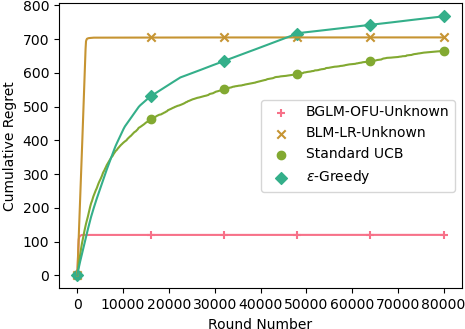}
\end{minipage}
}%
\centering
\label{fig.experiment}
\end{figure}

We can observe from the results that when $T$ is small, BGLM-OFU-Unknown struggles to accurately learn the graph structure, leading to a significant regret. In contrast, BLM-LR-Unknown performs well under these conditions. However, when $T$ is sufficiently large, BGLM-OFU-Unknown is able to consistently identify the correct graph structure, resulting in superior performance compared to all other algorithms.

\subsection{Experiment Settings}
\label{app.experiment}

Due to the limited number of rounds, we adjust $\rho_t$ and $\rho$ to be $\frac{1}{10}$ of our original parameter settings for BGLM-OFU-Unknown and BLM-LR-Unknown. Both algorithms have constants $c_0$ and $c_1$ set to $0.1$. We employ the pair-oracle implementation as described in Appendix H.1 of \cite{feng2022combinatorialarxiv}. When BGLM degenerates to BLM, we remove the second initialization phase (line~\ref{alg:secondinit} of Algorithm~\ref{alg:glm-ofu}) of BGLM-OFU-Unknown by setting $T_1=T_0$. This is because the second-order derivative of a linear function is $0$, making $L_{f_X}^{(2)}$ and $R$ in BGLM-OFU-Unknown arbitrarily small; thus, the minimum eigenvalues of $M_{t,X}$'s should satisfy Lemma~\ref{thm.learning_glm}'s condition after $T_0$ rounds. Additionally, for completeness, we provide the specific BLM used to test our algorithms in Fig.~\ref{fig.experimentblm}.

\begin{figure}[htbp]
\centering
\includegraphics[scale=0.24]{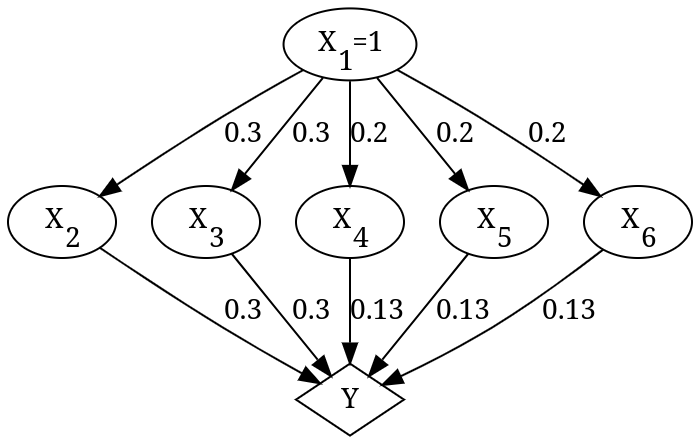}
\caption{The BLM Employed for Evaluating Algorithms~\ref{alg:glm-ofu} and~\ref{alg:blmlr-nogap}}
\label{fig.experimentblm}
\end{figure}

For the standard UCB algorithm, we use the commonly adopted upper confidence bound $\sqrt{\frac{\ln t}{n_{i,t}}}$, where $t$ is the current round number and $n_{i,t}$ is the number of times arm $i$ has been played up to the $t^{th}$ round \citep{slivkins2019introduction}. For the $\epsilon$-greedy algorithm, we set $\epsilon=0.02$, a typical implementation. We tested various settings for these two baselines, and our choices are near-optimal for BLMs. For both baselines, we treat each possible $2$-node intervention set as an arm, resulting in a total of $\binom{7-2}{2}=10$ arms. All experiments were executed using Python in a multithreaded environment on Arch Linux, utilizing $4$ performance cores of an Intel Core\texttrademark~i7-12700H Processor at 4.30GHz with 32GB DDR5 SDRAM. The total execution time amounts to $1687$ seconds. Our Python implementation can be found in the supplementary material.

\section{Pure Exploration of Causal Bandits without Graph Structure}\label{sec:pureexploration}

Another performance measure for bandit algorithms is called sample complexity. In this setting, 
the agent aims to
	find an action with the maximum expected reward using as small number of rounds as possible. This setting is also called pure exploration. 
 To be more specific, 
 the agent is willing to find $\varepsilon$-optimal arm with probability at least $1-\delta$ by sampling as few rounds as possible for fixed parameter $\varepsilon$ and $\delta$. 
 For pure exploration, we consider the general binary causal model with only null and atomic interventions, and study the gap-dependent bounds, 
 	meaning that the sample complexity depends on the reward gap between the optimal and suboptimal actions.
Moreover, let $a^*$ be one of the optimal actions. For each action $a = do(X_i=x)$, define $\mu_{a} = \mathbb{E}[Y\mid a]$ and the gap for action $a$ to be
\begin{align}
    \Delta_a = \left\{
\begin{array}{lcc}
\mu_{a^*}-\max_{a \in \bA \setminus \{a^*\}}\{\mu_{a} \},      &      & {a = a^*}; \\
\mu_{a^*}-\mu_{a},     &      & {a\neq a^*}.\\
\end{array} \right.
\end{align}
Here, $\Delta_a$ can be $0$.

According to the causal discovery literature~\citep{pearl2009causality}, by passive observations alone one can obtain an essential graph of the causal graph, with some edge
	directions unidentified.
We assume that the essential graph is known but the exact graph structure is unknown,
which is also considered by \cite{lu2021causal}, with additional assumptions on the graph.

One naive solution for this problem is to first identify the graph structure and then to performed the pure exploration algorithm of causal bandits with known graph \citep{xiong2022pure}. Define 
 $c_e=|P(X\mid do(X'=1))-P(X'\mid do(X=0))|$ for each edge $e=(X,X')$ and $c_X = \min_{e:X\to X'}\frac{1}{c_e^2}$.  Then this naive solution admits a sample complexity about
 \begin{small}
 \begin{align}
     \Tilde{O}\left(\sum_{a \in S}\frac{1}{\max\{\Delta_a,\varepsilon/2\}^2} + \sum_{x \in X}\frac{1}{c_X^2}\right),
 \end{align}
 \end{small}
 where $S$ is a particular set defined following the previous work \citep{xiong2022pure} and the definition is provided in Appendix~\ref{appendix:pureexploration}. The first term is the sample complexity in \cite{xiong2022pure}, while the second term is the cost for identifying the directions of all edges in the essential graph.

 This naive solution separates the causal discovery phase and learning phase, so it cannot discover the directions adaptively. In Appendix~\ref{appendix:pureexploration}, we propose an adaptive algorithm to discover the edges' directions and 
 	learn the reward distribution in parallel, which can provide a lower sample complexity for some cases.

 However, when the $\Delta_a$ and $c_X$ is small, both the naive algorithm and our algorithms provided in Appendix~\ref{appendix:pureexploration} suffers $\Omega(\frac{n}{\varepsilon^2}\log(1/\delta))$ sample complexity. 
 We claim that pure exploration for the general binary causal model is intrinsically hard due to unknown graph structure. To show this, we state a negative result for pure exploration of causal bandits on unknown graph structure with atomic intervention. It states that even if we have all observation distribution $P(\bX, Y)$ as prior knowledge, we still cannot achieve better sample complexity result than the result in the classical pure exploration problem for the 
 	multi-armed bandit $O(\frac{n}{\varepsilon^2}\log(1/\delta))$.

\begin{theorem}[Lower bound]\label{thm:lowerboundatomic}
    Consider causal bandits with only essential graph and atomic intervention, for any algorithm which can output $\varepsilon$-optimal action with probability at least $1-\delta$, there is a bandit instance with expected sample complexity $\Omega(\frac{n}{\varepsilon^2}\log (1/\delta))$ even if we have all observational distribution $P(\bX, Y)$.
\end{theorem}

Note that if we know distribution $P(\bX, Y)$ and the exact graph structure, we can compute each intervention $P(Y\mid do(X=x))$ by do-calculus because the absence of hidden variables. So Theorem~\ref{thm:lowerboundatomic} shows the intrinsic hardness provided by unknown graph structure. The detailed proof can be found in Appendix~\ref{appendix:pureexploration}.
 
\section{General Causal Bandits without Graph Structure}\label{appendix:pureexploration}
In this section, we only consider the atomic intervention, and provide an algorithm to solve causal bandits with the graph skeleton on binary model. We only consider the atomic intervention setting. An atomic intervention is $do(X = x)$, where $X$ is a node of graph $G$ and $x \in \{0,1\}$. 
\subsection{General Causal Bandit Algorithms}
We first provide the positive results, which provides an algorithm to improve the sample complexity comparing to applying the multi-armed bandit approach directly.
\begin{algorithm}[t]

	\caption{{\textrm Causal-PE-unknown}$(G,A,\varepsilon,\delta)$}
	\label{alg:general}
	\begin{algorithmic}[1]

	\STATE Initialize $t=1$, $T_a(0)=0, \hat{\mu}_a=0$ for all arms $a \in A$, $\cA_{known} = \emptyset$
	
	\FOR{$t=1,2,\ldots,$}
	\STATE $a^{t-1}_h = \argmax_{a \in A}\hat{\mu}_a^{t-1}$
 
        \STATE  \label{line:lucb}$ a_l^{t-1} = \argmax_{a \in A\setminus a_h^{t-1}}(U_a^{t-1})$ 
    
	\IF{$U_{a_l^{t-1}}\le L_{a_h^{t-1}}+\varepsilon$}
 
        \STATE \label{line:endlucb} Return $a_h^{t-1}$
	
        \ENDIF
        
	\STATE Perform $do()$ operation and observe $\bX_t$ and $Y_t$. For $a=do()$, $T_a(t)=T_a(t-1)+1$, $D_a(t)=D_a(t-1), r_{a,\emptyset}(t)=\frac{1}{T_a(t)}\sum_{j=1}^t Y_j,  \ p_{a,\emptyset}(t)=1$.\label{line:observe}
	
	\FOR{$a=do(X=x) \in \cA_{known}$}\label{line:possibleparent}
        \STATE$T_{a, \bz}(t)=T_{a, \bz}(t-1)+\mathbb{I}\{X_t=x, \bP=\bz\},$ $T_{a}(t)=\min_{\bz}\{T_{a,\bz}(t)\},$ where $\bP = \Pa(X)$.
	$D_a(t)=D_a(t-1).$
	
	\STATE Update $r_{a,\bz}(t)=\frac{1}{T_{a,\bz}(t)}\sum_{j=1}^t \mathbb{I}\{X_j=x, \bP_j=\bz\}Y_j$.
	
	\STATE Update $p_{a,\bz}(t)=\frac{1}{t}\sum_{j=1}^t \mathbb{I}\{\bP_j=\bz\}.$
	
	\STATE Estimate $\hat{\mu}_{O,a}(t)=\sum_{z}r_{a,\bz}(t)p_{a,\bz}(t)$ and calculate $[L_{O,a}^t, U_{O,a}^t]$by \eqref{eq:obs and int confidence bound} and \eqref{eq:confidence radius}.

        \ENDFOR\label{line:endpossibleparent}

        \STATE RECOVER-EDGE($a_h^{t-1}$).\label{line:recoveredge}
        \STATE  RECOVER-EDGE($a_l^{t-1}$).\label{line:endrecoveredge}

	\STATE Update empirical mean $\hat{\mu}_{I,a}(t)$ using interventional data%$\hat{\mu}_{I,a}=\frac{1}{D_a(t)}\sum_{j=1}^t(\mathbb{I}\{a_l^{j-1}=a\}Y_l^{j-1}+\mathbb{I}\{a_h^{j-1}=a\}Y_h^{j-1})$ 
        and interventional confidence bound $[L_{I,a}^t, U_{I,a}^t]$
		
	\STATE Update confidence bound $[L_a^t,U_a^t]$ by \eqref{eq:finalconfbound}, $\hat{\mu}_a= (L_a^t+U_a^t)/2$, for each arm $a$.

    \ENDFOR
 \end{algorithmic}
\end{algorithm}

\begin{algorithm}[t]
\caption{RECOVER-EDGE($a$)}
\begin{algorithmic}[1]
\label{alg:recover-edge-unknown}
	\IF{$a=do()$} \STATE Return.
	\ELSE \STATE Assume $a = do(X=x)$. Sample action $do(X=1), do(X=0)$.
        \STATE $D_{a'}(t) = D_{a'}(t)+1$ for $a' = do(X=1)$ and $a' = do(X=0)$.
	\STATE Estimate $P(X'=1\mid do(X=1))$ and $P(X'=1\mid do(X=0))$ using interventional data for neighbor $X'$, where the direction of $(X',X)$ is unknown. 
        \STATE Update the confidence bound $[L_{X'\mid do(X=1)}, U_{X'\mid do(X=1)}]$ and $[L_{X'\mid do(X=0)}, U_{X'\mid do(X=0)}]$ by \eqref{eq:identify confidence bound}.
	
	\IF {$[L_{X'\mid do(X=1)}, U_{X'\mid do(X=1)}]\cap[L_{X'\mid do(X=0)}, U_{X'\mid       
        do(X=0)}] = \emptyset$} \STATE recover $X\to X_i$. \ENDIF
        \IF{$\exists X'$ such that $(X',X)$ is unknown}
        \STATE Choose one such $X'$ and perform $do(X'=1)$ and  $do(X'=0)$.
        \STATE Estimate $P(X=1\mid do(X'=0))$ and $P(X=1\mid do(X'=1))$ using interventional    data.
        \STATE Update the confidence bound $[L_{X\mid do(X'=1)}, U_{X\mid do(X'=1)}]$ and $[L_{X\mid do(X'=0)}, U_{X\mid do(X'=0)}]$ by \eqref{eq:identify confidence bound}.
	
	\IF {$[L_{X\mid do(X'=1)}, U_{X\mid do(X'=1)}]\cap[L_{X\mid do(X'=0)}, U_{X\mid    
        do(X'=0)}] = \emptyset$} \STATE recover $X\to X_i$.\ENDIF
        \STATE $D_{a'}(t) = D_{a'}(t)+1$ for $a' = do(X'=1)$ and $a'=do(X'=0).$

	\ENDIF
	
        \ENDIF
\end{algorithmic}
\end{algorithm}

At each iteration we try to recover the edges' direction in parallel using sub-procedure "RECOVER-EDGE$(a)$" for $a \in A$. For action $a=do(X=x)$, this sub-procedure first performs two interventions $do(X=1)$ and $do(X=0)$, then chooses an undirected edge $(X,X')$ corresponding to $X$ (if exists), and then perform $do(X'=1)$, $do(X'=0)$. The goal of these operations is to estimate the difference between $P(X=1\mid do(X'=0))$ and $P(X=1\mid do(X'=1))$, 
 and also the difference between $P(X'=1\mid do(X=0))$, $P(X'=1\mid do(X=0))$. which decides whether $X'\to X$ or $X\to X'$. By this sub-procedure in parallel, the algorithm estimate the model and recover the edges' direction simultaneously and adaptively. To measure the difficulty for identified the direction of edges, for $e:X\to X'$ we define \begin{gather}c_e = P(X'=1\mid do(X=1))-P(X'=1\mid do(X=0))\label{def:ce}\\ c_a = c_X = \min_{e:X\to X'}c_e.\label{def:ca}\end{gather} $c_e$ measure the difficulty for distinguishing the direction for an edge, and $c_a=c_X$ represents the hardness for discovering all directions corresponding to $X$ and its childs.

The main Algorithm~\ref{alg:general} is followed from \cite{xiong2022pure}. During the algorithm, we add "RECOVER-EDGE" sub-procedure to identify the directions of the unknown edges. This sub-procedure first perform intervention $do(X=0)$ and $do(X=1)$ on the node $X$. Then if there is an edge $(X',X)$ which direction has not been identified, it chooses one such edge and perform $do(X'=1)$ and $do(X'=0)$. Then it constructs the confidence bound for all $P(X'=1\mid do(X=1))$, $P(X'=1\mid do(X=0))$, $P(X=1\mid do(X'=1))$ and $P(X=1\mid do(X'=0))$ based on Hoeffding's concentration bound. In fact, assume there are $D_a(t)$ samples for $a = do(X'=x), x \in \{0,1\}$ until round $t$, then the confidence bound for $X$ conditioning on $do(X'=x)$ is defined by 
\begin{equation}
\begin{aligned}\label{eq:identify confidence bound}
    [L_{X\mid do(X'=x)}, U_{X\mid do(X'=x)}] &= \left[\hat{P}(X=1\mid do(X'=x))- \sqrt{\frac{2}{D_a(t)}\log \frac{4n^2t^2}{\delta}},\right. \\&\quad\quad\left.\hat{P}(X=1\mid do(X'=x)) + \sqrt{\frac{2}{D_a(t)}\log \frac{4n^2t^2}{\delta}} \right],
\end{aligned}
\end{equation}
where $n$ is the number of nodes, and $\hat{P}(X=1\mid do(X'=x))$ are the empirical mean of $P(X=1\mid do(X'=x))$ using all these $D_a(t)$ samples for $do(X'=x)$. Other confidence bounds define in this way similarly.

Moreover, at iteration $t$, Line~\ref{line:lucb}-Line~\ref{line:endlucb} first choose two actions $a_h^{t-1}$ and $a_l^{t-1}$ through LUCB1 algorithm. Then, we use $\cA_{known}$ to represent all nodes actions $do(X=x)$ where all the edges corresponding to $X$ are identified. In fact, if all the edges corresponding to $X$ are identified, we can find the true parent set $\Pa(X)$.  Then we can use do-calculus to estimate the causal effect:
\begin{align}\label{eq:docalculus}
    \EE[Y\mid do(X=x)] = \sum_{z}P(Y\mid X=x, Z=z)P(Z=z).
\end{align}

Line~\ref{line:possibleparent}-\ref{line:endpossibleparent} enmurates all these actions, and calculate corresponding confidence bound. The confidence bound is calculated by  
\begin{align}\label{eq:finalconfbound}
    [L_a^t,U_a^t]=[L_{O,a}^t, U_{O,a}^t]\cap [L_{I,a}^t, U_{I,a}^t],
\end{align}
where the first term $[L_{O,a}^t, U_{O,a}^t] = (-\infty, \infty)$ for $a = do(X=x)$ if the parents of $X$ are not sure at time $t$.  In fact, if we do not discover all the edges corresponding to $X$, we cannot estimate the causal effect $\EE[Y\mid do(X=x)]$ using do-calculus. For nodes which parent set is identified, we calculate \begin{equation}\label{eq:obs and int confidence bound}
\begin{aligned}
    \relax[L_{O,a}^t, U_{O,a}^t] &= [\hat{\mu}_{O,a}(t)-\beta_{O,a}(t), \hat{\mu}_{O,a}(t) + \beta_{O,a}(t)],\\ \relax[L_{I,a}^t, U_{I,a}^t] &= [\hat{\mu}_{I,a}(t)-\beta_{I,a}(t), \hat{\mu}_{I,a}(t) + \beta_{I,a}(t)]
\end{aligned}
\end{equation}
The term $\hat{\mu}_{O,a}$ is calculated by estimating all terms at the right side of \eqref{eq:docalculus} empirically, and confidence radius is given by 
\begin{align}\label{eq:confidence radius}
    \beta_{O,a}(t) = \sqrt{\frac{12}{T_a(t)}\log\frac{16n^2Z_at^3}{\delta}}, \beta_{I,a}(t) = 2\sqrt{\frac{1}{D_a(t)}\log \frac{2n\log (2t)}{\delta}}
\end{align}
Similar to \cite{xiong2022pure}, we can prove it is a valid confidence radius, which means that the true effect $\mu_{O,a}$ will fall into the confidence bound $[L_{O,a}^t, U_{O,a}^t]$ with a high probability. 

Line~\ref{line:recoveredge}-\ref{line:endrecoveredge} try to recover the edge for action chosen by LUCB1 algorithm. At the end of this iteration, the algorithm updates all parameters and confidence bounds. 

To represent the complexity result, we first provide the definition of gap-dependent threshold in \cite{xiong2022pure}:
For $a = do(X=x)$ and one possible configuration of the parent $\bz \in \{0,1\}^{|\Pa(X)|}$, define $q_{a,\bz} = P(X=x, \Pa(X) = \bz)$ and $q_a = \min_{\bz} \{q_{a,\bz}\}$.
Then sort the arm set as $ q_{a_1}\cdot\max\{\Delta_{a_1},\varepsilon/2\}^2\le q_{a_2}\cdot \max\{\Delta_{a_2},\varepsilon/2\}^2\le \ldots \le q_{a_{|\bA|}}\cdot \max\{\Delta_{a_{|\bA|}},\varepsilon/2\}^2$. Recall that $\Delta_a = \mu^* - \mu_a$ is the reward gap between the optimal reward and the reward of action $a$. Then $H_r$ is defined by 
\begin{align}
    H_r = \sum_{i=1}^r \frac{1}{\max\{\Delta_{a_i},\varepsilon/2\}^2}. \label{def:H}
\end{align}
\begin{definition}[Gap-dependent observation threshold \citep{xiong2022pure}]
\label{def:gapm}
For a given causal graph $G$ and its associated $q_a$'s and $\Delta_a$'s, 
	the {\em gap-dependent observation threshold}
	$m_{\varepsilon, \Delta}$ is defined as:
   \begin{align} \label{eq:gapmdef}
   	\ \ m_{\varepsilon, \Delta} \nonumber = \min\left\{\tau: \left|\left\{a\in \bA \Bigg| q_a \max\left\{\Delta_a,\varepsilon/2\right\}^2<\frac{1}{H_{\tau}}\right\}\right|\le \tau\right\}.\end{align}
\end{definition}

Denote action set $S = \{a \in \bA: q_a \max\{\Delta_a, \varepsilon/2\}^2<\frac{1}{H_{m_{\varepsilon,\Delta}}}\}$ are all actions which $q_a$ is relatively small, then $|S|\le m_{\varepsilon,\Delta}$. Intuitively, action $a$ with smaller $q_a$ are harder to be estimated by observation: If we assume $q_a = q_{a,\bz}$ for a fixed vector $\bz$, then $P(X=x, \Pa(X) = \bz)$ is hard to observe and estimate by empirical estimation. Thus $S$ contains all actions that are relatively hard to observe, so it is more efficient to estimate $\mu_a$ by intervention for $a \in S$.
Based on this definition, we can provide the final sample complexity result:
\begin{theorem}
    \label{thm:easy-unknowngraph}
Denote $H=\sum_{a \in S}\frac{1}{\max\{\Delta_a,\varepsilon/2\}^2}+\sum_{a\notin S}\min\{\frac{1}{\max\{\Delta_a,\varepsilon/2\}^2}, \frac{1}{c_a^2}+\sum_{e:X'\to X}\frac{1}{c_e^2}\}$. With probability $1-4\delta$, Algorithm \ref{alg:general} will return a $\varepsilon$-optimal arm with sample complexity bound at most 
\begin{equation*}
    T=O\left(H\log \left(\frac{nZH}{\delta}
    \right)\right),
\end{equation*}
where $c_e,c_a$ is defined in \eqref{def:ce} and \eqref{def:ca}.
\end{theorem}

The result can be explained in an intuitive way. The first term of $H$ is the summation of all actions in $S$. As we discussed above, it is more efficient to estimate the $\mu_a$ with intervention for $a \in S$. Thus, this summation can be regarded as the sample complexity applying multi-armed bandit algorithm (e.g. LUCB1) directly. The second term is to estimate the actions by observation. For each action $a = do(X=x)$ with larger $q_a$, we can first identify the edge's direction corresponding to the node $X$, and then using do-calculus to estimate the reward. The term $\frac{1}{c_a^2}+\sum_{e:X'\to X}\frac{1}{c_e^2}$ represents the complextity to identify the directions, and the complexity for using do-calculus can be contained in the first term $\sum_{a \in S}\frac{1}{\max\{\Delta_a,\varepsilon/2\}^2}$ because of the definition of gap-dependent observation threshold. Also, the term $\min\{\frac{1}{\max\{\Delta_a,\varepsilon/2\}^2}, \frac{1}{c_a^2}+\sum_{e:X'\to X}\frac{1}{c_e^2}\}$ is because when we are discovering the edges' direction, if the reward can be estimated by intervention accurately, we turn to use interventional estimation and give up the causal discovery for this node.
The detailed proof can be found in the Section~\ref{subsec: proof of general}.

Even if these two mechanisms can reduce the sample complexity, at the worst case the complexity also degenerates to $O(n/\varepsilon^2)$, which is equal to the complexity for multi-armed bandit. We provide a lower bound to show that this problem cannot be avoided.

\begin{theorem}[Lower bound]\label{thm:generallowerbound}
    Consider causal bandits with only essential graph and atomic intervention, for any $(\varepsilon,\delta)-PAC$ algorithm, there is a bandit instance with expected sample complexity $\Omega(\frac{n}{\varepsilon^2}\log (1/\delta))$ even if we have all observational distribution $P(\bX, Y)$.
\end{theorem}

Theorem~\ref{thm:generallowerbound} states that even if we receive all observational distribution, which shows the intrinsic hardness for unknown graph. Indeed, the proof of lower bound shows that the unknown direction will lead to different interventional effects even when the observational distribution are the same, leading to a unavoidable hardness. 

\subsection{Proof of Theorem~\ref{thm:easy-unknowngraph}}\label{subsec: proof of general}
First, fixed an action $a = do(X_i=x)$,  $\bz \in \{0,1\}^{|\Pa(X)|}$ , then $T_{a,\bz}(t) = \sum_{j=1}^t \II\{X_{j,i}=x, \Pa(X_i)_j = \bz\}$ and the empirical mean $\hat{q}_{a,\bz}(t) = T_{a,\bz}(t)/t.$ Then denote $2^{|\Pa(X)|} = Z_a$, if $q_{a,\bz}(t)\ge \frac{6}{t}\log(2nZ_a/\delta)$, with probability at least $1-\frac{\delta}{2nZ_a}$, we can have 
\begin{align*}
    |\hat{q}_{a,\bz}(t)-q_{a,\bz}(t)|< \sqrt{\frac{6q_{a,\bz}(t)}{t}\log\left(\frac{2nZ_a}{\delta}\right)}
\end{align*}
Hence 
	\begin{gather}\label{eq_28}
		\hat{q}_{a}(t)=\min_{\bz}\{\hat{q}_{a,\bz}(t)\}\le \min_{\bz}\{q_{a,\bz}+\sqrt{\frac{6q_{a,\bz}}{t}\log \frac{2nZ_a}{\delta}}\}= q_{a}+\sqrt{\frac{6q_a}{t}\log \frac{2nZ_a}{\delta}}.
	\end{gather}
	When $q_a\ge \frac{3}{t}\log \frac{2nZ_a}{\delta}$,
	$f(x)=x-\sqrt{\frac{6x}{t}\log \frac{2nZ_a}{\delta}}$ is a increasing function.
	\begin{gather}\label{eq_29}
		\hat{q}_{a}(t)\ge \min_{\bz}\{q_{a,\bz}-\sqrt{\frac{6q_{a,\bz}}{t}\log \frac{2nZ_a}{\delta}}\} =  q_{a}-\sqrt{\frac{6q_a}{t}\log \frac{2nZ_a}{\delta}}.
	\end{gather}
So define the event as 
\begin{align*}
    \cE_1(t) = \left\{\forall a \in \bA\ \mbox{with}\ t\ge \frac{6}{q_a}\log\left(\frac{2nZ_a}{\delta}\right), |\hat{q}_a(t)-q_a|\le \sqrt{\frac{6q_a}{t}\log \left(\frac{2nZ_a}{\delta}\right)}\right\}
\end{align*}
then $\Pr\{\cE_1^c(t)\}\le \delta$, where $\cE^c$ means the complement of the event $\cE$.

Now we consider the concentration bound. First, by classical anytime confidence bound, with probability at least $1-\frac{\delta}{2n}$, for any time $D_a(t) \ge 1$
\begin{align*}
    |\hat{\mu}_{I,a}(t)-\mu_{I,a}|< 2\sqrt{\frac{1}{D_a(t)}\log \left(\frac{2n\log(2D_a(t))}{\delta}\right)} \le 2\sqrt{\frac{1}{D_a(t)}\log \left(\frac{2n\log(2t)}{\delta}\right)}
\end{align*}
Thus define the event as 
\begin{align*}
    \cE_2 = \left\{\forall t, a, |\hat{\mu}_{I,a}(t)-\mu_{I,a}|<2\sqrt{\frac{1}{D_a(t)}\log \left(\frac{2n\log(2t)}{\delta}\right)}\right\},
\end{align*}
then $\Pr\{\cE^c_2\}\le \delta$.

Consider the observational confidence bound. First, if $a \notin \cA_{known}$, $[L_{O,a}^t, U_{O,a}^t] = (-\infty,\infty)$ and then the $\hat{\mu}_{O,a}(t) \in [L_{O,a}^t, U_{O,a}^t]$. Now we consider that if $a = do(X=x) \in \cA_{known}$ and the parent of $X$ is $\bP$.
By Hoeffding's inequality, with probability at least $1-\delta/16n^2Z_at^3$, for $a=do(X=x)$, 
		\begin{align}\label{eq:r bound}
			|r_{a,\bz}(t)-P(Y=1\mid X=x, \bP=\bz)|&>\sqrt{\frac{1}{2T_{a,\bz}(t)}\log\frac{16n^2Z_at^3}{\delta}}
   \end{align}
   Also, by Chernoff's inequality, since $q_a\le P(\bP=\bz)$ for all $\bz \in \{0,1\}^{|\bP|}$, when $t\ge \frac{6}{q_a}\log \left(\frac{16n^2Z_at^3}{\delta}\right)$ with probability at least $1-\delta/16n^2Z_at^3$ we will have 
   \begin{align}\label{eq: p bound}
       |p_{a,\bz}(t)-P(\bP=\bz)|&>\sqrt{\frac{6P(\bP=z)}{t}\log \frac{16n^2Z_at^3}{\delta}},
   \end{align}

then 
\begin{align*}
    \hat{\mu}_{O,a}&= \sum_{\bz}r_{a,\bz}(t)\cdot p_{a,\bz}(t)\nonumber\\
			&\le \sum_{\bz}P(Y=1\mid X=x, \bP=\bz)p_{a,\bz}(t)+\sum_{\bz}p_{a,\bz}(t)\sqrt{\frac{1}{2T_{a,\bz}(t)}\log\frac{16n^2Z_at^3}{\delta}}\nonumber\\
			&\le \sum_{\bz}P(Y=1\mid X=x, \bP=\bz)p_{a,\bz}(t)+ \sqrt{\frac{1}{2T_{a}(t)}\log\frac{16n^2Z_at^3}{\delta}}\nonumber\\
			&\le \sum_{\bz}P(Y=1\mid X=x, \bP=\bz)P(\bP=\bz) + \sum_{\bz}\sqrt{\frac{6P(\bP=\bz)}{t}\log \frac{16n^2Z_at^3}{\delta}} +\nonumber\\ &\notag \sqrt{\frac{1}{2T_{a}(t)}\log\frac{16n^2Z_at^3}{\delta}}\nonumber\\
			&\le \mu_a+\sqrt{\frac{6Z}{t}\log \frac{16n^2Z_at^3}{\delta}} + \sqrt{\frac{1}{2T_a(t)}\log\frac{16n^2Z_at^3}{\delta}}\nonumber\\
                &\le \mu_a+\sqrt{\frac{6}{T_a(t)}\log \frac{16n^2Z_at^3}{\delta}} + \sqrt{\frac{1}{2T_a(t)}\log\frac{16n^2Z_at^3}{\delta}},\nonumber\\
                & = \mu_a+\sqrt{\frac{8}{T_a(t)}\log \frac{16n^2Z_at^3}{\delta}}.
\end{align*}
Also, if $t\le \frac{6}{q_a}\log \frac{16n^2Z_at^3}{\delta}$, first by Chernoff inequality, set $Q = \frac{6}{q_a}\log\frac{16n^2Z_at^3}{\delta}$, then with probability at least $1-\delta/16n^2Z_at^3$, we have 
\begin{align}\label{eq:qa bound}
    \hat{q}_a(Q)\le 2q_a.
\end{align}
by $\cE_1(Q)$.
\begin{align*}
    T_a(t)  \le T_a(Q) \le \hat{q}_{a}(Q)\cdot Q \le 2q_a\cdot Q = \frac{12}{q_a}\log \frac{16n^2Z_at^3}{\delta}.
\end{align*}
Then $\sqrt{\frac{12}{T_a(t)}\log \frac{16n^2Z_at^3}{\delta}}\ge 1$ and the inequality
\begin{align*}
    |\hat{\mu}_{O,a}(t)-\mu_{O,a}|\le \sqrt{\frac{12}{T_a(t)}\log \frac{16n^2Z_at^3}{\delta}}
\end{align*}
also holds.
Thus we define the event 
\begin{align*}
    \cE_3 = \left\{\forall a, t,  |\hat{\mu}_{O,a}(t)-\mu_{O,a}|\le \sqrt{\frac{12}{T_a(t)}\log \frac{16n^2Z_at^3}{\delta}}\right\}
\end{align*}
then by taking the union bound of \eqref{eq:r bound}, \eqref{eq: p bound} and \eqref{eq:qa bound}, 
\begin{align*}
    \Pr\{\cE^c_3\}&\le \sum_{t=1}^\infty \sum_{a \in \bA} \sum_{\bz} 3\cdot \frac{\delta}{16n^2Z_at^3} \\
    &\le \sum_{t=1}^\infty \frac{\delta}{4t^3}\\
    &\le \delta.
\end{align*}

Now we consider how to bound our sample complexity based on events $\cE_1, \cE_2$ and $\cE_3$. First, we provide the following lemma in \cite{xiong2022pure}:
\begin{lemma}[Lemma 6 in \cite{xiong2022pure}]\label{lemma: 6 in xiong}
    Under the event $\cE_1, \cE_2$ and $\cE_3$, at round $t$, if we have 
		    $$\beta_{a_h^t}(t)\le \frac{\max\{\Delta_{a_h^t},\varepsilon/2\}}{4},\beta_{a_l^t}(t)\le \frac{\max\{\Delta_{a_l^t},\varepsilon/2\}}{4}, $$
        where $a_h^t, a_l^t$ are the actions performed by algorithm at round $t$. then the algorithm will stop at round $t+1$.
\end{lemma}

Now assume the algorithm does not terminate at $T_1 = 192H\log(nZT_1^3/\delta)$, where $Z = \max_a Z_a$. 
For $a \in S$, $D_a(t)$.  Note that $H\ge H_{m_{\varepsilon,\Delta}}$.
Thus at round $T_1$, for action $a$ with $q_a\ge\frac{1}{H_{m_{\varepsilon,\Delta}}\cdot \max\{\Delta_a,\varepsilon/2\}^2}\ge \frac{192}{T_1}\log \frac{16nZ_aT_1^3}{\delta}$, if $a \in \cA_{known}$, then under event $\cE_{1}(T_1)$,
 we have 
 \begin{align*}
     \hat{q}_a(T_1)\ge q_a - \sqrt{\frac{6q_a}{T_1}\log \frac{16nZ_aT_1^3}{\delta} }\ge \frac{q_a}{2}.
 \end{align*}
 Then
\begin{align*}
    \beta_{a}(T_1)\le \beta_{O,a}(T_1) = \sqrt{\frac{12}{T_a(T_1)}\log \frac{16n^2Z_at^3}{\delta}}\le \sqrt{\frac{12q_a}{2T_1}} \le \frac{\max\{\Delta_a, \varepsilon/2\}^2}{4}.
\end{align*}

Now we prove that if $D_a(t)$ is large for some $a$, then $a \in \cA_{known}$.
\begin{lemma}\label{lemma:identify}
    With probability at least $1-\delta$, denote $C_a = \frac{1}{c_a^2}+\sum_{e:X'\to X}\frac{1}{c_e^2}$. If $D_a(t)\ge 32 C_a\log (4n^2t^2/\delta)$, $a \in \cA_{known}$.
\end{lemma}
\begin{proof}
    If $D_a(t)\ge 8C_a\log t$, we have called sub-procedure RECOVER-EDGE$(a)$ for $D_a(t)$ times. Then, for each edge $e:X\to X'$, we will perform intervention $do(X=1)$, $do(X=0)$ for at least $D_a(t)$ times and observe the empirical difference $|\hat{P}(X'\mid do(X=1))-\hat{P}(X'\mid do(X=0))|$. By Hoeffding's inequality and union bound on all time $t$ and the $\binom{n-1}{2}$ ordered-pair $(X',X)$, with probability at least $1-\delta$, for all $t \in [T]$ and all $X',X$ we have 
    \begin{align*}
        |\hat{P}(X'\mid do(X=1))-P(X'\mid do(X=1))|&\le \sqrt{\frac{2}{D_a(t)}\log \frac{4n^2t^2}{\delta}}\\
         |\hat{P}(X'\mid do(X=0)-P(X'\mid do(X=0))|&\le \sqrt{\frac{2}{D_a(t)}\log \frac{4n^2t^2}{\delta}}\\
    \end{align*}
    Then for the confidence bounds \begin{small}\begin{align*}
        &[L_{X'\mid do(X=1)}, U_{X'\mid do(X=1)}] \\&= \left[\hat{P}(X'\mid do(X=1))-\sqrt{\frac{2}{D_a(t)}\log \frac{4n^2t^2}{\delta}}, \hat{P}(X'\mid do(X=1))+\sqrt{\frac{2}{D_a(t)}\log \frac{4n^2t^2}{\delta}}\right],\\
         &[L_{X'\mid do(X=0)}, U_{X'\mid do(X=0)}] \\&= \left[\hat{P}(X'\mid do(X=0))-\sqrt{\frac{2}{D_a(t)}\log \frac{4n^2t^2}{\delta}}, \hat{P}(X'\mid do(X=0))+\sqrt{\frac{2}{D_a(t)}\log \frac{4n^2t^2}{\delta}}\right],
    \end{align*}\end{small}
    the intersection 
    \begin{align*}
         [L_{X'\mid do(X=1)}, U_{X'\mid do(X=1)}]\cap [L_{X'\mid do(X=0)}, U_{X'\mid do(X=0)}] = \emptyset,
    \end{align*}
    since 
    \begin{align*}
        &|\hat{P}(X'\mid do(X=1)-\hat{P}(X'\mid do(X=0))|\\\ge & |P(X'\mid do(X=1)-P(X'\mid do(X=0))|-|P(X'\mid do(X=1)-\hat{P}(X'\mid do(X=1))|\\&\ \ \ \ -|P(X'\mid do(X=0)-\hat{P}(X'\mid do(X=0))|\\
        \ge& c_a-2\sqrt{\frac{2}{D_a(t)}\log \frac{4n^2t^2}{\delta}} \\\ge& 2\sqrt{\frac{2}{D_a(t)}\log \frac{4n^2t^2}{\delta}}. 
    \end{align*}
    where we use $D_a(t)\ge \frac{1}{c_a^2}\log\frac{4n^2t^2}{\delta}$.
    Then the edge's direction will be identified correctly.

    Consider the edge $e: X'\to X$, then if we sample $do(X'=1)$ and $do(X'=0)$ for $\frac{1}{c_e^2}\log \frac{4n^2t^2}{\delta}$ times within sub-procedures RECOVER-EDGE$(a)$, similarly we will identify the edge $X'\to X$. Then because the RECOVER-EDGE$(a)$ will perform intervention $do(X'=0)$ and $do(X'=1)$ for the $X'$ that the direction of $(X',X)$ has not been discovered each time, after $\sum_{e:X'\to X}\frac{1}{c_e^2}\log \frac{4n^2t^2}{\delta}$.
\end{proof}
Then we define 
\begin{align*}
    \cE_4 = \{\mbox{Lemma~\ref{lemma:identify} holds}\}
\end{align*}
Then $\Pr\{\cE_4^c\}\le \delta$.
Also, under the event $\cE_2$, the following lemma shows that if $D_a(t)$ is really large, we can estimate the $\mu_a$ accurately.
\begin{lemma}
    Under event $\cE_2$, if $D_a(t)\ge \frac{64}{\max\{\Delta_a,\varepsilon/2\}^2}\log \frac{16n^2Z_at^3}{\delta}$, then 
    \begin{align*}
        \beta_a(T_1)\le \frac{\max\{\Delta_a,\varepsilon/2\}^2}{4}.
    \end{align*}
\end{lemma}
\begin{proof}
    In fact, 
    \begin{align*}
        \beta_a(t)\le \beta_{I,a}(t) = 2\sqrt{\frac{1}{D_a(t)}\log\left(\frac{2n\log(2t)}{\delta}\right)}\le 2\sqrt{\frac{1}{D_a(t)}\log \frac{16n^2Z_at^3}{\delta}}\le \frac{\max\{\Delta_a,\varepsilon/2\}^2}{4}.
    \end{align*}
\end{proof}

Now we turn to our main result.
From the Lemma~\ref{lemma: 6 in xiong}, at least one arm $a$ with $\beta_{a}(t)\ge \frac{\max\{\Delta_{a},\varepsilon/2\}}{4}$ will be performed an intervention at each round $t\ge T_1$. Under the event $\cE_1, \cE_2, \cE_3$ and $\cE_4$, these interventions will only performed in two types of action $a$:
\begin{itemize}
    \item $q_a\le \frac{1}{H_{m_{\varepsilon,\Delta}}\cdot \max\{\Delta_a,\varepsilon/2\}^2}$ and $D_a(t)\le \frac{64}{\max\{\Delta_a,\varepsilon/2\}^2}\log \frac{16n^2Z_at^3}{\delta}$.
    \item $D_a(t)\le \min\{MC_a\log(t), \frac{64}{\max\{\Delta_a,\varepsilon/2\}^2}\log \frac{16n^2Z_at^3}{\delta}\}$.
\end{itemize}
Note that $q_a\le \frac{1}{H_{m_{\varepsilon,\Delta}}\cdot \max\{\Delta_a,\varepsilon/2\}^2}$ implies that $a \in S$, then after at most $T_2$ rounds, where
\begin{align*} 
    T_2 &= 64\left(\sum_{a \in S}\frac{1}{\max\{\Delta_a,\varepsilon/2\}^2}+\sum_{a\notin S}\min\left\{\frac{1}{\max\{\Delta_a,\varepsilon/2\}^2}, \frac{1}{c_a^2}+\sum_{e:X'\to X}\frac{1}{c_e^2}\right\}\right)\log \frac{16n^2ZT_2^3}{\delta} \\&= 64H\log\frac{16n^2ZT_2^3}{\delta}
\end{align*}
 the algorithm should terminates.  The fist term is the summation of all actions in $S$, and the second term is for the second type of actions, where $$D_a(t)\le\min\{MC_a\log(t), \frac{64}{\max\{\Delta_a,\varepsilon/2\}^2}\log \frac{16n^2Z_at^3}{\delta}\}. $$
 Denote $T = T_1+T_2$, then 
 \begin{align*}
     T = T_1 + T_2 \le 256H \log\frac{16n^2ZT^3}{\delta}\le 768H\log \frac{16nZT}{\delta}
 \end{align*}
 
 Then by the Lemma~\ref{lemma:technical transform}, with probability at least $1-4\delta$, the sample complexity has the upper bound
 \begin{align*}
     T = O\left(H\log\left(\frac{nZH}{\delta}\right)\right)
 \end{align*}
 Replace $\delta$ to $\delta/4$, we derive the sample complexity in the Theorem~\ref{thm:easy-unknowngraph}.
 The correctness of algorithm can be derived by LUCB1 algorithm. We provide a short argument here.
 Because the stopping rule is $\hat{\mu}_{a_l^{t}}^{t}+\beta_{a_l^t}(t)\le \hat{\mu}_{a_h^{t}}^{t}-\beta_{a_h^t}(t)+\varepsilon$,
	if $a^*\neq a_h^t$, we have 
	\begin{align*}
	    \mu_{a_h^t}+\varepsilon\ge \hat{\mu}_{a_h^{t}}-\beta_{a_h^t}(t)+\varepsilon\ge \hat{\mu}_{a_l^{t}}+\beta_{a_l^t}(t)
	    \ge \hat{\mu}_{a^*}+\beta_{a^*}(t)
	    \ge \mu_{a^*}.
	\end{align*}
	Hence either $a^*=a_h^t$ or $a_h^t$ is $\varepsilon$-optimal arm.

\subsection{Proof of Lemma~\ref{lemma: 6 in xiong}}
For completeness, we provide the proof in \cite{xiong2022pure}.
\begin{proof}
    If the optimal arm $a^*= a_h^t$,
		    \begin{align}
		        \hat{\mu}_{a_l^{t}}+\beta_{a_l^t}(t)&\le \mu_{a_l^t}+2\beta_{a_l^t}(t)\nonumber\\
		        &\le \mu_{a_l^t}+\frac{\max\{\Delta_{a_l^t},\varepsilon/2\}}{2}\nonumber\\
		        &\le \mu_{a_h^t}-\Delta_{a_l^t}+\frac{\max\{\Delta_{a_l^t},\varepsilon/2\}}{2}\nonumber\\
		        &\le \hat{\mu}_{a_h^t}+\beta_{a^*}(T_{a^*}(t))-\Delta_{a_l^t}+\frac{\max\{\Delta_{a_l^t},\varepsilon/2\}}{2}\nonumber\\
	           &\le \hat{\mu}_{a_h^t}-\beta_{a^*}(T_{a^*}(t))+\frac{\max\{\Delta_{a^*},\varepsilon/2\}+\max\{\Delta_{a_l^t},\varepsilon/2\}}{2}-\Delta_{a_l^t}\nonumber\\
	           &\le \hat{\mu}_{a_h^t}-\beta_{a^*}(T_{a^*}(t))+\frac{\Delta_{a^*}+\varepsilon/2+\Delta_{a_l^t}+\varepsilon/2}{2}-\Delta_{a_l^t}\nonumber\\
	           &\le \hat{\mu}_{a_h^t}-\beta_{a^*}(T_{a^*}(t))+\varepsilon.\nonumber
		    \end{align}
		    If optimal arm $a^*\neq a_h^t$, and the algorithm doesn't stop at round $t+1$, then we prove $a^*\neq a_l^t$. Otherwise, assume $a^*=a_l^t$
		    \begin{align}
		        \hat{\mu}_{a_h^t}^t &\le \mu_{a_h^t}^t+\frac{\max\{\Delta_{a_h^t},\varepsilon/2\}}{4}\\&= \mu_{a_l^t}^t-\Delta_{a_h^t}+\frac{\max\{\Delta_{a_h^t},\varepsilon/2\}}{4}\\&\le
		        \mu_{a_l^t}^t-\frac{3\Delta_{a_h^t}}{4}+\varepsilon/4\\
		        &\le \hat{\mu}_{a_l^t}^t+\frac{\max\{\Delta_{a^*},\varepsilon/2\}}{4}-\frac{3\Delta_{a_h^t}}{4}+\varepsilon/4\\&\le
		        \hat{\mu}_{a_l^t}^t+\varepsilon/2-\frac{\Delta_{a_h^t}}{2}.
		    \end{align}
		    From the definition of $a_{h}^t$, we know $\varepsilon>\Delta_{a_h^t}\ge\Delta_{a^*}, \beta_{a_h^t}(t)\le \varepsilon/4, \beta_{a_l^t}(t)\le \varepsilon/4$.
		    Then $\hat{\mu}_{a_l^t}^t+\beta_{a_l^t}(t)+\beta_{a_h^t}(t)\le \hat{\mu}_{a_l^t}+\varepsilon/2\le \hat{\mu}_{a_h^{t}}^t+\varepsilon,$ which means the algorithm stops at round $t+1$.
		    
		    Now we can assume $a^*\neq a_l^t, a^*\neq a_h^t$. Then 
		    \begin{align}
		          \mu_{a_l^t}+2\beta_{a_l^t}(t)\ge\hat{\mu}_{a_l^t}+\beta_{a_l^t}(t)\ge \hat{\mu}_{a^*}+\beta_{a^*}(T_{a^*}(t))\ge \mu_{a^*}=\mu_{a_l^t}+\Delta_{a_l^t}.
		    \end{align}
		    Thus 
		    \begin{align}
		        \Delta_{a_l^t}\le 2\beta_{a_l^t}(t)\le \frac{\max\{\Delta_{a_l^t},\varepsilon/2\}}{2},
		    \end{align}
		    which leads to $\Delta_{a_l^t}\le \varepsilon/2, \beta_{a_l^t}(t)\le \varepsilon/8$. Since 
		    
		    Also, 
		    \begin{align}
		        \mu_{a_h^t}+\beta_{a_h^t}(t)\ge\hat{\mu}_{a_h^t}\ge \hat{\mu}_{a_l^t}\ge \mu_{a^*}-\beta_{a_l^t}(t)= \mu_{a_h^t}+\Delta_{a_h^t}-\beta_{a_l^t}(t),
		    \end{align}
		    which leads to 
		    \begin{align}
		        \frac{\max\{\Delta_{a_h^t},\varepsilon/2\}}{4}\ge \Delta_{a_h^t}-\varepsilon/8,
		    \end{align}
		    and $\Delta_{a_h^t}\le \varepsilon/2, \beta_{a_h^t}(t)\le \varepsilon/8.$ Hence
		    $\hat{\mu}_{a_l^t}^t+\beta_{a_l^t}(t)+\beta_{a_h^t}(t)\le \hat{\mu}_{a_l^t}+\varepsilon/2\le \hat{\mu}_{a_h^{t}}^t+\varepsilon,$ which means the algorithm stops at round $t+1$.
\end{proof}
\subsection{Proof of Theorem~\ref{thm:generallowerbound}}
\begin{proof}
    We construct $n-1$ graphs with the same distribution $P(\bX, Y)$ but different causal graph. Indeed, We construct the bandit instances $\{\xi_i\}_{2\le i\le n}$ as follows.
    For instance $\xi_2$, the graph structure contains edge $X_1\to Y, X_2\to X_1, X_1\to X_i(3\le i\le n)$ and $X_2\to X_i(3\le i\le n)$.
    For instances $\xi_i(3\le i\le n)$, we change $X_1\to X_i$ to $X_i\to X_1$. The graph structure are shown in the Figure~\ref{fig:lowerbound2 tau2} and Figure~\ref{fig:lowerbound2 taui}.

    The observational distribution for all instance is:
    \begin{align}
        P(\bX, Y) = p_1p_2\ldots p_{n},
    \end{align}
    where 
    \begin{align}
        p_1 &= 0.5, \\ \ p_2 &= \left\{\begin{array}{cc}
            0.5+\varepsilon & x_2 = x_1 \\
            0.5-\varepsilon & x_2 \neq x_1
        \end{array}\right. \\ \ \ p_i &= \left\{\begin{array}{cc}
            0.5+4\varepsilon & x_i = x_1 \\
            0.5-4\varepsilon & x_i \neq x_1
        \end{array}\right..
    \end{align}
    \begin{figure}[H]
        \centering
        \includegraphics[scale = 0.7]{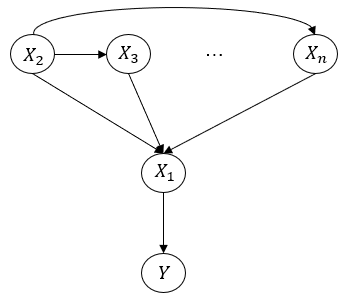}
        \caption{Causal Bandits Instance $\tau_2$}
        \label{fig:lowerbound2 tau2}
    \end{figure}
    \begin{figure}[H]
        \centering
        \includegraphics[scale = 0.7]{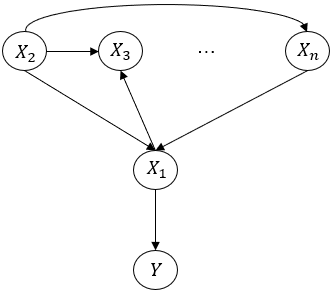}
        \caption{Causal Bandits Instance $\tau_i (i = 3)$}
        \label{fig:lowerbound2 taui}
    \end{figure}

    It is easy to check that $\sum_{\bx, y}P(\bX=\bx, Y=y) = 1$ and $P(X_i=1)=0.5$. The action set is $do(), do(X_i=1), do(X_i=0)$ where $2\le i\le n$, which means the action set does not contain $do(X_1=x)$ for $x = 0,1$.
    
    Now in $\xi_2$, we consider $P(Y=1\mid do(X_2=1))$. Actually, it is easy to show that $P(Y=1\mid do(X_2=1))=P(X_1=1\mid do(X_2=1))=0.5+\varepsilon$. Similarly, $P(Y=1\mid do(X_2=0))=0.5-\varepsilon$. For other actions, $P(Y=1\mid a) = P(X_1=1\mid a) = 0.5$ since other actions $a$ will not influence the value of $X_1$.

    Now consider instance $\xi_i$ for $3\le i\le n$. For action $do()$ and $do(X_j=x)$ with $j\neq 2,i$, it will not influence the value of $X_1$ and then $P(Y=1\mid a) = 0.5$.
    Now consider action $a = do(X_2=1)$, we have 
    \begin{align*}
        P(Y=1\mid do(X_2=1)) & = P(X_1=1\mid do(X_2=1))\\
        &=P(X_1=1\mid X_2=1)=0.5+\varepsilon.
    \end{align*}
    Similarly, $P(Y=1\mid do(X_2=0))=0.5-\varepsilon$.

    Now we calculate $P(Y=1\mid do(X_i=1))$ in instance $\xi_i$.
    In fact, denote $q = 0.5+4\varepsilon$ and by do-calculus,
    \begin{align*}
        &P(X_1=1\mid do(X_i=1))\\&=\sum_{x = 0,1}P(X_1=1\mid X_i=1, X_2=x)P(X_2=x)\\
        & = 0.5(P(X_1=1\mid X_i=1, X_2=0)+P(X_1=1\mid X_i=1, X_2 = 1)\\
        &= 0.5\left(\frac{P(X_1=1, X_i=1, X_2=0)}{P(X_i=1, X_2=0)}+\frac{P(X_1=1, X_i=1, X_2=1)}{P(X_i=1, X_2=1)}\right)\\
        &=0.5\left(\frac{(0.5+4\varepsilon)(0.5-\varepsilon)}{(0.5+4\varepsilon)(0.5-\varepsilon)+(0.5-4\varepsilon)(0.5+\varepsilon)}+\frac{(0.5+4\varepsilon)(0.5+\varepsilon)}{(0.5+4\varepsilon)(0.5+\varepsilon)+(0.5-4\varepsilon)(0.5-\varepsilon)}\right)\\
        & = 0.5\left(\frac{q(0.5-\varepsilon)}{q(0.5-\varepsilon)+(1-q)(0.5+\varepsilon)}+\frac{q(0.5+\varepsilon)}{q(0.5+\varepsilon)+(1-q)(0.5-\varepsilon)}\right)\\
        &=0.5\left(\frac{q(0.5-\varepsilon)}{0.5-(2q-1)\varepsilon}+\frac{q(0.5+\varepsilon)}{0.5+(2q-1)\varepsilon}\right)\\
        & = q\left(\frac{0.5^2-(2q-1)\varepsilon^2}{0.5^2-(2q-1)^2\varepsilon^2}\right)\le q = 0.5+4\varepsilon.
    \end{align*}
    Also, we prove that \begin{align*}
        q\left(\frac{0.5^2-(2q-1)\varepsilon^2}{0.5^2-(2q-1)^2\varepsilon^2}\right) \ge 0.5+2\varepsilon.
    \end{align*}
    Actually, this inequality is equal to 
    \begin{align*}
        (0.5+4\varepsilon)(0.5^2-8\varepsilon^3)&\ge (0.5+2\varepsilon)(0.5^2-8\varepsilon^4)\\
        \iff 1&\ge 56\varepsilon^3 + 8\varepsilon^2 - 32\varepsilon^4.
    \end{align*}
    When $\varepsilon$ is small enough, this inequality holds. 
    In summary, we have 
    \begin{align*}
        P(X_1=1\mid do(X_i=1)) \in [0.5+2\varepsilon, 0.5+4\varepsilon].
    \end{align*}
    Similarly, we can get 
    \begin{align*}
        P(X_1=1\mid do(X_i=0)) &= 0.5(P(X_1=1\mid X_i=0, X_2=1)+P(X_1=1\mid X_i=0, X_2=0)) \\&= (1-q)\left(\frac{0.5^2-(1-2q)\varepsilon^2}{0.5^2-(1-2q)^2\varepsilon^2}\right) \in [0.5-4\varepsilon,0.5].
    \end{align*}
    Now in instance $\xi_2$, the output action should be $do(X_2=1)$, while in instance $\xi_i$, the output action should be $do(X_i=1)$.

    Now by Pinkser's inequality, for an policy $\pi$, we have 
    \begin{align*}
        2\delta \ge P_{\xi_2}(a^o = do(X_i=1))+P_{\xi_i}(a^o \neq do(X_i=1))\ge \exp(-\mbox{KL}(\xi_2^\pi, \xi_i^\pi)).
    \end{align*}
    Also, assume the stopping time as $\tau$ for the environment $\cE$, the KL divergence can be rewritten as 
    \begin{align}
        \mbox{KL}(\xi_2^\pi, \xi_i^\pi) &= \EE_{A_t\sim\xi_2^\pi}\left[\sum_{t=1}^\tau \mbox{KL}(P_{\xi_2}(\bX_t, Y_t\mid A_t),P_{\xi_i}(\bX_t, Y_t\mid A_t) )\right]\\
        &=\EE_{\xi_2^\pi}\left[\sum_{t=1}^\tau P_{\xi_2}(\bX_t, Y_t\mid A_t)\left(\log\frac{ P_{\xi_2}(\bX_t, Y_t\mid A_t)}{ P_{\xi_i}(\bX_t, Y_t\mid A_t)}\right)\right]\\
        &= \EE_{\xi_2^\pi}\left[\sum_{t=1}^\tau P_{\xi_2}(X_{t,i},X_{t,1}\mid A_t)\left(\log\frac{P_{\xi_2}(X_{t,i}, X_{t,1}\mid A_t)}{P_{\xi_i}(X_{t,i}, X_{t,1}\mid A_t)}\right)\right]\label{eq:KLdiv}
    \end{align}
    where the last equation is derived as follows: \begin{align*}\frac{P_{\xi_2}(\bX_t, Y_t\mid A_t)}{P_{\xi_i}(\bX_t, Y_t\mid A_t)} = \frac{P_{\xi_2}(X_{t,i}, X_{t,1}\mid A_t)\cdot P_{\xi_2}(\Bar{\bX}_{t,i},Y_t\mid X_{t,i}, X_{t,1}, A_t)}{P_{\xi_i}(X_{t,i}, X_{t,1}\mid A_t)\cdot P_{\xi_i}(\Bar{\bX}_{t,i}, Y_t\mid X_{t,i}, X_{t,1}, A_t)}\end{align*}
    where $\Bar{\bX}_{t,i} = \bX_t\setminus \{X_{t,i}, X_{t,1}\}$.  Now since $\Bar{\bX}_{t,i}$ is only decided by $X_1, X_2$ and $X_2$ is only decided by $A_t$, then $$P_{\xi_2}(\Bar{\bX}_{t,i},Y_t\mid X_{t,i}, X_{t,1}, A_t) = P_{\xi_i}(\Bar{\bX}_{t,i},Y_t\mid X_{t,i}, X_{t,1}, A_t)$$ and then 
    \begin{align*}
        \frac{P_{\xi_2}(\bX_t, Y_t\mid A_t)}{P_{\xi_i}(\bX_t, Y_t\mid A_t)} =  \frac{P_{\xi_2}(X_{t,i}, X_{t,1}\mid A_t)}{P_{\xi_i}(X_{t,i}, X_{t,1}\mid A_t)}.
    \end{align*}

    Note that only when $A_t = do(X_i=1), do(X_i=0)$, $P_{\xi_2}(X_{t,i}, X_{t,1}\mid A_t) \neq P_{\xi_i}(X_{t,i}, X_{t,1}\mid A_t)$. Then the equation \eqref{eq:KLdiv} can be further calculated as 
    \begin{align*}
        \mbox{\eqref{eq:KLdiv}}  &= \sum_{x = 0,1}\EE_{\xi_2^\pi}\left[\sum_{t=1}^\tau \II\{A_t = do(X_i=x)\}\right]\cdot P_{\xi_2}(X_{t,i}, X_{t,1}\mid do(X_i=x))\\&\cdot\left(\log\frac{P_{\xi_2}(X_{t,i}, X_{t,1}\mid do(X_i=x))}{P_{\xi_i}(X_{t,i}, X_{t,1}\mid do(X_i=x))}\right)  \\
        &=\sum_{x = 0,1}\EE_{\xi_2^\pi}\left[\sum_{t=1}^\tau \II\{A_t = do(X_i=x)\}\right]\cdot P_{\xi_2}( X_{t,1}\mid do(X_i=x))\\&\cdot\left(\log\frac{P_{\xi_2}( X_{t,1}\mid do(X_i=x))}{P_{\xi_i}(X_{t,1}\mid do(X_i=x))}\right) \\
        & \le \sum_{x = 0,1}\EE_{\xi_2^\pi}\left[\sum_{t=1}^\tau \II\{A_t = do(X_i=x)\}\right]\left(0.5\cdot \left(\log\frac{0.5}{0.5+4\varepsilon} + \log\frac{0.5}{0.5-4\varepsilon}\right)\right)\\
        & \le \sum_{x = 0,1}\EE_{\xi_2^\pi}\left[\sum_{t=1}^\tau \II\{A_t = do(X_i=x)\}\right] 96\varepsilon^2\\
        & = 96\varepsilon^2\cdot \EE_{\xi_2^\pi}[N(do(X_i=1))+N(do(X_i=0))].
    \end{align*}
    where the $\EE_{\xi_2^\pi}N(a)$ represents that the number of times taking action $a$ for policy $\pi$ under the instance $\xi_2$. Now we have 
    \begin{align*}
         \EE_{\xi_2^\pi}[N(do(X_i=1))+N(do(X_i=0))]\ge \frac{\mbox{KL}(\xi_2^\pi, \xi_i^\pi)}{96\varepsilon^2}\ge \frac{1}{96\varepsilon^2}\log\frac{1}{2\delta}.
    \end{align*}
    Hence the stopping time $\tau$ under policy $\pi$ can be lower bounded by 
    \begin{align*}
        \EE_{\xi_2^\pi}[\tau]\ge \sum_{i=3}^n \EE_{\xi_2^\pi}[N(do(X_i=1))+N(do(X_i=0))] \ge \frac{n-2}{96\varepsilon^2}\log\frac{1}{2\delta} = O\left(\frac{n}{\varepsilon^2}\log \frac{1}{\delta}\right).
    \end{align*}
    
\end{proof}
\subsection{Technical Lemma}
\begin{lemma}\label{lemma:technical transform}
    If $T = CH\log \frac{dT}{\delta}$ for some constant $C$ and parameter $d$ such that $d\ge e\delta$, then $T = O(H\log \frac{Hd}{\delta})$.
\end{lemma}
\begin{proof}
    Let $f(x) = \frac{x}{\log (dx/\delta)}$, then for $x\ge 1$ 
    \begin{align*}
        f'(x) = \frac{\log (dx/\delta) - 1}{\log^2 dx/\delta}\ge 0
    \end{align*}
    because $dx/\delta>e$. Then $f(x)$ is non-decreasing for $x\ge 1$.

    To prove $T = O(H\log \frac{Hd}{\delta})$, we only need to show that $f(T) \le f(C'H\log \frac{Hd}{\delta})$ for some constant $C'$.
    Since
    \begin{align*}
        \log \frac{C'Hd\log \frac{Hd}{\delta}}{\delta} & = \log\frac{C'Hd}{\delta} + \log\log \frac{Hd}{\delta}
    \end{align*}
    we only need to prove 
    \begin{align*}
        f(C'H\log \frac{Hd}{\delta}) = \frac{C'H\log \frac{Hd}{\delta}}{\log\frac{C'Hd}{\delta} + \log\log \frac{Hd}{\delta}}\ge CH =f(T).
    \end{align*}
    If we choose $C'\ge 2C+C\log C'$, then
    \begin{align*}
        CH\left(\log\frac{C'Hd}{\delta} + \log\log \frac{Hd}{\delta}\right) &\le CH(\log\frac{C'Hd}{\delta} + \log \frac{Hd}{\delta})\\
        &\le 2CH\log \frac{Hd}{\delta} + CH\log C'\\
        &\le (2C+C\log C') H\log \frac{Hd}{\delta}\\
        &\le C'H\log \frac{Hd}{\delta}.
    \end{align*}
\end{proof}

\label{sec.general}

%%%%%%%%%%%%%%%%%%%%%%%%%%%%%%%%%%%%%%%%%%%%%%%%%%%%%%%%%%%%%%%%%%%%%%%%%%%%%%%
%%%%%%%%%%%%%%%%%%%%%%%%%%%%%%%%%%%%%%%%%%%%%%%%%%%%%%%%%%%%%%%%%%%%%%%%%%%%%%%

\end{document}